\documentclass{article}

\PassOptionsToPackage{numbers, compress}{natbib}
\usepackage[preprint]{neurips_2026}

\usepackage[utf8]{inputenc}
\usepackage[T1]{fontenc}
\usepackage{hyperref}
\usepackage{url}
\usepackage{booktabs}
\usepackage{amsfonts}
\usepackage{amsmath}
\usepackage{amssymb}
\usepackage{nicefrac}
\usepackage{microtype}
\usepackage{xcolor}
\usepackage{graphicx}
\graphicspath{{results/comprehensive_e2e/}{results/synthetic_continuous_hb_vs_dibs_er_r5/}{results/}{results/e_rho_b_recovery/}{figures/}{./}}
\usepackage{multirow}
\usepackage{subcaption}
\usepackage{bm}
\usepackage{amsthm}

\newcommand{\D}{\mathcal{D}}

\newcommand{\vtheta}{\bm{\Theta}}
\newcommand{\KG}{\mathcal{K}}

\newcommand{\E}{\mathbb{E}}

\newtheorem{proposition}{Proposition}

\title{Coupling Noisy Pairwise Knowledge to the DAG Posterior for Causal Discovery}

\author{
  Guoliang Xu \\
  Columbia University
  \And
  James E. Corter \\
  Columbia University
}

\begin{document}
\maketitle

\begin{abstract}
External causal reports can improve structure learning from limited observations, but their reliability varies across sources and variable pairs. We introduce \emph{HB-NoisyKG}, a Bayesian framework that combines observational data with repeated causal reports from sources such as large language models. Each report is a noisy observation of a direct pair state implied by one DAG. A feature-conditioned Beta prior pools information about pair reliability, and a shared error matrix captures systematic mistakes. Alternating inference uses the graph posterior to refine reliability estimates, which determine how reports influence subsequent graph updates. The report likelihood uses only graph pair-state marginals, so the same observation layer supports discrete and continuous likelihoods in graph-only and joint inference. Against an 80-restart no-KG baseline, HB uses at most 80 total restarts and lowers mean SHD from $22.39$ to $16.06$ on five discrete benchmarks. On a physical light tunnel with random variable IDs and retained descriptions, HB lowers SHD from $39.00$ for no-KG to $27.30$. On continuous Sachs, graph-only BGe raises AUROC by $.121$ over no-KG Top-$K$. In a controlled synthetic study, continued updating also reduces mean reliability estimation error and held-out report log loss compared with one-time estimation.
\end{abstract}


\section{Introduction}
\label{sec:intro}

Researchers often have partial causal information from pathway databases, mechanistic reports, the literature, or large language models (LLMs). These sources can help recover structure from limited observations, but the quality of their claims varies. Some reports identify a direct relation, others reverse its direction, and still others describe an influence mediated by another measured variable. Repeated queries provide additional evidence, including disagreements and abstentions. The practical task is to combine this information with data and estimate how much influence each reported relation should have.

Bayesian structure learning represents uncertainty over directed acyclic graphs (DAGs)~\citep{madigan1995bayesian,friedman2003bayesian,lorch2021dibs}. It supports model averaging for causal quantities and posterior-based intervention design~\citep{tong2001active}, but inference is difficult because the graph space is combinatorial. External knowledge can guide this inference. Existing expert models estimate or marginalize source accuracy when scoring structures~\citep{amirkhani2017experts}; LLM-based approaches also use structural priors with a global confidence parameter~\citep{ban2025harmonized}. We study repeated reports whose reliability can vary among variable pairs and whose error patterns can be shared across the source.

We propose \emph{HB-NoisyKG}, an observation model in the Dawid--Skene tradition~\citep{dawid1979,raykar2010}. A DAG assigns each unordered pair a direct state: no edge or one of the two orientations. Reports provide noisy observations of that state. Pair reliability determines the probability of a correct report, and a shared error matrix describes the alternatives reported when the source errs. The model combines feature-conditioned pair reliability with a shared error distribution, linking report learning and graph inference through the pair-state marginals of a single DAG. Joint estimation uses report features and agreement between reports and plausible graphs. This interface supports collapsed graph-only inference and joint inference over graphs and data-model parameters. We use ``KG'' as shorthand for the repeated causal reports.

The experiments evaluate both structural recovery and source estimation. Budget-controlled comparisons on five discrete benchmarks and a physical light-tunnel task establish gains over data-only inference. A controlled synthetic study tests the value of continued updating through reliability estimation error and prediction of independent held-out reports. Independent reliability models and post-fit pooling examine how the estimated trust levels and their assignment affect graph recovery. Accuracy, coverage, and elicitation studies identify conditions under which reports remain useful. Continuous Sachs~\citep{sachs2005causal} and generated ER graphs evaluate the interface across posterior families, and sparse binary experiments assess execution at 100 and 125 variables.

We make the following contributions.
\begin{enumerate}
    \item We formulate a likelihood for repeated causal reports with pair-specific reliability, a feature-conditioned prior, and a shared error matrix (Section~\ref{sec:reliability_posterior}). This makes report influence estimable within Bayesian graph inference.
    \item We derive a report reward that depends only on the graph marginal, together with alternating graph, reliability, and confusion updates (Section~\ref{sec:attribution}). Closed-form or low-dimensional non-graph updates allow one observation-layer implementation to serve BDeu, BGe, and linear and nonlinear joint posteriors.
    \item We demonstrate structural gains over data-only inference on five classical networks and a physical system. Controlled experiments show that continued updating reduces mean reliability estimation error and held-out report log loss compared with one-time estimation (Tables~\ref{tab:search_budget_main}, \ref{tab:optical_main}, and~\ref{tab:feedback_main}).
\end{enumerate}

\section{Related work}
\label{sec:related}

\paragraph{Bayesian structure learning and knowledge guidance.}
Bayesian network (BN) structure learning spans score-based, constraint-based, MCMC, and differentiable approaches~\citep{kitson2023survey}; methods such as DiBS~\citep{lorch2021dibs} move inference to continuous latent spaces and support both graph-only and joint graph-and-parameter inference. Prior knowledge has long been used through hard constraints, soft edge priors, and side constraints~\citep{heckerman1995learning,castelo2000priors,de2004bayesian,li2018bayesian,mokhtarian2022learning}. \citet{amirkhani2017experts} model forward, reverse, and absent-edge opinions, allow missing opinions, and estimate or marginalize expert-level accuracy when scoring structures. HB-NoisyKG instead assigns feature-conditioned reliability to repeatedly reported pairs and estimates a shared distribution over incorrect states alongside the graph posterior. VaMSL~\citep{bjorkman2026mixtures} extends DiBS to mixtures of causal graphs and incorporates actively elicited expert feedback through a graph prior. Our focus is the observation model for repeated reports of unknown reliability about a single graph.

\paragraph{LLMs, noisy labels, and Dawid--Skene coupling.}
Recent work uses large language models as sources of causal hypotheses or structural priors~\citep{kiciman2023causal,zevcevic2023causal,jiralerspong2024efficient,darvariu2024priors}. Harmonized-Prior~\citep{ban2025harmonized} converts LLM output into a structural prior with globally tuned confidence $c$. MosaCD~\citep{lyu2025mosacd} combines conditional-independence tests with shuffled LLM queries to select reliable orientation seeds and prioritize propagation. \citet{long2023imperfect} use acyclicity and Markov-equivalence consistency to amend expert orientations. \citet{vashishtha2025order} use causal order to address direct-versus-mediated ambiguity. HB-NoisyKG retains repeated reports in an observation likelihood with pair-specific reliability and a shared error matrix inferred alongside $G$. Its latent pair states are constrained by a common DAG, and the likelihood can be used with either collapsed or joint graph-posterior inference.

\section{Method}
\label{sec:method}

Bayesian structure learning~\citep{koller2009probabilistic} targets either a collapsed graph posterior $p(G\mid\D)\propto p_0(G)p(\D\mid G)$ or a joint posterior $p(G,\vtheta\mid\D)$ retaining data-model parameters. HB-NoisyKG adds a likelihood for repeated observations of the direct pair states implied by $G$. The same report model is used in both posterior families.

\subsection{Core hierarchical model}
\label{sec:calibrated_prior}

Let $G$ be a DAG on variables $x_{1:d}$, let $\D$ denote observed data, and let $\vtheta$ denote optional data-model parameters. HB-NoisyKG defines the hierarchical posterior
\begin{equation}
\begin{aligned}
    p(G, \vtheta, \rho, B, \beta, \tau \mid \D, Y, \Phi)
    \propto {} & p_0(G)\, p(\vtheta \mid G)\, p(\D \mid G, \vtheta) \\
    & \times p(B)\, p(\beta)\, p(\tau) \\
    & \times \prod_{i<j} p(\rho_{ij} \mid \phi_{ij}, \beta, \tau) \\
    & \times \prod_{i<j}\prod_r p\!\left(Y_{ij}^{(r)} \mid Z_{ij}(G), \rho_{ij}, B\right),
\end{aligned}
\label{eq:master_posterior}
\end{equation}
where $G$ is the graph structure, $\vtheta$ are optional data-model parameters, and $Z_{ij}(G)\in\{\text{none}, i\to j, j\to i\}$ is the direct state implied by $G$ for pair $\{i,j\}$. The report parameters are pair reliability $\rho_{ij}$, a shared error matrix $B$, and prior hyperparameters $(\beta,\tau)$. The features $\Phi=\{\phi_{ij}\}$ describe the pairs, and $Y=\{Y_{ij}^{(r)}\}$ contains their repeated reports. The data term $p(\D\mid G,\vtheta)$ is setting-specific; the report likelihood is shared.

When features are report summaries, we compute $\Phi(Y)$ once and hold it fixed as covariates in an empirical reliability prior; $(\beta,\tau)$ are estimated by the MAP empirical-Bayes step.

This master posterior induces two posterior families. If $\vtheta$ is integrated out, then
\begin{equation}
    p(\D\mid G) = \int p(\vtheta\mid G)\, p(\D\mid G,\vtheta)\, d\vtheta
\end{equation}
and inference targets the collapsed graph-only posterior over $(G,\rho,B,\beta,\tau)$. Retaining $\vtheta$ gives the joint posterior over $(G,\vtheta,\rho,B,\beta,\tau)$. In either case, reports observe the graph through $Z_{ij}(G)$.

\subsection{Noisy pairwise observation model}
\label{sec:reliability_posterior}

We index the direct states by $\{0,1,2\}\equiv\{\text{none},i\to j,j\to i\}$. A direct-edge prompt asks for immediate parents relative to the measured variable set. A path-existence prompt asks whether one variable is an ancestor of another~\citep{ban2025harmonized,vashishtha2025order}. For example, $i\to k\to j$ establishes a path from $i$ to $j$, and a separate direct edge remains a distinct claim. Our likelihood models the direct state. The prompt comparisons in Appendix~\ref{app:prompt_semantics} evaluate these targets explicitly.

Let $Y_{ij}^{(r)}$ denote the $r$-th \emph{non-abstain} report for pair $\{i,j\}$. The principal LLM studies use five query rounds per batch; the synthetic report-count studies vary $R$. We model observed labels as conditionally independent given $Z_{ij}(G)$, $\rho_{ij}$, and $B$. Abstentions are missing labels: observed reports supply the update counts, and the abstention rate enters the pair features. Pair reliability $\rho_{ij}\in[0,1]$ is the modeled probability that an observed report matches the graph-implied state.

The shared $3\times3$ error matrix $B$ describes the alternative labels reported after an error. Its diagonal is zero, $B_{zz}=0$, and each off-diagonal row sums to one, $\sum_{y\ne z}B_{zy}=1$. Thus $\rho_{ij}$ governs report accuracy and $B$ distributes the error probability among the other states. This permits different levels of accuracy across pairs together with a common source bias, such as reversing direction. The observation model is
\begin{equation}
    p(Y_{ij}^{(r)} = y \mid Z_{ij}(G)=z, \rho_{ij}, B)
    = \rho_{ij}\, \mathbb{I}[y=z] + (1-\rho_{ij})\, B_{zy}.
\end{equation}
We place a row-wise Dirichlet prior on the off-diagonal entries of $B$, with the diagonal fixed to zero.
A shared $B$ pools sparse off-diagonal counts across pairs; five reports per pair provide limited information for estimating separate error matrices.

To pool information across pairs, we place a feature-conditioned prior on reliability:
\begin{equation}
    \mu_{ij} = \sigma(\phi_{ij}^\top \beta),
    \qquad
    \rho_{ij} \sim \mathrm{Beta}(\tau \mu_{ij},\, \tau(1-\mu_{ij})).
\end{equation}
Here $\sigma(\cdot)$ is the logistic sigmoid. In our experiments, $\phi_{ij}$ contains twelve simple per-pair statistics, including mean reported confidence, coverage (fraction of rounds with a non-abstain report), directional agreement among rounds, explicit-none rate (fraction of ``no relation'' reports), and rate of contradictory direction reports. We place a Gaussian prior on $\beta$ and a Gamma prior on $\tau$, and update them by a maximum a posteriori (MAP) empirical-Bayes step.
Specifically, $\beta \sim \mathcal{N}(0, 4I)$ and $\tau \sim \mathrm{Gamma}(2.0, 0.25)$, both weakly informative. At a fixed warm strength, the tested hyperprior changes have comparatively small effects (Appendix~\ref{app:sensitivity}).

\subsection{Closed-form reward and alternating inference}
\label{sec:attribution}

Inference starts from a graph posterior based on the observations. Its pair-state marginals determine the expected number of matching and incorrect reports. These counts update reliability and the shared error matrix. Expected report log likelihoods then become graph rewards, which are supplied to the next graph update. We alternate these steps using the variational / empirical-Bayes factorization
\begin{equation}
    q(G,\vtheta)\, q(\rho)\, q(B)\, \delta_{\beta,\tau},
\end{equation}
where $\delta_{\beta,\tau}$ denotes a point mass at the current empirical-Bayes estimate of $(\beta,\tau)$; the graph-only case is recovered when $\vtheta$ is integrated out. The only quantity required by the noisy-knowledge layer is the graph marginal
\begin{equation}
    q(G) = \int q(G,\vtheta)\, d\vtheta,
\end{equation}
which makes the same observation model reusable across graph-only and joint settings that expose weighted graph posterior approximations.

Given weighted DAG samples, we first compute pair-state marginals
\begin{equation}
    q_{ij}(z) = \sum_G q(G)\, \mathbb{I}\!\big[Z_{ij}(G)=z\big], \qquad z\in\{\text{none}, i\to j, j\to i\}.
\end{equation}
These marginals produce closed-form Beta updates for $q(\rho_{ij})$ and Dirichlet-style updates for the off-diagonal rows of $q(B)$; Appendix~\ref{app:implementation} gives the explicit update rules. We then summarize the noisy-knowledge contribution by a per-pair reward.

Let $n_{ij}(y)$ be the non-abstain count of observed label $y\in\{\text{none}, i\to j, j\to i\}$ and let
\begin{equation}
    R_{ij} = \sum_y n_{ij}(y).
\end{equation}
For a candidate latent state $z$, the uncentered reward is
\begin{equation}
\begin{aligned}
    \psi^{\mathrm{raw}}_{ij}(z)
    = {} & n_{ij}(z)\, \E_q[\log \rho_{ij}]
    + \big(R_{ij} - n_{ij}(z)\big)\, \E_q[\log (1-\rho_{ij})] \\
    & + \sum_{y \neq z} n_{ij}(y)\, \E_q[\log B_{zy}],
\end{aligned}
\label{eq:psi_raw}
\end{equation}
with
\begin{equation}
    \E_q[\log \rho_{ij}] = \psi(a_{ij}) - \psi(a_{ij}+b_{ij}),
    \qquad
    \E_q[\log (1-\rho_{ij})] = \psi(b_{ij}) - \psi(a_{ij}+b_{ij}),
\end{equation}
and
\begin{equation}
    \E_q[\log B_{zy}]
    = \psi(\alpha_{zy}) - \psi\!\left(\sum_{y'\neq z}\alpha_{zy'}\right),
\end{equation}
where $\psi(\cdot)$ on the right-hand side denotes the digamma function and $(a_{ij}, b_{ij})$ and $\alpha$ are the current Beta/Dirichlet posterior parameters (prior pseudocounts plus expected counts from the current $q_{ij}$).

We then center rewards within each pair,
\begin{equation}
    \psi_{ij}(z) = \psi^{\mathrm{raw}}_{ij}(z)
    - \frac{1}{3}\sum_{z'} \psi^{\mathrm{raw}}_{ij}(z'),
\end{equation}
which removes state-independent offsets and preserves all graph rankings. The centered rewards enter the graph posterior update. In the collapsed graph-only family,
\begin{equation}
    q(G) \propto p_0(G)\, p(\D\mid G)\, \exp\!\left(\sum_{i<j} \psi_{ij}(Z_{ij}(G))\right),
    \label{eq:mstep_graph_only}
\end{equation}
and in the explicit joint family,
\begin{equation}
    q(G,\vtheta) \propto p_0(G)\, p(\vtheta\mid G)\, p(\D\mid G,\vtheta)\, \exp\!\left(\sum_{i<j} \psi_{ij}(Z_{ij}(G))\right).
    \label{eq:mstep_joint}
\end{equation}
The noisy-knowledge term depends only on $G$, not directly on $\vtheta$.

\subsection{Instantiating the alternating inference procedure}

The graph backend supplies weighted graph samples for $q(G)$. We use Stein variational gradient descent (SVGD)~\citep{liu2016stein} on DiBS~\citep{lorch2021dibs} for continuous joint inference, and Top-$K$ local search for discrete benchmarks and graph-only Sachs. Their data models are BDeu~\citep{heckerman1995learning}, BGe~\citep{geiger2002parameter,kuipers2014addendum}, and explicit linear or nonlinear SEM likelihoods.

Before the first reliability update, we blend the data-only pair-state marginal with empirical report frequencies, weighted by report agreement. The blend helps avoid premature downweighting of consistent reports when the initial graph posterior is diffuse. The strength $\lambda_{\mathrm{warm}}$ controls this one-time blend. It is $1$ in the benchmark suite and is specified separately for the controlled studies. After each reliability update, a low-dimensional L-BFGS-B MAP step estimates $(\beta,\tau)$ from $\E_q[\log p(\rho\mid\beta,\tau)]+\log p(\beta)+\log p(\tau)$. Appendix~\ref{app:implementation} gives the update rules and study settings. Exact or objective-improving block updates make the variational/MAP objective nondecreasing; Appendix~\ref{app:proofs} states the corresponding conditions and the finite-report reliability bound.

\subsection{Properties of the observation layer}

The observation layer has an exact data-only limit and a conditional reliability-concentration property. The propositions below characterize these cases. Centering invariance and finite per-pair influence are proved in Appendix~\ref{app:proofs}.

\begin{proposition}[Reduction to data-only]
\label{prop:data_only_reduction}
If $\psi_{ij}(z)=0$ for all pairs and all latent states, then Equations~\eqref{eq:mstep_graph_only} and~\eqref{eq:mstep_joint} reduce to ordinary data-only Bayesian structure learning. In particular, this occurs when all knowledge reports abstain, so $R_{ij}=0$ for every pair and the reward term vanishes.
\end{proposition}

\noindent\emph{Proof sketch.} The exponential reward factor becomes identically one, leaving only the usual prior and data likelihood terms.

\begin{proposition}[Reliability concentration under pair-state convergence]
\label{prop:reliability_concentration}
Fix a pair $\{i,j\}$ with prior Beta parameters $a_{ij}, b_{ij} > 0$ and true reliability $\rho^*_{ij} \in (0,1)$. Assume (a)~the graph posterior has converged to the true pair state, $q_{ij}(z^*_{ij}) = 1$ with $z^*_{ij} = Z_{ij}(G^*)$, and (b)~the $R_{ij}$ non-abstain reports are conditionally i.i.d.\ under the observation model. Let $n^+_{ij}$ be the number of reports matching $z^*_{ij}$. Then the mean-field variational posterior $q(\rho_{ij}) = \mathrm{Beta}(a_{ij} + n^+_{ij},\, b_{ij} + R_{ij} - n^+_{ij})$ satisfies
\begin{enumerate}
    \item[(i)] $\E_q[\rho_{ij}] \to \rho^*_{ij}$ in probability as $R_{ij}\to\infty$;
    \item[(ii)] $\mathrm{Var}_q[\rho_{ij}] = O(1/R_{ij})$, so the posterior standard deviation is $O(1/\sqrt{R_{ij}})$.
\end{enumerate}
\end{proposition}

\noindent\emph{Proof sketch.} Given pair-state convergence, the observation model implies that each non-abstain report matches $z^*_{ij}$ with probability $\rho^*_{ij}$, so $n^+_{ij}\mid R_{ij}\sim\mathrm{Binomial}(R_{ij},\rho^*_{ij})$. The Beta posterior mean equals $(a_{ij}+n^+_{ij})/(a_{ij}+b_{ij}+R_{ij})$, which tends to $\rho^*_{ij}$ by the weak law of large numbers. The Beta variance $\alpha\beta/[(\alpha+\beta)^2(\alpha+\beta+1)]$ is bounded above by $1/[4(\alpha+\beta+1)] = O(1/R_{ij})$.

For an imperfect pair marginal with $q_{ij}(z^*_{ij})\geq1-\epsilon_{ij}$, the finite-report bound in Equation~\eqref{eq:finite_report_bound} separates three error sources: graph-state error, report-sampling variation, and prior shrinkage. With fixed prior parameters, their bounds scale as $O(\epsilon_{ij})$, $O(R_{ij}^{-1/2})$, and $O(R_{ij}^{-1})$, respectively. More reports reduce the latter two terms; accurate graph marginals remain important for reliable source estimation.

\paragraph{Roles of reliability and shared errors.}
Reliability uses match counts and pair features; $B$ pools the off-diagonal counts conditional on each latent state. This division permits accuracy to vary among pairs and learns a common error distribution. The known-graph experiment in Appendix~\ref{app:rho_b_recovery} evaluates recovery of both quantities as reports accumulate.

\section{Experiments}
\label{sec:experiments}

We first evaluate structure recovery on the five discrete benchmarks and the physical light tunnel (Sections~\ref{sec:exp_discrete} and~\ref{sec:exp_optical}). We then study reliability estimation and continued updating (Section~\ref{sec:exp_mechanism}), followed by source quality, abstention, report semantics, and inference stability (Section~\ref{sec:exp_robustness}). Continuous Sachs and generated ER graphs evaluate the interface across posterior families (Section~\ref{sec:exp_continuous}). Section~\ref{sec:exp_computational} reports computational cost.

\begin{figure}[tbp]
\centering
\includegraphics[width=\linewidth]{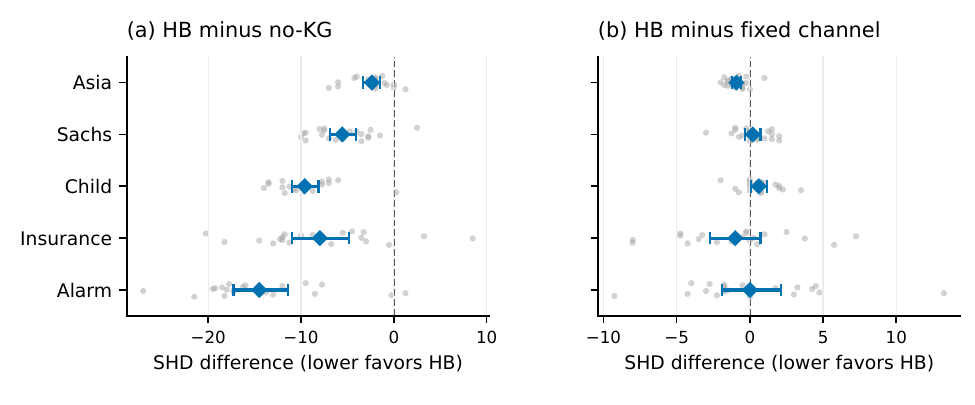}
\caption{Matched report integration on five networks. Dots are paired SHD differences for 20 data seeds, averaging four report batches within each seed. Diamonds show means; bars are unadjusted 95\% paired-bootstrap intervals, conditional on the benchmark graph and the four batches. Negative values favor HB. Panels use separate horizontal scales. All methods share the data and Top-$K$ backend; HB and the fixed-channel control receive identical reports.}
\label{fig:discrete_main}
\end{figure}

\subsection{Setup}
\label{sec:exp_setup}

\paragraph{Data.}
The discrete benchmarks are Asia ($d=8$, $8$ edges), Sachs ($d=11$, $17$ edges), Child ($d=20$, $25$ edges), Insurance ($d=27$, $52$ edges), and Alarm ($d=37$, $46$ edges) from bnlearn. Their sample-size studies use $N\in\{100,300,500,1000,5000\}$. Continuous Sachs~\citep{sachs2005causal} has 11 variables and 7,466 observations, standardized by condition. The physical light-tunnel task from Causal Chambers~\citep{gamella2025chambers} has 20 variables, ten disjoint blocks of 1,000 observations, and nested 250-row subsets. Synthetic ER studies use linear and nonlinear SEMs at $d\in\{20,30\}$, five graphs per setting, and five datasets per graph. Reusing each dataset at three report accuracies gives 75 paired conditions per configuration.

\paragraph{Reports.}
GPT-4.1-mini supplies repeated causal reports. The benchmark suite uses a causal-candidate prompt; the matched naming and reliability studies request direct parents. Four five-round batches per naming condition are paired with twenty data seeds. In the optical task, forty calls cover natural labels and random IDs with identical physical descriptions. Synthetic reports permit controlled changes in accuracy, edge retention, spurious claims, direction errors, and abstention. The naming intervention changes labels at a fixed domain and omits descriptions, measuring the role of variable semantics. Prompt definitions and the additional provider comparison appear in Appendices~\ref{app:implementation} and~\ref{app:source_robustness}.

\paragraph{Comparisons.}
The matched discrete study compares HB with no-KG Top-$K$ inference and a fixed-channel control ($w=.8$, $\epsilon=.05$). Each no-KG or fixed-channel fit makes one graph call; HB uses initialization and up to three report-driven updates, with twenty restarts per call. A further budget comparison increases each control to 80 restarts in one call, matching HB's maximum total number of restarts. The broader benchmark suite includes DiBS and DiBS+~\citep{lorch2021dibs}, GES~\citep{chickering2002optimal}, PC~\citep{spirtes2000causation}, MMHC~\citep{tsamardinos2006max}, bootstrap BPC and BGES~\citep{friedman1999data}, and Harmonized-Prior~\citep{ban2025harmonized}. KG-SoftMAP~\citep{xu2026kgsoftmap} uses greedy MAP search with a fixed soft edge prior; Random-KG-SoftMAP uses uniformly sampled edges. Continuous comparisons include NOTEARS~\citep{zheng2018dags}, GOLEM~\citep{ng2020golem}, greedy BGe search, DiBS variants, and no-KG Top-$K$. Harmonized-Prior~\citep{ban2025harmonized} uses a path-existence prior with fixed confidence on the same asserted Sachs pairs. The optical task uses four graph calls for all methods, including fixed-reliability HB with $\rho=.8$ and learned $B$. Table~\ref{tab:study_budgets} gives the study budgets.

\paragraph{Metrics and replication.}
We evaluate SHD, E-SHD, AUROC, report F1, and negative marginal log likelihood. AUROC uses posterior edge probabilities over off-diagonal ordered pairs; bootstrap edge frequencies provide scores for point-estimate continuous baselines. The benchmark suite counts a reversal once, and the matched, local Sachs, and ordered-adjacency studies count it twice. Appendix~\ref{app:implementation} defines the point estimators and E-SHD conventions. Negative I-MLL is query-averaged interventional predictive cross-entropy. Mean reliability $\bar\rho$ averages pair posterior means. The matched study averages four source batches within each of twenty data seeds for paired comparisons. Benchmark-suite summaries use five seed indices per network on cached observations; Sachs standard deviations describe inference restarts on one dataset. Each controlled study states its own replication units.

\subsection{Primary evidence: discrete Bayesian network recovery}
\label{sec:exp_discrete}

Report integration improves recovery even when the data-only baseline receives more search restarts per call. Increasing baseline search from 20 to 80 restarts lowers mean SHD from $23.96$ to $22.39$ for no-KG and from $16.27$ to $14.50$ for the fixed channel (Table~\ref{tab:search_budget_main}). On the same inputs, full HB gives $16.06$ with 20 restarts per graph call and at most 80 in total. Its SHD advantage over the larger no-KG search is $6.33$, with improvements on all five networks ($p_{\mathrm{Holm}}\leq.028$ across ten network--comparator tests). Against the larger fixed-channel search, HB has lower SHD on Asia; the fixed channel has lower mean SHD on the other four networks, with significant differences on Insurance and Alarm. The fixed channel also benefits from more search. Appendix~\ref{app:search_budget} gives the protocol, intervals, and secondary metrics.

\begin{table}[tbp]
\centering
\small
\caption{Mean SHD in the search-budget comparison. Each network uses 20 data seeds and four report batches. No-KG and the fixed channel use one graph call with the indicated restart count. Full HB uses 20 restarts per call, with at most four calls. All methods share the inputs of the reliability-assignment study; its full-HB fits provide the HB column.}
\label{tab:search_budget_main}
\begin{tabular}{lrrrrr}
\toprule
 & \multicolumn{2}{c}{No-KG} & \multicolumn{2}{c}{Fixed channel} & Full HB \\
\cmidrule(lr){2-3}\cmidrule(lr){4-5}
Network & 20 & 80 & 20 & 80 & $\leq80$ total \\
\midrule
Asia & $4.30$ & $3.25$ & $2.84$ & $2.68$ & $1.93$ \\
Sachs & $12.60$ & $11.20$ & $7.04$ & $6.68$ & $7.24$ \\
Child & $16.90$ & $12.20$ & $6.84$ & $6.83$ & $7.46$ \\
Insurance & $42.45$ & $40.10$ & $35.59$ & $31.23$ & $34.61$ \\
Alarm & $43.55$ & $45.20$ & $29.03$ & $25.10$ & $29.06$ \\
\midrule
Mean & $23.96$ & $22.39$ & $16.27$ & $14.50$ & $16.06$ \\
\bottomrule
\end{tabular}
\end{table}

The matched direct-parent study in Table~\ref{tab:matched_reports_main}(a) tests report integration with a common Top-$K$ backend and shared observations. Paired comparisons average four independently queried report batches within each of 20 data seeds. With canonical names, HB reduces mean SHD from $24.03$ without reports to $16.03$, improving on all five networks (each two-sided paired Wilcoxon $p<.001$, before multiplicity correction). The fixed-channel control averages $16.25$, giving HB $.22$ lower mean SHD across networks. After Holm correction across the five canonical-network SHD comparisons, the improvement over this control is significant only on Asia ($\Delta\mathrm{SHD}=-.91$, $p_{\mathrm{Holm}}=.0026$). Appendix~\ref{app:name_only} gives the complete paired comparisons and naming interventions.

Figure~\ref{fig:discrete_main} shows the paired contrasts from this matched study. The broader benchmark suite appears in Figure~\ref{fig:benchmark_suite_full}. It uses cached observations, causal-candidate reports, $\lambda_{\mathrm{warm}}=1$, and a reversal cost of one; the matched study uses newly sampled observations, $\lambda_{\mathrm{warm}}=.5$ (also used in the optical study), and a reversal cost of two. Their absolute SHD values therefore describe different protocols. The five-seed benchmark summaries give HB mean SHD $14.84$ and AUROC $.925$. Relative to DiBS-only, the differences are $-7.96$ SHD and $+.178$ AUROC; relative to MMHC, they are $-2.44$ and $+.222$. SHD measures a selected structure, whereas AUROC evaluates posterior edge rankings. Mean negative MLL decreases from $9.77$ for DiBS-only to $9.62$ for HB, compared with $9.81$ for MMHC (Table~\ref{tab:discrete_significance_full}). Interventional prediction results are included in the full benchmark figure.

\begin{figure}[tbp]
\centering
\includegraphics[width=\linewidth]{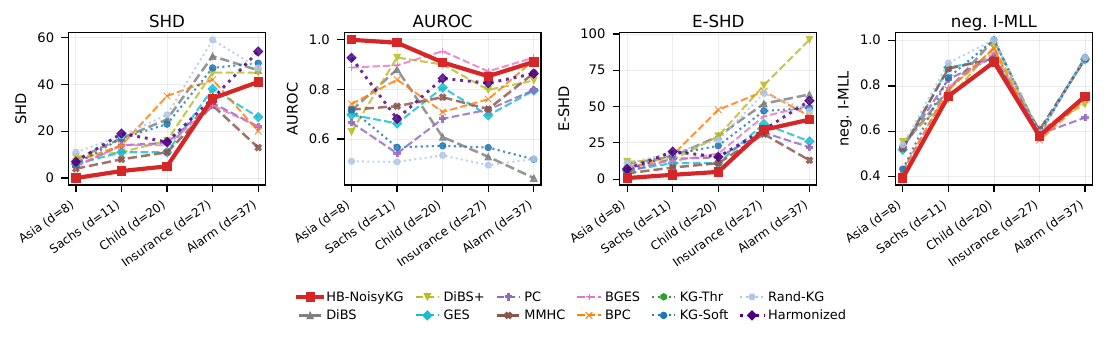}
\caption{Complete single-run benchmark comparison ($N=1000$): SHD, E-SHD, AUROC, and interventional prediction log loss. KG-Thr retains confidence-$\geq.70$ edges subject to acyclicity; KG-Soft and Rand-KG denote KG-SoftMAP and Random-KG-SoftMAP. Figure~\ref{fig:discrete_main} separately presents the matched, repeated-run report-integration comparison.}
\label{fig:benchmark_suite_full}
\end{figure}

\begin{table}[tbp]
\centering
\small
\caption{Matched direct-parent study ($N=1000$). (a) Structure recovery with canonical-name reports: HB and the fixed-channel control use identical reports and the same Top-$K$ backend. SHD averages four source batches and 20 shared data seeds; reversals cost two. The mean weights networks equally. (b) Direct-edge F1 of aggregated LLM reports under three naming conditions, averaged over four source batches. The domain is fixed and variable descriptions are absent in all conditions.}
\label{tab:matched_reports_main}
\begin{minipage}[t]{.47\linewidth}
\centering
\textit{(a) Structure recovery (SHD $\downarrow$)}\par\smallskip
\setlength{\tabcolsep}{4pt}
\begin{tabular}{lrrr}
\toprule
Network & No-KG & HB & Fixed channel \\
\midrule
Asia & $4.30$ & $1.93$ & $2.84$ \\
Sachs & $12.80$ & $7.25$ & $7.05$ \\
Child & $17.05$ & $7.45$ & $6.84$ \\
Insurance & $42.45$ & $34.49$ & $35.49$ \\
Alarm & $43.55$ & $29.05$ & $29.04$ \\
\midrule
Mean & $24.03$ & $16.03$ & $16.25$ \\
\bottomrule
\end{tabular}
\end{minipage}\hfill
\begin{minipage}[t]{.51\linewidth}
\centering
\textit{(b) Aggregated-report F1 $\uparrow$}\par\smallskip
\setlength{\tabcolsep}{3pt}
\begin{tabular}{lrrr}
\toprule
Network & Canonical & Anonymous & Pseudo \\
\midrule
Asia & $.591$ & $.000$ & $.000$ \\
Sachs & $.392$ & $.018$ & $.000$ \\
Child & $.480$ & $.000$ & $.000$ \\
Insurance & $.401$ & $.000$ & $.000$ \\
Alarm & $.355$ & $.000$ & $.228$ \\
\bottomrule
\end{tabular}
\end{minipage}
\end{table}

\subsection{External reports in a physical system}
\label{sec:exp_optical}

The light tunnel provides a physical task with an intervention-validated reference graph~\citep{gamella2025chambers}. We use BGe as a linear-Gaussian approximation to its nonlinear measurements. GPT-4.1-mini receives the same physical descriptions under natural labels or random IDs. Four batches of five rounds per presentation give forty calls. All methods share the observational blocks, graph prior, and four-call search budget; the reference graph supplies the evaluation labels.

The reports are sparse and imperfect: aggregated direct-edge F1 is $.407$ with natural labels and $.395$ with IDs, covering $9.9\%$ and $8.8\%$ of unordered pairs. At $N=1000$ with IDs, HB lowers mean SHD from $39.00$ to $27.30$ and raises AUROC from $.738$ to $.764$ (Table~\ref{tab:optical_main}). The paired SHD difference is $-11.70$, with an exploratory 95\% bootstrap interval $[-16.68,-6.43]$. HB also improves both means over no-KG at $N=250$. Report integration therefore improves recovery in this physical system under both variable presentations.

HB also improves mean SHD and AUROC over independently learned common reliability in all four conditions. Fixed-reliability HB has higher AUROC throughout and lower SHD at $N=250$; full HB has lower mean SHD at $N=1000$. The two criteria capture different uses of the inferred graph: selecting a structure and ranking possible edges. Appendix~\ref{app:optical} reports E-SHD, Brier score, and the paired intervals.

\begin{table}[tbp]
\centering
\small
\caption{Optical-system recovery. Fixed-$\rho$ HB fixes reliability at $.8$ and learns $B$. Both label conditions retain identical physical descriptions. Each report-based entry averages ten data blocks and four report batches. No-KG is fitted once per block and sample size, then reused across label conditions.}
\label{tab:optical_main}
\begin{tabular}{rlrrrr}
\toprule
 & & \multicolumn{2}{c}{Natural labels} & \multicolumn{2}{c}{Random IDs} \\
\cmidrule(lr){3-4}\cmidrule(lr){5-6}
$N$ & Model & SHD $\downarrow$ & AUROC $\uparrow$ & SHD $\downarrow$ & AUROC $\uparrow$ \\
\midrule
1000 & No-KG & $39.000$ & $.7383$ & $39.000$ & $.7383$ \\
 & Common reliability & $44.800$ & $.6961$ & $42.500$ & $.6992$ \\
 & Fixed-$\rho$ HB & $31.200$ & $.7872$ & $28.300$ & $.8182$ \\
 & Full HB & $29.550$ & $.7579$ & $27.300$ & $.7643$ \\
\midrule
250 & No-KG & $47.900$ & $.6938$ & $47.900$ & $.6938$ \\
 & Common reliability & $55.025$ & $.6347$ & $55.700$ & $.6328$ \\
 & Fixed-$\rho$ HB & $31.775$ & $.7856$ & $33.350$ & $.7994$ \\
 & Full HB & $36.250$ & $.7310$ & $37.600$ & $.7306$ \\
\bottomrule
\end{tabular}
\end{table}

\subsection{Reliability learning and knowledge use}
\label{sec:exp_mechanism}

Continued alternation reduces mean pair-reliability error from $.240$ to $.152$ and held-out report log loss from $.874$ to $.762$ relative to one-time estimation (Table~\ref{tab:feedback_main}). This controlled synthetic study also compares report-only estimation before graph search; all arms use the same graph-call budget. The twelve primary SHD contrasts have Holm-adjusted $p\geq.787$. At $\alpha=.75$, full HB gives SHD $2.40$, compared with $3.83$ for an intercept-only reliability prior. At $\alpha=0$, the corresponding values are $13.40$ and $7.53$, and no-KG gives $11.10$. Appendix~\ref{app:feedback_controls} reports all accuracy levels and the independent-feature controls.

\begin{table}[tbp]
\centering
\small
\caption{Continued reliability estimation with synthetic reports. Means average three networks, four generating accuracies, and ten data seeds per condition. Reliability MAE averages $|\E_q\rho_{ij}-\alpha|$ over pairs; log loss evaluates five independent held-out reports per pair. Lower is better in all columns.}
\label{tab:feedback_main}
\begin{tabular}{lrrr}
\toprule
Inference & SHD & Reliability MAE & Held-out log loss \\
\midrule
No-KG & $11.10$ & -- & -- \\
Report-only estimation & $7.83$ & $.246$ & $.983$ \\
One-time estimation & $7.00$ & $.240$ & $.874$ \\
Full alternation & $6.63$ & $.152$ & $.762$ \\
\bottomrule
\end{tabular}
\end{table}

\begin{table}[tbp]
\centering
\small
\setlength{\tabcolsep}{4pt}
\caption{Reliability controls and their purposes. Reused estimates are parameters obtained from a completed full-HB fit. The remaining controls are fitted independently.}
\label{tab:reliability_controls}
\begin{tabular}{@{}p{.18\linewidth}p{.29\linewidth}p{.15\linewidth}p{.30\linewidth}@{}}
\toprule
Control & Change & Reused estimates & Purpose \\
\midrule
No-KG & Omit report likelihood & None & Value of external reports \\
Fixed channel & Freeze a symmetric noise channel & None & Recovery with a supplied report channel \\
Fixed-$\rho$ HB & Fix $\rho=.8$; continue learning $B$ & None & Reliability estimation within HB's channel \\
Post-fit pooling & Replace pair weights by their observed-pair mean & Fitted reliability and $B$ & Variation among learned pair weights \\
Common reliability & Fit one shared Beta reliability & None & Independent estimation of source trust \\
Pair-intercept & Fit pair reliabilities with an intercept-only prior & None & Alternative prior features and initialization \\
One-time estimation & Freeze the observation model after its first update & None & Continued updating \\
Report-only estimation & Estimate unconstrained pair states before graph search & None & Feedback from graph inference \\
\bottomrule
\end{tabular}
\end{table}

With error-free positive reports, HB lowers SHD by $20.19$ versus no-KG across five networks. The fixed-channel control at $w=.99$ improves another $1.62$, and hard direct use reaches zero SHD. HB obtains about $93\%$ of the SHD improvement achieved by this fixed-channel control, estimating reliability from five reports per asserted edge. In the benchmark suite, replacing LLM reports with ground-truth reports lowers SHD by a further $4.00$ and raises AUROC by $.019$ (Table~\ref{tab:discrete_diag}(a)).

With reported confidence fixed at $.5$, mean inferred reliability increases with generating accuracy at every tested warm strength. Increasing $R$ from 2 to 20 reduces error in mean reliability from $.197$ to $.045$ on Asia and from $.187$ to $.043$ on Sachs (Table~\ref{tab:discrete_diag}(b)). At $\lambda_{\mathrm{warm}}=.5$, varying accuracy from zero to one changes final mean reliability by $.558$ and $.576$, respectively. At fixed accuracy, median ranges across warm strengths are $.009$ and $.011$. The preferred warm strength depends on the report condition: among the tested values, $.5$ minimizes maximum SHD regret on the accuracy grid, and the retention/spurious-report grid favors $.1$. Relative performance against a frozen channel also varies with accuracy. At $\alpha=0$ on Asia, HB gives SHD $10.0$, compared with $21.3$ for the fixed channel and $5.5$ for no-KG. On Sachs, the respective values are $19.5$, $19.8$, and $9.1$. At $\alpha\geq.6$, the fixed channel has $.2$--$1.4$ lower mean SHD; the fourteen corrected network--accuracy tests establish a difference only at $\alpha=0$ on Asia. Appendix~\ref{app:sensitivity} gives the complete warm-strength and accuracy grids.

The reliability-assignment study refits HB, giving mean SHD $16.06$, and holds its final $B$ posterior and graph-search budget fixed. Substituting pair-specific point means gives SHD $16.04$; pooling those means over reported pairs gives $16.27$. Pooled weights have higher AUROC ($.921$ versus $.917$), and the corrected network tests establish no pair-specific SHD advantage. Thus similar structural recovery can be retained after fitted pair weights are averaged. These controls use HB's learned parameters; Appendix~\ref{app:reliability_pooling} gives their complete comparison.

Independent estimation uses a new fit of each model with exactly three graph updates (Table~\ref{tab:common_main}). Full HB gives mean SHD $16.26$ with the canonical reports, versus $19.92$ for one learned source reliability. Alarm and Insurance account for the mean advantage, with reductions of $9.00$ and $10.25$ (Holm-adjusted $p=.00027$ and $.000048$). The other three corrected tests are inconclusive. A pair-specific model with an intercept-only prior gives mean SHD $23.42$. The fitted trust levels also differ. On Insurance, common reliability averages $.219$ and HB's observed-pair mean is $.602$; their SHD values are $45.84$ and $35.59$, compared with $42.45$ for the shared data-only initialization. Alarm's reliability estimates are $.427$ and $.657$. Independent fitting can therefore give different outcomes from averaging parameters after HB has learned them. Each variant uses its specified prior and initialization, detailed with the network results in Appendix~\ref{app:common_reliability}.

The synthetic source shares one generating accuracy across all pairs. At $\alpha=0$, common reliability gives SHD $.6$ on Asia and $.9$ on Sachs. Data-only initialization gives $5.5$ and $9.1$, and full HB gives $9.8$ and $19.5$; both HB--common contrasts have Holm-adjusted $p=.027$. Here the generator always selects a state other than the truth, so repeated labels provide exclusion information. A single source reliability fits this homogeneous error mechanism well. The caption of Table~\ref{tab:common_main} distinguishes this endpoint from the other six accuracy levels.

\begin{table}[tbp]
\centering
\small
\caption{Independently fitted reliability models. All variants use one initialization and three graph updates. Left: equal-network means on the five-network LLM study. Right: equal-network and equal-accuracy means over seven synthetic accuracy levels, including $\alpha=0$. That endpoint drives the overall SHD difference: across the other six levels, mean SHD is $3.33$ for HB and $3.35$ for common reliability. These are separate refits from Table~\ref{tab:matched_reports_main}.}
\label{tab:common_main}
\begin{tabular}{lrrrr}
\toprule
 & \multicolumn{2}{c}{LLM reports} & \multicolumn{2}{c}{Homogeneous synthetic reports} \\
\cmidrule(lr){2-3}\cmidrule(lr){4-5}
Model & SHD $\downarrow$ & AUROC $\uparrow$ & SHD $\downarrow$ & AUROC $\uparrow$ \\
\midrule
Common reliability & $19.92$ & $.906$ & $2.98$ & $.950$ \\
Pair-specific, intercept prior & $23.42$ & $.873$ & $4.52$ & $.916$ \\
Full HB, feature prior & $16.26$ & $.916$ & $4.94$ & $.914$ \\
\bottomrule
\end{tabular}
\end{table}

The controls in Table~\ref{tab:reliability_controls} compare report integration, reliability assignment, and continued updating. In the component analysis (Table~\ref{tab:discrete_diag}(a)), removing local search increases SHD by $6.36$; removing the warm start increases it by $1.60$ and lowers AUROC by $.053$. Updating $B$ and $(\beta,\tau)$ has smaller incremental effects. Those two controls continue updating pair reliability, so they measure the additions of shared-error and hyperparameter estimation within HB. Correct reports receive mean reliability $.24$ higher than incorrect reports (panel~b).

\begin{table}[tbp]
\caption{Mechanism analysis. (a) Component shifts relative to HB-NoisyKG in the discrete benchmark suite. The matched fixed-channel comparison is in Table~\ref{tab:matched_reports_main}(a). (b) Reliability separation, report-count effects on mean reliability estimation, and accurate-report summaries from the studies described in Appendix~\ref{app:discrete_full}.}
\label{tab:discrete_diag}
\centering
\scriptsize
\setlength{\tabcolsep}{3pt}
\begin{minipage}[t]{0.52\textwidth}
\begin{tabular*}{\linewidth}[t]{@{}p{0.52\linewidth}@{\extracolsep{\fill}}c c@{}}
\toprule
\textit{(a)}\quad Variant & $\Delta$SHD & $\Delta$AUROC \\
\midrule
Synthetic GT KG (oracle)  & $-4.00$ & $+0.019$ \\
\midrule
No confusion update       & $+0.12$ & $-0.006$ \\
No $\beta,\tau$ update    & $+0.20$ & $-0.006$ \\
No KG warm start          & $+1.60$ & $-0.053$ \\
No local search           & $+6.36$ & $-0.028$ \\
\bottomrule
\end{tabular*}
\end{minipage}\hfill
\begin{minipage}[t]{0.45\textwidth}
\begin{tabular*}{\linewidth}[t]{@{}p{0.58\linewidth} p{0.34\linewidth}@{}}
\toprule
\textit{(b)}\quad Metric & Value \\
\midrule
MAE, Asia ($R=2\to20$) & $.197\to.045$ \\
MAE, Sachs ($R=2\to20$) & $.187\to.043$ \\
Accurate reports: SHD gain & $20.19$ \\
correct--incorrect $\rho$ gap     & $+0.24$ \\
mean neg.\ MLL (DiBS$\rightarrow$HB) & $9.77 \rightarrow 9.62$ \\
\addlinespace[5pt]
\bottomrule
\end{tabular*}
\end{minipage}
\end{table}

\subsection{Robustness under knowledge degradation}
\label{sec:exp_robustness}

We corrupt initially error-free synthetic reports on five networks, using 1,000 observations and three sampling seeds per setting. HB and no-KG share the observations and per-call search settings. Table~\ref{tab:corruption_main} reports the clean condition and severity endpoints. HB retains lower mean SHD throughout the missing/spurious grid. Severe direction reversals and mixed errors exceed no-KG, illustrating the importance of report quality. At intermediate mixed corruption, SHD improves by $2.60$ but AUROC falls to $.848$. Mean reliability ranges from $.42$ to $.55$ across settings, averaging reported and unreported pairs. Figure~\ref{fig:discrete_robustness_full} gives the full curves. A separate stability study extends inference to fifty outer updates over 5,508 settings with weak data and inaccurate reports. Every run remains finite, and two Top-$K$ cases repeatedly alternate between two states. Tolerance attainment and recovery quality vary by backend and report accuracy (Appendix~\ref{app:computational_stability}).

\begin{table}[tbp]
\centering
\small
\caption{Selected settings from the paired corruption study. No-KG has mean SHD $17.93$ and AUROC $.880$. Differences are HB minus no-KG; lower SHD and higher AUROC are better. Complete settings are in Table~\ref{tab:corruption_nokg_all}.}
\label{tab:corruption_main}
\begin{tabular}{lrrr}
\toprule
Report setting & HB SHD & $\Delta$SHD & HB AUROC \\
\midrule
Clean & $2.80$ & $-15.13$ & $.969$ \\
Missing, retention $.3$ & $7.93$ & $-10.00$ & $.925$ \\
Spurious, rate $.5$ & $2.93$ & $-15.00$ & $.969$ \\
Flipped, rate $.7$ & $19.13$ & $+1.20$ & $.808$ \\
Mixed, level $.9$ & $19.80$ & $+1.87$ & $.830$ \\
\bottomrule
\end{tabular}
\end{table}

At accuracy $.75$, heterogeneous pair-group errors combined with correlated reports give SHD reductions of $2.6$ on Asia and $5.9$ on Sachs. With $75\%$ independent abstention, gains remain $2.6$ and $1.2$. On Sachs at the same coverage, preferentially withholding incorrect reports preserves a $5.8$ improvement; withholding correct reports increases SHD by $1.3$. Source utility therefore depends on the quality of surviving reports as well as coverage. Complete abstention recovers no-KG exactly. For fully incorrect reports, increasing $R$ from two to twenty reduces excess SHD from $5.0$ to $1.8$ on Asia and from $10.0$ to $1.2$ on Sachs.

\label{sec:exp_source_diagnostics}
Removing canonical names in the absence of variable descriptions sharply reduces report F1 (Table~\ref{tab:matched_reports_main}(b)). Names provide semantic information, and familiarity with public benchmarks acquired during pretraining may also improve report quality. Surviving incorrect anonymous reports on Sachs give SHD $13.31$, compared with $12.80$ for no-KG. The light-tunnel study provides complementary evidence with random IDs and retained physical descriptions: at $N=1000$, HB lowers SHD from $39.00$ without reports to $27.30$ (Table~\ref{tab:optical_main}). A separate prompt intervention holds presentation and sampling fixed and requests direct parents, directed paths, or nonzero total effects. Across ten parameterized SCMs and 150 calls, target-matched F1 is $1.000$, $1.000$, and $.945$. Total-effect prompts select all five reachable pairs with cancelling effects. These examples make the elicitation target consequential: graph reachability and numerical total effect give different labels. Appendix~\ref{app:prompt_semantics} describes the prompts and truth definitions.

\subsection{Cross-family portability and continuous inference}
\label{sec:exp_continuous}

We evaluate continuous Sachs data under graph-only BGe, linear joint, and nonlinear joint inference (Table~\ref{tab:sachs_continuous}). A separate ER study uses $d\in\{20,30\}$, expected degree 2, and 500 observations per dataset. Its generated graphs and synthetic reports provide a setting independent of LLM benchmark familiarity (Appendix Table~\ref{tab:e2_synth}).
\begin{table}[t]
\caption{Continuous Sachs results. Our rows use a shared condition-standardized protocol (mean $\pm$ std over 20/10/10 restarts for graph-only / linear-joint / nonlinear-joint panels). MC3 and BPC are reproduced from \citet[Table~1]{lorch2021dibs} (30 restarts). $\bar\rho$ is mean inferred reliability. DiBS+ coincides with DiBS in the joint panels.}
\label{tab:sachs_continuous}
\centering
\scriptsize
\begin{tabular}{l c c c c}
\toprule
Method & AUROC $\uparrow$ & E-SHD $\downarrow$ & SHD $\downarrow$ & $\bar{\rho}$ \\
\midrule
\multicolumn{5}{l}{\textit{Continuous graph-only (BGe marginal likelihood)}} \\
MC3 (published) & $0.616 \pm 0.027$ & $34.0 \pm 0.7$ & -- & -- \\
BPC (published) & $0.566 \pm 0.020$ & $25.5 \pm 2.3$ & -- & -- \\
NOTEARS & $0.590 \pm 0.012$ & $18.10 \pm 0.59$ & $18.00 \pm 1.33$ & -- \\
GOLEM & $0.586 \pm 0.011$ & $18.63 \pm 0.74$ & $19.00 \pm 1.89$ & -- \\
BGES-style (continuous BGe) & $0.655 \pm 0.000$ & $15.00 \pm 0.00$ & $15.00 \pm 0.00$ & -- \\
DiBS & $0.617 \pm 0.035$ & $16.99 \pm 0.08$ & $16.00 \pm 1.63$ & -- \\
DiBS+ & $0.565 \pm 0.032$ & $16.08 \pm 1.89$ & $16.00 \pm 1.63$ & -- \\
Top-$K$ MAP (no KG) & $0.702 \pm 0.009$ & $18.13 \pm 0.50$ & $19.00 \pm 2.83$ & -- \\
Harmonized-Prior (fixed $c$) & $0.649 \pm 0.016$ & $17.45 \pm 1.12$ & $17.45 \pm 1.12$ & -- \\
\textbf{HB-NoisyKG} (graph-only) & $\mathbf{0.823 \pm 0.012}$ & $\mathbf{10.38 \pm 0.61}$ & $\mathbf{10.75 \pm 1.65}$ & $0.460$ \\
\midrule
\multicolumn{5}{l}{\textit{Continuous linear joint}} \\
DiBS (linear joint) & $0.572 \pm 0.035$ & $18.40 \pm 2.00$ & $18.40 \pm 2.00$ & -- \\
Harmonized-Prior (linear joint) & $0.573 \pm 0.035$ & $17.80 \pm 2.04$ & $17.80 \pm 2.04$ & -- \\
\textbf{HB-NoisyKG} (linear joint, KG) & $\mathbf{0.584 \pm 0.038}$ & $\mathbf{17.20 \pm 1.70}$ & $\mathbf{17.20 \pm 1.70}$ & $0.413$ \\
\midrule
\multicolumn{5}{l}{\textit{Continuous nonlinear joint}} \\
DiBS (nonlinear joint) & $0.548 \pm 0.037$ & $28.25 \pm 2.63$ & $28.25 \pm 2.63$ & -- \\
Harmonized-Prior (nonlinear joint) & $0.517 \pm 0.051$ & $29.40 \pm 4.25$ & $29.40 \pm 4.25$ & -- \\
\textbf{HB-NoisyKG} (nonlinear joint, KG) & $\mathbf{0.579 \pm 0.050}$ & $\mathbf{25.75 \pm 4.79}$ & $\mathbf{25.75 \pm 4.79}$ & $0.406$ \\
\bottomrule
\end{tabular}
\end{table}
With BGe scoring and the same Top-$K$ backend, incorporating the Sachs reports raises mean AUROC from $.702$ to $.823$ and lowers E-SHD from $18.13$ to $10.38$. Harmonized-Prior uses the same asserted pairs and gives AUROC $.649$. NOTEARS, GOLEM, and greedy BGe search are also below HB on this metric.
In the linear- and nonlinear-joint regimes, all rows use DiBS-joint SVGD; the knowledge-aware rows receive the same LLM-asserted pairs. HB has numerically higher mean AUROC and lower mean E-SHD than both no-KG DiBS and Harmonized-Prior. The standard deviations describe variation across inference restarts on the same dataset. Harmonized-Prior is close to DiBS in the linear setting and has lower AUROC and higher E-SHD in the nonlinear setting. HB's mean AUROC gains over DiBS are $.012$ and $.031$ for linear and nonlinear joint inference, respectively; the graph-only gain over no-KG Top-$K$ is $.121$.

Across the four ER settings, mean E-SHD is $1.39$ to $4.70$ lower for HB, with graph-level Holm-adjusted $p$-values from $.25$ to $.50$. These comparisons use their study-specific base priors and graph-update budgets. Appendix~\ref{app:continuous_synth} gives the full methods comparison, larger-budget examples, and accurate-report results.


\subsection{Computational cost}
\label{sec:exp_computational}
Graph inference dominates runtime; the observation update costs $O(d^2R)$. Table~\ref{tab:large_discrete_main} reports sparse binary experiments with 100 and 125 variables over ten paired seeds. HB lowers SHD by $47.9$ and $55.4$ at roughly three times no-KG runtime, using $.202$ and $.254$ GiB resident memory. The added cost reflects repeated graph updates, and the report layer remains usable beyond one hundred variables.

\begin{table}[tbp]
\centering
\small
\caption{Larger sparse binary DAGs: ten paired seeds, $N=1000$, $\alpha=.75$, $R=5$, 20 restarts per graph call and up to three HB updates. Entries give no-KG/HB means. RSS is measured HB memory.}
\label{tab:large_discrete_main}
\begin{tabular}{crrr}
\toprule
$d$ & SHD $\downarrow$ & Seconds & RSS (GiB) \\
\midrule
100 & $101.1/53.2$ & $1603/4668$ & $.202$ \\
125 & $130.2/74.8$ & $4089/12391$ & $.254$ \\
\bottomrule
\end{tabular}
\end{table}

\section{Conclusion}
\label{sec:conclusion}

HB-NoisyKG provides a common observation likelihood for repeated causal reports of unknown reliability. The same layer combines report evidence with discrete or continuous data and supports both graph-only and joint inference. Matched benchmark comparisons and the physical light-tunnel task demonstrate structural gains over data-only inference. The controls distinguish learning the level of trust from assigning fitted weights across pairs. In the controlled synthetic study, continued updating reduces reliability estimation error and improves held-out report prediction.

Practitioners can use partial external reports alongside observations and estimate their influence during inference. The elicitation target should match the modeled relation, and coverage and source quality determine the information available. Completely incorrect reports remain challenging; absent reports recover data-only inference. The experiments give practical guidance on report count and initialization and demonstrate sparse recovery at 125 variables. Extending the observation interface to richer graph classes requires corresponding pair states and a suitable graph posterior, as illustrated by the controlled studies in Appendix~\ref{app:graph_classes}.


\clearpage


\appendix

\section{Proofs and conditional bounds}
\label{app:proofs}

\begin{proof}[Proof of Proposition~\ref{prop:data_only_reduction}]
If $\psi_{ij}(z)=0$ for every unordered pair $\{i,j\}$ and every latent state $z$, then
\[
\sum_{i<j} \psi_{ij}(Z_{ij}(G)) = 0
\qquad \text{for every graph } G.
\]
Therefore the exponential reward factor in Equations~\eqref{eq:mstep_graph_only} and~\eqref{eq:mstep_joint} is identically equal to one, and the graph update reduces to
\[
q(G) \propto p_0(G)\,p(\D\mid G)
\]
in the collapsed graph-only family and
\[
q(G,\vtheta) \propto p_0(G)\,p(\vtheta\mid G)\,p(\D\mid G,\vtheta)
\]
in the explicit joint family, which are precisely the corresponding data-only Bayesian posteriors.

For the abstention special case, if every knowledge report abstains, then by construction there are no explicit pairwise labels, so $n_{ij}(y)=0$ for all $y$ and all pairs. Hence $R_{ij}=0$ and Equation~\eqref{eq:psi_raw} gives $\psi^{\mathrm{raw}}_{ij}(z)=0$ for all $z$. Centering preserves zero, so $\psi_{ij}(z)=0$ for all pairs and the reduction above applies.
\end{proof}

\begin{proposition}[Pairwise-centering invariance]
\label{prop:centering_invariance}
Fix arbitrary constants $c_{ij}\in\mathbb{R}$. Replacing every reward table $\psi_{ij}(z)$ by $\psi'_{ij}(z)=\psi_{ij}(z)+c_{ij}$ leaves all graph rankings under Equations~\eqref{eq:mstep_graph_only} and~\eqref{eq:mstep_joint} unchanged.
\end{proposition}

\begin{proof}[Proof of Proposition~\ref{prop:centering_invariance}]
Let $\psi'_{ij}(z)=\psi_{ij}(z)+c_{ij}$ with constants $c_{ij}$ independent of $z$. For any graph $G$,
\[
\sum_{i<j} \psi'_{ij}(Z_{ij}(G))
=
\sum_{i<j} \psi_{ij}(Z_{ij}(G))
+
\sum_{i<j} c_{ij}.
\]
The second term does not depend on $G$, because each unordered pair contributes exactly one latent state under any graph. Thus in either Equation~\eqref{eq:mstep_graph_only} or Equation~\eqref{eq:mstep_joint}, replacing $\psi$ by $\psi'$ multiplies every graph weight by the same positive constant
\[
\exp\!\left(\sum_{i<j} c_{ij}\right).
\]
After normalization, the posterior ranking over graphs is unchanged.
\end{proof}

\begin{proposition}[Finite influence of noisy knowledge]
\label{prop:finite_influence}
Assume $a_{ij}>0$, $b_{ij}>0$, and $\alpha_{zy}>0$ for all off-diagonal confusion entries, and suppose each pair has finite non-abstain count $R_{ij}<\infty$. Then every centered reward $\psi_{ij}(z)$ is finite, and for each pair there exists a constant $C_{ij}<\infty$ such that $\max_z |\psi_{ij}(z)| \le C_{ij}$.
\end{proposition}

\begin{proof}[Proof of Proposition~\ref{prop:finite_influence}]
Fix an unordered pair $\{i,j\}$. By assumption, the Beta posterior parameters satisfy $a_{ij}>0$ and $b_{ij}>0$, and the off-diagonal Dirichlet parameters satisfy $\alpha_{zy}>0$ for all $y\neq z$. Since the digamma function is finite on $(0,\infty)$, the quantities
\[
\ell^{(\rho)}_{ij}
:=
\E_q[\log \rho_{ij}]
=
\psi(a_{ij})-\psi(a_{ij}+b_{ij}),
\]
\[
\ell^{(1-\rho)}_{ij}
:=
\E_q[\log (1-\rho_{ij})]
=
\psi(b_{ij})-\psi(a_{ij}+b_{ij}),
\]
and
\[
\ell^{(B)}_{zy}
:=
\E_q[\log B_{zy}]
=
\psi(\alpha_{zy})-\psi\!\left(\sum_{y'\neq z}\alpha_{zy'}\right)
\]
are all finite real numbers.

Define the finite pairwise constants
\[
L^{(\rho)}_{ij} := |\ell^{(\rho)}_{ij}|,\qquad
L^{(1-\rho)}_{ij} := |\ell^{(1-\rho)}_{ij}|,
\]
and
\[
L^{(B)}_{ij} := \max_{z}\max_{y\neq z} |\ell^{(B)}_{zy}|.
\]
Using Equation~\eqref{eq:psi_raw}, for any candidate state $z$ we have
\begin{align*}
\big|\psi^{\mathrm{raw}}_{ij}(z)\big|
&\le
n_{ij}(z) L^{(\rho)}_{ij}
 + \big(R_{ij}-n_{ij}(z)\big)L^{(1-\rho)}_{ij}
 + \sum_{y\neq z} n_{ij}(y) L^{(B)}_{ij} \\
&\le
R_{ij}L^{(\rho)}_{ij}
 + R_{ij}L^{(1-\rho)}_{ij}
 + R_{ij}L^{(B)}_{ij}.
\end{align*}
Therefore, if we define
\[
M_{ij}:=
R_{ij}\big(L^{(\rho)}_{ij}+L^{(1-\rho)}_{ij}+L^{(B)}_{ij}\big),
\]
then $\max_z |\psi^{\mathrm{raw}}_{ij}(z)| \le M_{ij}<\infty$.

Now recall that centered rewards are defined by
\[
\psi_{ij}(z)=\psi^{\mathrm{raw}}_{ij}(z)-\frac13\sum_{z'}\psi^{\mathrm{raw}}_{ij}(z').
\]
Hence
\begin{align*}
|\psi_{ij}(z)|
&\le
\big|\psi^{\mathrm{raw}}_{ij}(z)\big|
+
\frac13\sum_{z'}\big|\psi^{\mathrm{raw}}_{ij}(z')\big| \\
&\le
M_{ij}+M_{ij}
=2M_{ij}.
\end{align*}
Thus every centered reward is finite, and Proposition~\ref{prop:finite_influence} holds with, for example,
\[
C_{ij}:=2R_{ij}\big(L^{(\rho)}_{ij}+L^{(1-\rho)}_{ij}+L^{(B)}_{ij}\big) < \infty.
\]
\end{proof}

\begin{proof}[Proof of Proposition~\ref{prop:reliability_concentration}]
\emph{Step 1: Binomial likelihood.}
Under assumption~(a), every non-abstain report is generated with true latent pair state $z^*_{ij}$. By the observation model in Section~\ref{sec:reliability_posterior},
\[
\Pr\!\big(Y^{(r)}_{ij}=z^*_{ij}\mid \rho^*_{ij},B\big)
= \rho^*_{ij}\cdot 1 + (1-\rho^*_{ij})\cdot B_{z^*_{ij},z^*_{ij}}
= \rho^*_{ij},
\]
because $B$ has zero diagonal. Assumption~(b) then gives
\[
n^+_{ij}\mid R_{ij}\sim\mathrm{Binomial}(R_{ij},\rho^*_{ij}).
\]

\emph{Step 2: Posterior mean consistency.}
The $\mathrm{Beta}(\alpha,\beta)$ mean is $\alpha/(\alpha+\beta)$, so
\[
\E_q[\rho_{ij}]
= \frac{a_{ij}+n^+_{ij}}{a_{ij}+b_{ij}+R_{ij}}
= \frac{a_{ij}/R_{ij}+n^+_{ij}/R_{ij}}{(a_{ij}+b_{ij})/R_{ij}+1}.
\]
The constant prior terms $a_{ij}/R_{ij}$ and $(a_{ij}+b_{ij})/R_{ij}$ vanish as $R_{ij}\to\infty$. The weak law of large numbers applied to the i.i.d.\ indicators $\mathbb{I}[Y^{(r)}_{ij}=z^*_{ij}]$ gives $n^+_{ij}/R_{ij}\to\rho^*_{ij}$ in probability. The continuous mapping theorem then yields $\E_q[\rho_{ij}]\to\rho^*_{ij}$ in probability.

\emph{Step 3: Variance bound.}
With $\alpha=a_{ij}+n^+_{ij}$ and $\beta=b_{ij}+R_{ij}-n^+_{ij}$,
\[
\mathrm{Var}_q[\rho_{ij}]
= \frac{\alpha\beta}{(\alpha+\beta)^2(\alpha+\beta+1)}.
\]
Using $\alpha\beta\le(\alpha+\beta)^2/4$,
\[
\mathrm{Var}_q[\rho_{ij}]
\le \frac{1}{4(\alpha+\beta+1)}
= \frac{1}{4(a_{ij}+b_{ij}+R_{ij}+1)}
= O(1/R_{ij}).
\]
Taking square roots yields the $O(1/\sqrt{R_{ij}})$ rate on the posterior standard deviation.
\end{proof}

\paragraph{A conditional finite-sample bound.}
The assumption of exact pair-state concentration can be relaxed. Fix a pair with $R_{ij}>0$ conditionally i.i.d.\ reports and true state $z^*_{ij}$. Let $a_{ij},b_{ij}>0$ be the prior parameters used in this update, set $c_{ij}=a_{ij}+b_{ij}$, and suppose $q_{ij}(z^*_{ij})\geq1-\epsilon_{ij}$. The expected match count is $m_{ij}=\sum_z q_{ij}(z)n_{ij}(z)$, giving posterior mean $\widehat\rho_{ij}=(a_{ij}+m_{ij})/(c_{ij}+R_{ij})$. For this fixed pair, with probability at least $1-\delta$ over the reports, the following bound holds for any marginal satisfying the stated concentration condition:
\begin{equation}
\begin{aligned}
\bigl|\widehat\rho_{ij}-\rho^*_{ij}\bigr|
\leq {} & \frac{2\epsilon_{ij}R_{ij}}{c_{ij}+R_{ij}}
+\frac{R_{ij}}{c_{ij}+R_{ij}}
\sqrt{\frac{\log(2/\delta)}{2R_{ij}}} \\
&+\frac{|a_{ij}-c_{ij}\rho^*_{ij}|}{c_{ij}+R_{ij}}.
\end{aligned}
\label{eq:finite_report_bound}
\end{equation}
To see this, let $n^+_{ij}=n_{ij}(z^*_{ij})$. Since $\|q_{ij}-\delta_{z^*_{ij}}\|_1\leq2\epsilon_{ij}$ and each count is at most $R_{ij}$, $|m_{ij}-n^+_{ij}|\leq2\epsilon_{ij}R_{ij}$. Hoeffding's inequality bounds $|n^+_{ij}/R_{ij}-\rho^*_{ij}|$ by the square-root term with probability at least $1-\delta$. Substituting these bounds into
\[
\widehat\rho_{ij}-\rho^*_{ij}
=\frac{m_{ij}-n^+_{ij}}{c_{ij}+R_{ij}}
+\frac{n^+_{ij}-R_{ij}\rho^*_{ij}}{c_{ij}+R_{ij}}
+\frac{a_{ij}-c_{ij}\rho^*_{ij}}{c_{ij}+R_{ij}}
\]
proves Equation~\eqref{eq:finite_report_bound}. Its three terms separate graph-marginal error, report-sampling error, and prior bias. For fixed prior parameters, the bound recovers consistency as $\epsilon_{ij}\to0$ and $R_{ij}\to\infty$.

\paragraph{Dependent reports and heterogeneous errors.}
With the true pair state fixed, the Beta update counts matches and mismatches irrespective of the distribution among wrong labels. Heterogeneous off-diagonal errors affect joint inference through their influence on the graph posterior. Correlated reports provide less information: under a specified mixing model, a concentration inequality can use an effective sample size, while general dependence requires its own assumptions. Appendix~\ref{app:error_patterns} examines heterogeneity and dependence separately and jointly.

\paragraph{Conditional ascent of the inference objective.}
Write $\mathcal{L}(q,\beta,\tau)$ for the variational lower bound on the data and report likelihood, including $\log p(\beta)+\log p(\tau)$ for the MAP empirical-Bayes step. Holding the other blocks fixed, exact maximization of a variational or parameter block cannot decrease $\mathcal{L}$. The same statement holds for an inexact block update that is accepted only when it improves this objective. Applying this inequality successively to the graph, reliability, confusion, and empirical-Bayes blocks establishes conditional coordinate ascent. If the objective is bounded above, its values converge. Parameter convergence and graph recovery require additional conditions.

Our Top-$K$ and SVGD graph steps approximate the graph block. Their behavior is therefore assessed empirically in Appendix~\ref{app:computational_stability}. With weak observational evidence, repeated inaccurate reports can support an incorrect graph explanation even when the numerical updates remain well behaved. Approximate graph updates may also fail to improve the objective at every iteration. The audit distinguishes numerical failure, persistent oscillation, tolerance attainment, and statistical harm. Iteration limits and numerical or cycle monitoring provide practical diagnostics, with data-only inference available as a fallback.

\section{Implementation details}
\label{app:implementation}

\paragraph{LLM-KG construction.}
The main KG construction uses GPT-4.1-mini with structured prompts requesting JSON-formatted causal edges.
Each of $R{=}5$ rounds uses temperatures in $\{0.65, 0.70, 0.75, 0.80, 0.85\}$.
The minimum round threshold for edge retention is~2, and the existence threshold is $c_{\text{exist}} \geq 0.30$.
Direction uncertainty threshold is $c_{\text{dir}} < 0.65$.

The benchmark-suite prompt requested ``directed prerequisite/causal candidate edges between variables.'' The prompt-target study makes the target explicit while retaining the model, variable presentation, JSON schema, temperatures, query rounds, and aggregation rule within each comparison:
\begin{itemize}
    \item \emph{Direct parent:} report $i\to j$ when $i$ is an immediate causal parent of $j$; exclude relations supported only by mediators.
    \item \emph{Path existence:} report $i\to j$ when at least one directed path from $i$ to $j$ exists, including both direct and mediated relations.
    \item \emph{Total effect:} report $i\to j$ when intervening on $i$ has a nonzero total effect on $j$, whether direct or mediated.
\end{itemize}
The benchmark-suite results use that causal-candidate prompt. Appendix~\ref{app:prompt_semantics} reports the explicit target comparisons; Appendix~\ref{app:name_only} uses the explicit direct-parent prompt for all naming conditions.

\paragraph{Repeated reports, features, and conjugate updates.}
For the principal LLM studies, each source batch has five query rounds; synthetic studies specify $R$ separately. For one pair, let $n_{\text{abstain}}, n_{\text{none}}, n_{\text{fwd}}, n_{\text{rev}}$ denote the round counts, $n_{\text{obs}} = R - n_{\text{abstain}}$, and $n_{\text{present}} = n_{\text{fwd}} + n_{\text{rev}}$. The feature vector $\phi_{ij}$ uses \texttt{raw\_conf\_mean} and \texttt{raw\_conf\_max} (mean/max confidence over directional rounds), \texttt{coverage\_rate}$=n_{\text{obs}}/R$, \texttt{exist\_agree}$=n_{\text{present}}/\max(n_{\text{obs}},1)$, \texttt{dir\_agree}$=\max(n_{\text{fwd}},n_{\text{rev}})/\max(n_{\text{present}},1)$, \texttt{forward\_rate}$=n_{\text{fwd}}/\max(n_{\text{obs}},1)$, \texttt{reverse\_rate}$=n_{\text{rev}}/\max(n_{\text{obs}},1)$, \texttt{abstain\_rate}$=n_{\text{abstain}}/R$, \texttt{explicit\_none\_rate}$=n_{\text{none}}/\max(n_{\text{obs}},1)$, \texttt{conflict\_rate}$=n_{\text{conflict}}/\max(n_{\text{obs}},1)$, and the raw count features \texttt{rounds\_observed}$=n_{\text{obs}}$ and \texttt{rounds\_present}$=n_{\text{present}}$.

Unless otherwise noted, the off-diagonal confusion prior uses $\gamma_{0,zy}=1$ for $y\neq z$, the reliability hyperpriors use $\beta \sim \mathcal{N}(0, 4I)$ and $\tau \sim \mathrm{Gamma}(2.0, 0.25)$, and EM starts from $\tau=8.0$ together with the fixed default $\beta$ vector used in the implementation: intercept $-0.5$, \texttt{raw\_conf\_mean} $1.0$, \texttt{coverage\_rate} $0.4$, \texttt{exist\_agree} $0.6$, \texttt{dir\_agree} $1.0$, \texttt{explicit\_none\_rate} $-0.5$, \texttt{conflict\_rate} $-1.0$, and all remaining feature weights $0$.

Let $n_{ij}(y)$ be the non-abstain count for $y\in\{\text{none}, i\to j, j\to i\}$ and let $R_{ij}=\sum_y n_{ij}(y)$. Given pair-state marginals $q_{ij}(z)$, the expected match count is $m_{ij}=\sum_z q_{ij}(z)\,n_{ij}(z)$, so
\[
q(\rho_{ij}) = \mathrm{Beta}(a_{ij}, b_{ij}), \qquad
a_{ij} = \tau \mu_{ij} + m_{ij}, \qquad
b_{ij} = \tau(1-\mu_{ij}) + (R_{ij} - m_{ij}).
\]
For the shared error matrix $B$, the diagonal is fixed to zero and the off-diagonal Dirichlet parameters update as
\[
\alpha_{zy} = \gamma_{0,zy} + \sum_{i<j} q_{ij}(z)\, n_{ij}(y), \qquad y \neq z.
\]
Equivalently, before any graph-weighted counts are added, the posterior parameters start from the prior pseudo-counts $a_{ij}^{(0)}=\tau\mu_{ij}$, $b_{ij}^{(0)}=\tau(1-\mu_{ij})$, and $\alpha_{zy}^{(0)}=\gamma_{0,zy}$.

\paragraph{Practical instantiation details.}
\emph{ABSTAIN handling.} The observed symbols are \{ABSTAIN, NONE, FWD, REV\}, and the latent states are \{NONE, FWD, REV\}. ABSTAIN records missing evidence. The remaining symbols contribute counts to reliability, error-matrix, and reward updates.

\emph{Warm initialization.} A diffuse initial graph posterior can make the first E-step interpret many reports as mismatches. We warm-start this update by blending the data-only pair-state posterior with empirical report frequencies on observed pairs,
\[
q_{ij}^{\mathrm{warm}} = (1-\alpha_{ij})\, q_{ij}^{\mathrm{data}} + \alpha_{ij}\, q_{ij}^{\mathrm{KG}},
\]
where $q_{ij}^{\mathrm{KG}}$ is the empirical three-state frequency vector over \{none, $i\to j$, $j\to i$\}. Let $A_{ij}=\max_y n_{ij}(y)/R_{ij}$ denote report agreement on an observed pair. The blend coefficient is
\[
\alpha_{ij}=\lambda_{\mathrm{warm}}\,
\operatorname{clip}\!\left(0.3+0.8(A_{ij}-0.5),\,0.1,\,0.8\right).
\]
The strength $\lambda_{\mathrm{warm}}$ scales this agreement-weighted pair-state blend. It is $1$ in the benchmark suite, giving coefficients between $.1$ and $.8$, and $.5$ in the matched naming study. Unreported pairs retain the data-only marginal. The blended marginal supplies the first Beta update of $\rho_{ij}$. Appendix~\ref{app:sensitivity} varies the blend strength.

\emph{Empirical-Bayes hyperparameter updates.} We retain point estimates of $(\beta,\tau)$ and maximize the following objective after each reliability update:
\[
\E_q[\log p(\rho\mid\beta,\tau)] + \log p(\beta) + \log p(\tau)
\]
with a low-dimensional L-BFGS-B MAP-style empirical-Bayes step.

\paragraph{Fixed report controls.}
\label{app:fixed_controls}
The fixed-channel control in the naming and accuracy studies uses
\[
 p_w(Y=y\mid Z=z)=wC_\epsilon(z,y)+(1-w)/3,\qquad
 C_\epsilon(z,y)=
 \begin{cases}
  1-\epsilon,&y=z,\\
  \epsilon/2,&y\ne z,
 \end{cases}
\]
with $w=.8$ and $\epsilon=.05$, giving report-match probability $.8267$. The accurate-report controls use $w=.99$ in this same channel, giving match probability $.9438$. A fixed-channel fit sums report log likelihoods, omits abstentions, and makes one Top-$K$ graph call with the frozen reward. Fixed-reliability HB in the optical study sets the diagonal match probability to $\rho=.8$ and learns $B$ from its own graph posterior, using the common four-call budget.

\paragraph{Inference hyperparameters.}
Table~\ref{tab:hyperparams} gives core defaults; Table~\ref{tab:study_budgets} identifies study-specific search budgets.

\begin{table}[h]
\caption{Core hyperparameters and their scope. The outer-update limit and search budget are overridden by study as shown in Table~\ref{tab:study_budgets}.}
\label{tab:hyperparams}
\centering
\small
\begin{tabular}{@{}p{.27\linewidth}cp{.30\linewidth}p{.23\linewidth}@{}}
\toprule
Parameter & Default & Description & Scope \\
\midrule
$M$ (particles) & 20 & SVGD particle count & DiBS default \\
$T$ (steps) & 2000 & SVGD iterations & DiBS default \\
$K$ (restarts) & 20 & Restarts per graph call & Main discrete suite \\
$\tau$ (reliability prior) & 8.0 & Initial Beta concentration & HB default \\
$T_{\mathrm{EM}}$ & $\leq3$ & Outer graph updates & Main discrete suite \\
Reward centering & mean & Remove state-independent offsets & HB default \\
Warm init & on & Blend initial pair marginals & Report-based HB \\
$\lambda_{\mathrm{warm}}$ & 1.0 & Warm-blend strength & Main discrete suite \\
$q(B)$ update & on & Learn shared error patterns & HB default \\
$R$ (LLM rounds) & 5 & Rounds per source batch & Main, naming, optical \\
\bottomrule
\end{tabular}
\end{table}

\begin{table}[htbp]
\centering
\small
\caption{Top-$K$ budgets by study. Calls include initialization. Limits permit early stopping; fixed update counts run to completion. Optical no-KG repeats the zero-reward call four times under the common budget.}
\label{tab:study_budgets}
\begin{tabular}{lrrrr}
\toprule
Study and arm & Restarts/call & Updates & Calls/fit & Warm blend \\
\midrule
Benchmark suite, main HB & 20 & $\leq3$ & $\leq4$ & $1$ \\
Suite secondary diagnostics & 10 & $\leq2$ & $\leq3$ & $1$ \\
Matched naming, HB & 20 & $\leq3$ & $\leq4$ & $.5$ \\
Matched naming, fixed channel/no-KG & 20 & -- & $1$ & -- \\
Search-budget controls & $20/80$ & -- & $1$ & -- \\
Accuracy sweep, HB & 20 & $\leq3$ & $\leq4$ & $.1/.5/.9$ \\
Accuracy sweep, fixed channel/no-KG & 20 & -- & $1$ & -- \\
Paired corruption, HB & 10 & $2$ & $3$ & $1$ \\
Paired corruption, no-KG & 10 & -- & $1$ & -- \\
Independent reliability, all three arms & 20 & $3$ & $4$ & $.5$ \\
Optical, report-based methods & 20 & $3$ & $4$ & $.5$ \\
Optical, no-KG & 20 & -- & $4$ & -- \\
\bottomrule
\end{tabular}
\end{table}

The independently fitted reliability study comprises 1,620 fits, or 6,480 graph calls. The optical study comprises 500 fits and 2,000 graph calls. The paired corruption study comprises 225 HB fits and 15 no-KG fits, totaling 690 graph calls. The search-budget comparison adds 1,000 control fits and uses 400 full-HB fits from the reliability-assignment study. These studies use CPU Top-$K$ inference. Their report batches, datasets, and graph outputs are specified with the corresponding experiments.

\paragraph{Scoring functions.}
Discrete BNs use BDeu scoring~\citep{heckerman1995learning} with equivalent sample size $\alpha_{\text{BDeu}} = 1.0$.
Continuous graph-only inference uses BGe scoring~\citep{geiger2002parameter,kuipers2014addendum} with standard Normal-Wishart hyperparameters.

\paragraph{Base graph prior.}
HB uses a sparse Bernoulli graph prior with base edge probability $\theta_0=.01$. The matched no-KG Top-$K$ controls keep this prior and remove the report rewards. The joint ER comparison used a different no-KG default; its priors and graph-update budgets are specified explicitly in Appendix~\ref{app:continuous_synth}.

\paragraph{Inference backends (details).}
\emph{DiBS-SVGD backends.} We use SVGD~\citep{liu2016stein} on the latent-space parameterization of \citet{lorch2021dibs} for Sachs linear-joint and nonlinear-joint inference, the joint synthetic ER study at $d\in\{20,30\}$, and its $d=50$ linear extension. The standard budget is $M=20$ particles and $T=2000$ steps; Appendix~\ref{app:joint_budget} separately reports the larger budget. The graph-only DiBS and DiBS+ reference rows on Sachs use SVGD with BGe scoring. They are distinct from the Top-$K$ implementation used by graph-only HB and its matched no-KG control.

\emph{Top-$K$ backends.} The main discrete suite uses 20 restart orders, greedy forward initialization, local add/delete/reverse search with a light tabu heuristic, and softmax-weighted aggregation of the top unique DAGs. Study-specific budgets are listed in Table~\ref{tab:study_budgets}. The Sachs graph-only HB run uses the \texttt{pairwise\_topk} backend with BGe scoring, while its no-KG Top-$K$ control uses the corresponding graph search without knowledge rewards. The observation reward enters the graph-only score through Equation~\eqref{eq:mstep_graph_only}; joint DiBS-SVGD uses Equation~\eqref{eq:mstep_joint}.

\paragraph{Compute resources.}
Sachs experiments were run locally on a MacBook Pro with an Apple M3 Max chip (16 CPU cores) and 48\,GB RAM. Runtime depends on the posterior family and inference budget; Table~\ref{tab:sachs_continuous} and the study-specific descriptions state the restart units. The nonlinear synthetic ER study and the $d=50$ extension used GPU hardware; Appendix~\ref{app:continuous_synth} reports their hardware and per-condition timing.

\paragraph{Threshold parameters.}
Unless otherwise noted, all discrete thresholds ($c_{\text{exist}} \geq 0.30$, $c_{\text{dir}} < 0.65$) were set to conservative defaults and verified stable via sensitivity analysis; none required dataset-specific tuning.

\paragraph{Point estimates and structural metrics.}
For the benchmark-suite HB and DiBS estimates, edges with marginal probability above $.5$ are considered in descending probability and retained only if they preserve acyclicity. Bootstrap PC/GES estimates threshold frequencies at $.5$ and remove the lowest-frequency edge from each detected cycle. The local continuous Sachs posterior comparisons select the highest-weight stored graph for point SHD. The synthetic fixed-channel controls, hyperparameter grid, and matched discrete studies score the adjacency estimate $\widehat A_{ij}=\mathbb{I}[P_{ij}>.5]$. Direct-search methods use their returned graphs.

Let $A^\star$ be the true adjacency matrix and $P$ the matrix of edge marginals, with zero diagonal. Ordered-adjacency error is
\[
L(P,A^\star)=\sum_{i\ne j}\bigl[A^\star_{ij}(1-P_{ij})+(1-A^\star_{ij})P_{ij}\bigr].
\]
The local Sachs evaluations, ER comparisons, and matched discrete studies report $L(\widehat A,A^\star)$ as SHD and $L(P,A^\star)$ as E-SHD; a reversal costs two. If $P=\sum_m w_m A^{(m)}$, E-SHD equals the weighted mean ordered-adjacency error of those graphs. For a posterior over DAGs, the benchmark-suite reversal-once convention gives
\[
\mathrm{E\text{-}SHD}_1
=\sum_m w_m\,\mathrm{SHD}_1(A^{(m)},A^\star)
=L(P,A^\star)-\sum_{i\ne j}A^\star_{ij}P_{ji}.
\]
Opposite edges cannot coexist in a DAG, so each reversed true edge receives one unit of credit. Figure~\ref{fig:benchmark_suite_full} computes its E-SHD estimates with the edgewise expression $L(P,A^\star)-\sum_{i\ne j}A^\star_{ij}P_{ji}(1-P_{ij})$. The discrete Harmonized-Prior reference uses ordered-adjacency SHD and sets E-SHD equal to that point-estimate error. Published MC3/BPC values retain their source protocol.

\paragraph{Edge marginals and AUROC.}
Top-$K$ inference averages unique graphs with normalized exponentiated data-plus-prior scores. Joint DiBS uses a weighted graph-parameter mixture. Graph-only DiBS averages particle edge probabilities, and DiBS+ uses a weighted graph-sample mixture. AUROC evaluates off-diagonal edge scores with average ranks for ties. The sensitivity and provider tables focus on structural metrics and inferred reliability.

\paragraph{A diagnostic of joint particle weights.}
\label{app:joint_weight_diagnostic}
One no-KG linear-joint diagnostic on Sachs uses eight particles, 600 steps, and seed 42. Its eight graphs are distinct, but one receives all numerical mixture weight, giving effective sample size $1$. Six graphs are acyclic; the two cyclic graphs have zero weight. The edge marginals therefore equal the highest-weight adjacency, and both E-SHD and point SHD are $16$. This example illustrates how concentrated mixture weights produce equal metrics.

\paragraph{Replication units.}
A data seed indexes a sampled observational dataset; a source batch contains five report-query rounds for a fixed variable presentation; an inference restart varies the graph-inference seed while holding its inputs fixed. In the matched naming study, four source batches share 20 data seeds, and paired tests average the source batches within each data seed. The benchmark-suite seed evaluations reuse cached observations as detailed in Appendix~\ref{app:discrete_full}. The ER study has five independent graphs, with datasets and report conditions nested within each graph.

\section{Knowledge graph quality analysis}
\label{app:kg_quality}

Table~\ref{tab:kg_quality} reports directed-edge quality for the benchmark-suite causal-candidate reports. The prompt-target and naming studies use separately elicited reports, summarized in Tables~\ref{tab:prompt_targets} and~\ref{tab:name_only_f1}.
These metrics are computed on the \emph{directed} edge set after aggregation across $R{=}5$ rounds with minimum round support $\geq 2$: a true positive requires both the correct parent-child pair and the correct direction.
Precision ranges from $.41$ to $.73$ and recall from $.36$ to $1.00$, giving both useful edges and incorrect claims.

\begin{table}[h]
\caption{Directed-edge quality of the benchmark-suite causal-candidate reports. $|\KG|$ is the aggregated edge count. The four-batch naming intervention and direct/path prompt comparison use separate report sets.}
\label{tab:kg_quality}
\centering
\small
\begin{tabular}{l c c c c c c}
\toprule
Network & $d$ & $|E_{\text{true}}|$ & $|\KG|$ & Precision & Recall & F1 \\
\midrule
Asia & 8 & 8 & 11 & 0.73 & 1.00 & 0.84 \\
Sachs & 11 & 17 & 12 & 0.58 & 0.41 & 0.48 \\
Child & 20 & 25 & 22 & 0.41 & 0.36 & 0.38 \\
Insurance & 27 & 52 & 37 & 0.51 & 0.37 & 0.43 \\
Alarm & 37 & 46 & 43 & 0.47 & 0.43 & 0.45 \\
\bottomrule
\end{tabular}
\end{table}

\subsection{Prompt semantics and observation targets}
\label{app:prompt_semantics}

\paragraph{Direct edges and directed paths.}
On Asia and Sachs, we compare prompts requesting immediate direct parents with prompts requesting the existence of a directed path (Table~\ref{tab:prompt_targets}). Within this comparison, canonical variable names, domain, GPT-4.1-mini, the JSON output schema, temperatures, query rounds, and aggregation are fixed. Each prompt has four independent source batches of $R=5$ rounds. Direct-parent reports are evaluated against the true adjacency matrix; path reports are evaluated against reachability, obtained from the transitive closure of the DAG.

\begin{table}[htbp]
\centering
\small
\caption{Report quality under explicit causal targets, averaged over four source batches. Direct prompts are evaluated against adjacency; path prompts against reachability. Asia has 8 direct edges and 18 reachable ordered pairs.}
\label{tab:prompt_targets}
\begin{tabular}{lcc}
\toprule
Network & Direct prompt: adjacency F1 & Path prompt: reachability F1 \\
\midrule
Asia & $0.570\pm0.081$ & $0.872\pm0.021$ \\
Sachs & $0.440\pm0.037$ & $0.429\pm0.035$ \\
\bottomrule
\end{tabular}
\end{table}

\paragraph{Numerical total effects.}
Total effects sum contributions along directed paths, which can cancel. We evaluate this numerical target on ten parameterized linear SCMs using $(I-W)^{-1}-I$. Across 150 calls, target-matched F1 is $1.000$ for direct-parent prompts, $1.000$ for path-existence prompts, and $.945$ for total-effect prompts. The last prompt selects all five designated reachable pairs with zero total effect, treating them as positive despite cancellation.

\paragraph{Inference with path reports.}
We instantiate the likelihood over path states $\{\mathrm{none},i\leadsto j,j\leadsto i\}$ and compare it with direct-state HB and a fixed-$\rho=.8$ path-aware prior. All methods receive the same reports, with four batches paired to twenty data seeds. Table~\ref{tab:path_channel_comparison} reports direct-graph SHD. Paired Wilcoxon tests on seed means, averaging batches within seed, give inconclusive direct-versus-path HB differences. The prompt results above establish the distinct targets elicited by each instruction.

\begin{table}[htbp]
\centering
\small
\caption{Inference from the same path reports. Each entry is mean SHD; lower is better. The fixed comparator uses a path-state observation model with $\rho=0.8$.}
\label{tab:path_channel_comparison}
\begin{tabular}{lccc}
\toprule
Network & HB-direct & HB-path-aware & Fixed path-aware \\
\midrule
Asia & $1.52$ & $1.81$ & $2.34$ \\
Sachs & $7.86$ & $8.05$ & $8.03$ \\
\bottomrule
\end{tabular}
\end{table}

\subsection{Variable names and source quality}
\label{app:name_only}

The naming intervention changes only the variable labels among canonical names, anonymous IDs, and pseudo-scientific names. The true domain and explicit direct-parent prompt are fixed, and no condition provides variable descriptions. We make 20 calls per network and naming condition, for 300 calls in total. Four disjoint batches of five rounds are each paired with the same 20 data seeds. These source batches and data seeds are distinct levels of replication.

\begin{table}[htbp]
\centering
\small
\caption{Mean direct-edge F1 under the clean name-only intervention. Each value averages four source batches. A zero F1 can result from either absent predictions or incorrect predictions; the downstream comparison distinguishes these cases.}
\label{tab:name_only_f1}
\begin{tabular}{lccc}
\toprule
Network & Canonical & Anonymous IDs & Pseudo-scientific \\
\midrule
Alarm & $0.355$ & $0.000$ & $0.228$ \\
Asia & $0.591$ & $0.000$ & $0.000$ \\
Child & $0.480$ & $0.000$ & $0.000$ \\
Insurance & $0.401$ & $0.000$ & $0.000$ \\
Sachs & $0.392$ & $0.018$ & $0.000$ \\
\bottomrule
\end{tabular}
\end{table}

Table~\ref{tab:name_only_coverage} reports aggregated edge counts and empty elicitation rounds. Individual rounds can contain useful observations even when aggregation leaves an empty edge set.

\begin{table}[htbp]
\centering
\small
\caption{Source coverage in the naming intervention: mean aggregated directional-edge count across four batches / number of empty rounds out of twenty.}
\label{tab:name_only_coverage}
\begin{tabular}{lccc}
\toprule
Network & Canonical & Anonymous IDs & Pseudo-scientific \\
\midrule
Alarm & $30.00/0$ & $0.00/20$ & $14.50/6$ \\
Asia & $9.75/0$ & $0.50/17$ & $0.00/20$ \\
Child & $18.75/0$ & $0.00/20$ & $0.00/19$ \\
Insurance & $26.50/0$ & $0.00/19$ & $0.00/17$ \\
Sachs & $8.50/0$ & $3.25/11$ & $10.00/0$ \\
\bottomrule
\end{tabular}
\end{table}

Canonical names yield useful reports (Table~\ref{tab:name_only_f1}). The mean SHD reductions over no-KG are $2.38$, $5.55$, $9.60$, $7.96$, and $14.50$ on Asia, Sachs, Child, Insurance, and Alarm, respectively (Table~\ref{tab:name_only_downstream}). Paired Wilcoxon tests average four batches within each of twenty data seeds (all unadjusted $p<.001$). Removing canonical names lowers mean inferred reliability on every network. Anonymous Alarm and Child and pseudo-scientific Asia have zero report reward and reproduce no-KG. On Sachs, incorrect anonymous reports survive and increase SHD by $.51$ ($p=.444$). Report availability and the correctness of surviving claims produce different outcomes.

\begin{table}[htbp]
\centering
\small
\caption{Complete downstream naming comparisons. HB and the fixed-channel control use the same reports and graph backend, averaged over four batches and 20 shared data seeds. No-KG uses those datasets. The final column is HB's mean inferred reliability.}
\label{tab:name_only_downstream}
\begin{tabular}{llrrrr}
\toprule
Network & Labels & No-KG SHD & HB SHD & Fixed-channel SHD & HB $\bar\rho$ \\
\midrule
Asia & Canonical & $4.30$ & $1.93$ & $2.84$ & $0.501$ \\
 & Anonymous & $4.30$ & $3.95$ & $4.49$ & $0.403$ \\
 & Pseudo-scientific & $4.30$ & $4.30$ & $4.30$ & $0.378$ \\
\midrule
Sachs & Canonical & $12.80$ & $7.25$ & $7.05$ & $0.454$ \\
 & Anonymous & $12.80$ & $13.31$ & $13.43$ & $0.434$ \\
 & Pseudo-scientific & $12.80$ & $14.09$ & $15.22$ & $0.399$ \\
\midrule
Child & Canonical & $17.05$ & $7.45$ & $6.84$ & $0.423$ \\
 & Anonymous & $17.05$ & $17.05$ & $17.05$ & $0.378$ \\
 & Pseudo-scientific & $17.05$ & $16.90$ & $16.91$ & $0.381$ \\
\midrule
Insurance & Canonical & $42.45$ & $34.49$ & $35.49$ & $0.408$ \\
 & Anonymous & $42.45$ & $42.91$ & $43.09$ & $0.380$ \\
 & Pseudo-scientific & $42.45$ & $42.79$ & $42.80$ & $0.386$ \\
\midrule
Alarm & Canonical & $43.55$ & $29.05$ & $29.04$ & $0.403$ \\
 & Anonymous & $43.55$ & $43.55$ & $43.55$ & $0.378$ \\
 & Pseudo-scientific & $43.55$ & $39.55$ & $41.15$ & $0.394$ \\
\bottomrule
\end{tabular}
\end{table}

The intervention measures how variable semantics affect report quality. Comparisons on identical reports evaluate their use in inference, and Appendix~\ref{app:accuracy_calibration} changes synthetic accuracy and report count directly.

\section{Source comparison across LLM providers}
\label{app:source_robustness}

We run the report pipeline with GPT-4.1-mini, GPT-5.1, GPT-4o, and Claude Sonnet 4.5 on Asia, Sachs, and Child. Each source uses five rounds at temperatures $.65$--$.85$. This source study samples 1,000 observations at seed 42, constructs each round separately, and applies common HB defaults across providers without the benchmark suite's network-specific overrides; graph-inference seeds are zero. It therefore evaluates variation across sources under one common configuration.

Table~\ref{tab:source_robustness} reports mean inferred reliability. Within-network standard deviations across sources are $.049$ on Asia, $.027$ on Sachs, and $.003$ on Child. These summaries describe the fitted report layer across providers; calibration against known generating accuracy is assessed in Appendix~\ref{app:accuracy_calibration}.

\begin{table}[h]
\caption{Mean inferred reliability across LLM sources under the common source-study configuration. Each source has five report rounds and one sampled dataset per network.}
\label{tab:source_robustness}
\centering
\small
\begin{tabular}{lcccc}
\toprule
Network & GPT-4.1-mini & GPT-5.1 & GPT-4o & Claude Sonnet 4.5 \\
\midrule
Asia & $.783$ & $.714$ & $.800$ & $.830$ \\
Sachs & $.696$ & $.684$ & $.700$ & $.640$ \\
Child & $.804$ & $.802$ & $.799$ & $.806$ \\
\bottomrule
\end{tabular}
\end{table}

\section{Additional discrete diagnostics and ablation results}
\label{app:discrete_full}

The benchmark suite compares DiBS and DiBS+~\citep{lorch2021dibs}, GES~\citep{chickering2002optimal}, PC~\citep{spirtes2000causation}, MMHC~\citep{tsamardinos2006max}, bootstrap BPC and BGES~\citep{friedman1999data}, Harmonized-Prior~\citep{ban2025harmonized}, and knowledge-prior controls. KG-SoftMAP~\citep{xu2026kgsoftmap} uses greedy MAP search with a fixed soft edge prior; Random-KG-SoftMAP uses uniformly sampled edges. Figure~\ref{fig:benchmark_suite_full} shows the full single-run method comparison. Tables~\ref{tab:discrete_significance_full} and~\ref{tab:discrete_ablation_full} give the separate five-seed summaries, with 25 network--seed evaluations. Asia and Sachs permute their 1,000 observations; the larger networks use overlapping 1,000-row subsets of 2,000 rows. The mean observational negative MLL is $9.624$ for HB, $9.767$ for DiBS-only, and $9.808$ for MMHC. Interventional negative I-MLL averages BMA query probabilities before taking their negative logarithm, in nats.

\begin{table}[h]
\caption{Aggregate differences in the discrete benchmark suite, reported as HB minus the comparison method over 25 network-seed evaluations. These are descriptive differences: seed evaluations reuse observations within each fixed network.}
\label{tab:discrete_significance_full}
\centering
\small
\begin{tabular}{l l c}
\toprule
Comparison & Metric & Mean difference \\
\midrule
vs.\ DiBS-only & SHD & $-7.96$ \\
vs.\ DiBS-only & AUROC & $+0.178$ \\
vs.\ DiBS-only & neg.\ MLL & $-0.142$ \\
vs.\ MMHC & SHD & $-2.44$ \\
vs.\ MMHC & AUROC & $+0.222$ \\
vs.\ MMHC & neg.\ MLL & $-0.184$ \\
\bottomrule
\end{tabular}
\end{table}

Table~\ref{tab:discrete_ablation_full} reports absolute recovery and predictive metrics for the knowledge-source comparisons and component ablations. Figure~\ref{fig:discrete_sample_ablation_full} shows the sample-size trends in panel~\subref{fig:discrete_sample_size_full} and component effects in panel~\subref{fig:discrete_ablation_plot_full}.

\begin{table}[h]
\caption{Knowledge-source and component-ablation summaries for the main discrete benchmark suite.}
\label{tab:discrete_ablation_full}
\centering
\small
\begin{tabular}{l c c c c}
\toprule
Variant & SHD $\downarrow$ & AUROC $\uparrow$ & neg.\ MLL $\downarrow$ & $\bar{\rho}$ \\
\midrule
HB-NoisyKG (real multi-round LLM KG) & $14.84$ & $0.925$ & $9.624$ & $0.425$ \\
HB-NoisyKG (synthetic ground-truth KG) & $10.84$ & $0.944$ & $9.618$ & $0.443$ \\
\midrule
HB-NoisyKG & $14.84$ & $0.925$ & $9.624$ & $0.425$ \\
No confusion update & $14.96$ & $0.919$ & $9.626$ & $0.431$ \\
No $\beta,\tau$ update & $15.04$ & $0.919$ & $9.629$ & $0.440$ \\
No KG warm start & $16.44$ & $0.871$ & $9.618$ & $0.400$ \\
No local search & $21.20$ & $0.897$ & $9.739$ & $0.425$ \\
\bottomrule
\end{tabular}
\end{table}

\paragraph{Report weighting and reliability diagnostics.}
The matched comparison in Table~\ref{tab:matched_reports_main}(a) supplies HB and the fixed-channel control with identical direct-parent reports and observations. The benchmark-suite diagnostics in Figure~\ref{fig:discrete_reliability_full} address a different question: correct reports receive mean inferred reliability $.24$ higher than incorrect reports. Correlations with reported confidence, existence agreement, and directional agreement are $.87$, $.85$, and $.86$, respectively. These are associations with inputs to the reliability prior. The controlled accuracy and report-count studies below assess reliability against known generating accuracy; Appendix~\ref{app:sensitivity} examines the effect of initialization.

\paragraph{Fixed channels with synthetic reports.}
A separate experiment compares the fixed channel at $w=.8$ with per-pair channel weights obtained from synthetic reported confidence. It samples 1,000 observations on each of the five networks at seeds 41--45 and retains each true edge with probability $.7$, adding spurious edges at rate $.2$. Confidence is sampled from $U(.6,.95)$ for retained edges and $U(.3,.7)$ for spurious edges. Both controls use 20 Top-$K$ restarts and compute ordered-adjacency SHD after thresholding edge marginals at $.5$, without cycle pruning. Their mean SHD is $25.52$ and $24.00$, respectively. These values describe the synthetic-report controls; the benchmark-suite HB mean $14.84$ uses LLM reports, cached observations, and a different SHD convention. The matched report-weighting comparison is therefore Table~\ref{tab:matched_reports_main}(a).

\begin{figure*}[t]
\centering
\includegraphics[width=0.82\textwidth]{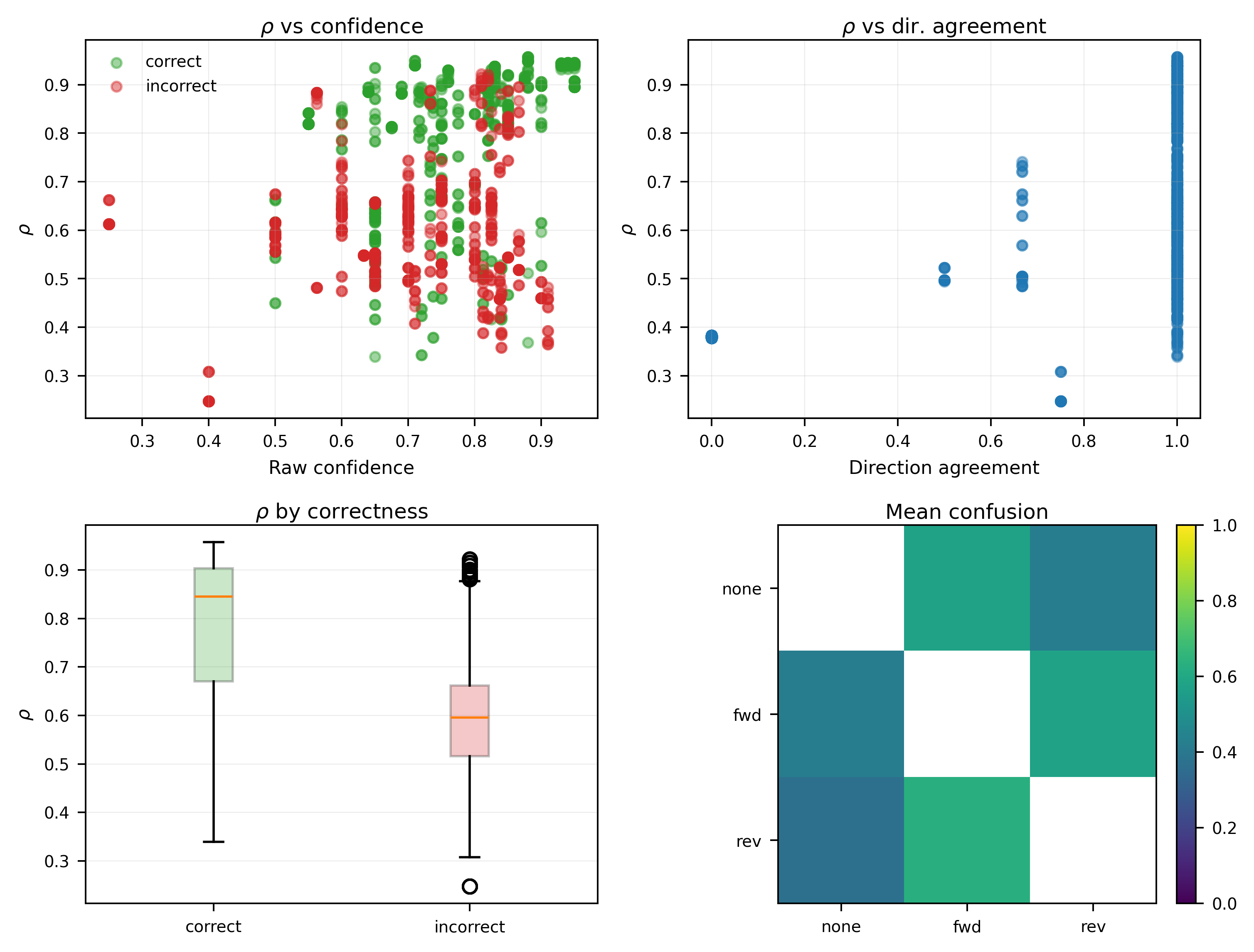}
\caption{Full reliability diagnostics for the main discrete benchmark suite. The learned reliability values rise with confidence and agreement, and correct claims receive higher $\rho$ than incorrect ones. The confusion panel summarizes the mean shared error structure.}
\label{fig:discrete_reliability_full}
\end{figure*}

\begin{figure*}[t]
\centering
\begin{subfigure}[t]{0.495\textwidth}
    \centering
    \includegraphics[width=\textwidth]{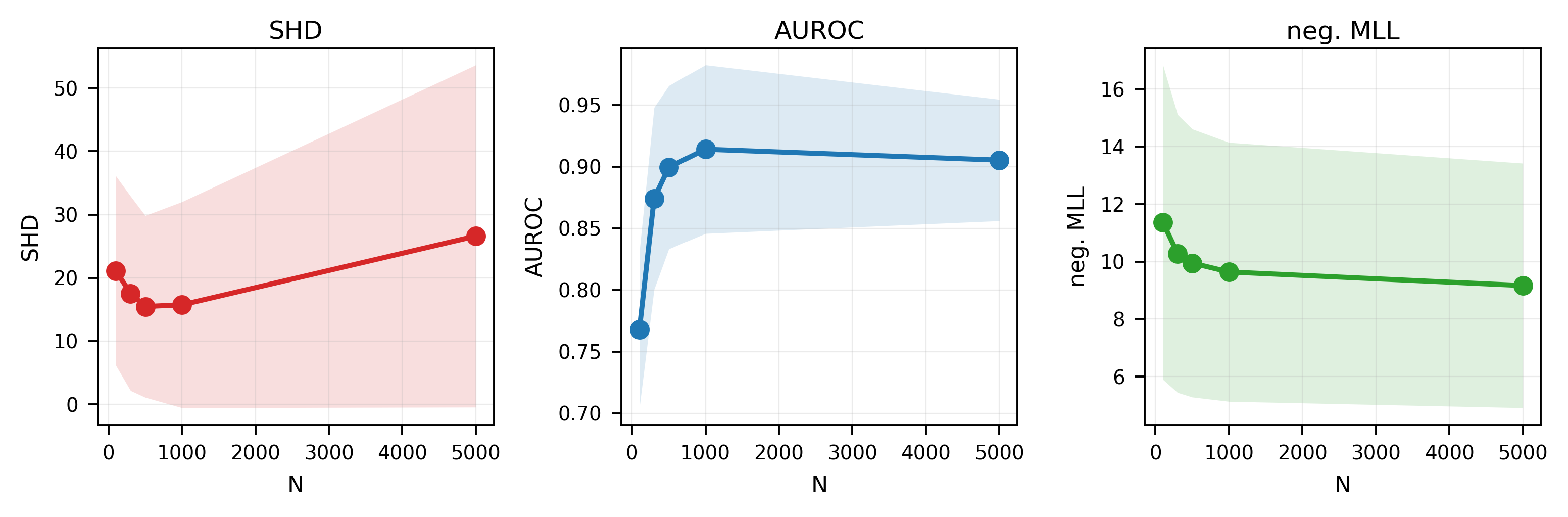}
    \caption{Sample-size trends.}
    \label{fig:discrete_sample_size_full}
\end{subfigure}\hfill
\begin{subfigure}[t]{0.495\textwidth}
    \centering
    \includegraphics[width=\textwidth]{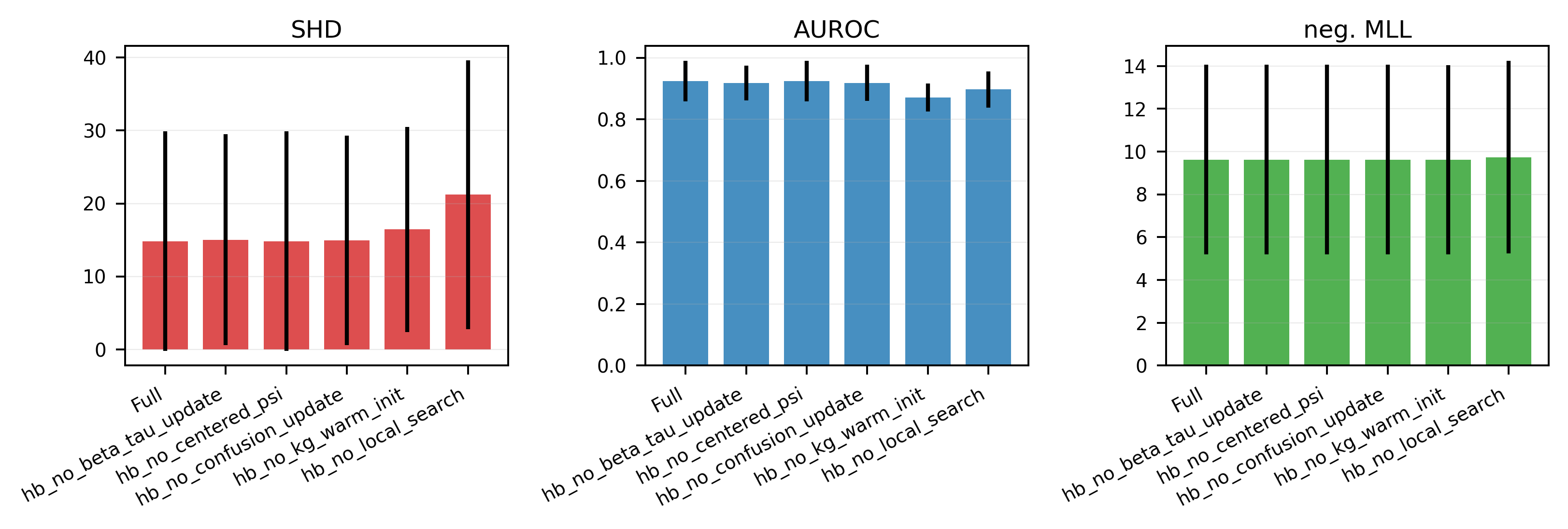}
    \caption{Component ablations.}
    \label{fig:discrete_ablation_plot_full}
\end{subfigure}
\caption{Discrete diagnostics. The sample-size sweep uses $N\in\{100,300,500,1000,5000\}$. Negative MLL decreases across this sweep. Structural recovery is not monotone: SHD rises at $N=5000$, while AUROC remains above its $N=100$ value. The ablation plot shows the effects of the individual components.}
\label{fig:discrete_sample_ablation_full}
\end{figure*}

\begin{table}[h]
\centering
\scriptsize
\setlength{\tabcolsep}{3.5pt}
\caption{Fully paired corruption study, averaged over five networks and three sampling seeds per setting. SHD gives the means at the lowest and highest tested severity; AUROC gives the lowest severity-specific mean. Reliability ranges are across condition means, not individual runs.}
\label{tab:robustness_summary}
\begin{tabular}{l c c c c}
\toprule
Corruption type & Severity range & Mean SHD (start $\to$ end) & Lowest mean AUROC & $\bar{\rho}$ range \\
\midrule
Missing edges      & retention $1.0 \to 0.3$ & $2.80 \to 7.93$ & $0.925$ & $0.451$--$0.483$ \\
Spurious edges     & spurious $0.0 \to 0.5$ & $2.80 \to 2.93$ & $0.969$ & $0.483$--$0.553$ \\
Flipped directions & flip $0.0 \to 0.7$ & $2.80 \to 19.13$ & $0.808$ & $0.418$--$0.483$ \\
Mixed corruption   & mixed $0.1 \to 0.9$ & $3.73 \to 19.80$ & $0.830$ & $0.487$--$0.496$ \\
\midrule
No-KG & all settings & $17.93$ & $0.880$ & -- \\
\bottomrule
\end{tabular}
\end{table}

Table~\ref{tab:corruption_nokg_all} gives every paired condition, and Figure~\ref{fig:discrete_robustness_full} plots the resulting curves. The missing-edge sweep fixes spurious rate at zero; the spurious-edge sweep fixes retention at one. Mixed indices $.1$, $.5$, and $.9$ correspond to (retention, spurious rate, flip rate) $(.9,.1,.1)$, $(.5,.3,.5)$, and $(.1,.5,.9)$, respectively. The clean $2.80$ mean SHD is the baseline of this synthetic corruption study, whereas the ground-truth replacement result in Table~\ref{tab:discrete_diag} uses the benchmark-suite observations and report construction.

The 225 HB fits are paired with fifteen data-only fits on common inputs. Each graph call uses ten restarts, 120 local-search steps, a tabu size of ten, and at most 25 retained graphs. HB makes three calls: data-only initialization and two updates. No-KG makes one call. Sampling seeds permute the Asia and Sachs pools of 1,000 rows and select overlapping subsets from pools of 2,000 rows on the other networks. The reported means average these five fixed networks and sampling conditions. AUROC uses average ranks for ties.

\begin{table}[htbp]
\centering
\small
\caption{Complete paired corruption profile. $r,s,f$ denote retention, spurious-edge rate, and flip rate. No-KG has SHD $17.933$ and AUROC $.880$ in every row. Differences are HB minus no-KG; negative SHD and positive AUROC differences favor HB. The clean condition appears in both the missing/spurious and flip protocol blocks.}
\label{tab:corruption_nokg_all}
\begin{tabular}{lrrrrrr}
\toprule
Block & $r$ & $s$ & $f$ & HB SHD & $\Delta$SHD & $\Delta$AUROC \\
\midrule
Missing/spurious & $1$ & $0$ & $0$ & $2.800$ & $-15.133$ & $+.089$ \\
 & $1$ & $.2$ & $0$ & $2.867$ & $-15.067$ & $+.089$ \\
 & $1$ & $.5$ & $0$ & $2.933$ & $-15.000$ & $+.089$ \\
 & $.7$ & $0$ & $0$ & $3.467$ & $-14.467$ & $+.077$ \\
 & $.7$ & $.2$ & $0$ & $3.467$ & $-14.467$ & $+.084$ \\
 & $.7$ & $.5$ & $0$ & $3.467$ & $-14.467$ & $+.077$ \\
 & $.3$ & $0$ & $0$ & $7.933$ & $-10.000$ & $+.045$ \\
 & $.3$ & $.2$ & $0$ & $7.600$ & $-10.333$ & $+.046$ \\
 & $.3$ & $.5$ & $0$ & $8.000$ & $-9.933$ & $+.047$ \\
\midrule
Flip & $1$ & $0$ & $0$ & $2.800$ & $-15.133$ & $+.089$ \\
 & $1$ & $0$ & $.3$ & $8.133$ & $-9.800$ & $+.053$ \\
 & $1$ & $0$ & $.7$ & $19.133$ & $+1.200$ & $-.072$ \\
\midrule
Mixed & $.9$ & $.1$ & $.1$ & $3.733$ & $-14.200$ & $+.079$ \\
 & $.5$ & $.3$ & $.5$ & $15.333$ & $-2.600$ & $-.032$ \\
 & $.1$ & $.5$ & $.9$ & $19.800$ & $+1.867$ & $-.051$ \\
\bottomrule
\end{tabular}
\end{table}

\begin{figure*}[t]
\centering
\includegraphics[width=\textwidth]{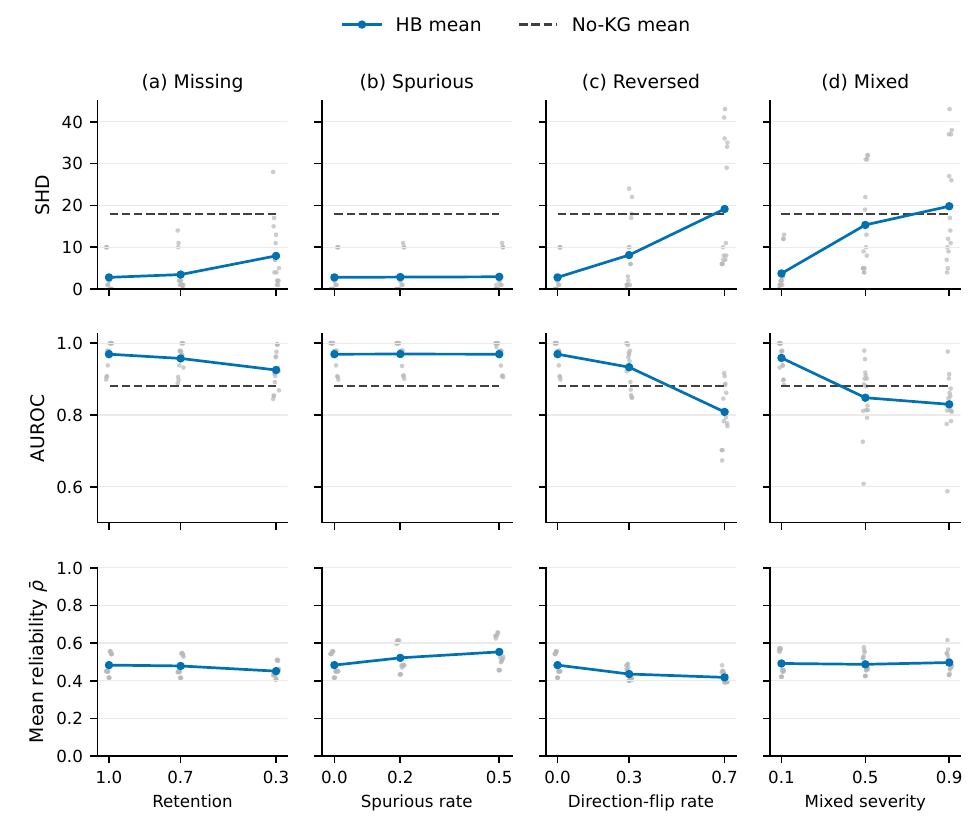}
\caption{Fully paired corruption results. Grey points are the 15 network-seed evaluations per condition; blue curves are their means. Dashed lines show no-KG means for SHD and AUROC. Reliability averages all variable pairs within a run, including unreported pairs. The missing curve fixes spurious rate at zero, and the spurious curve fixes retention at one. Mixed-index settings are defined beside Table~\ref{tab:robustness_summary}.}
\label{fig:discrete_robustness_full}
\end{figure*}

\subsection{Pair-specific and pooled reliability}
\label{app:reliability_pooling}

This study tests how the assignment of learned reliability affects the final graph update. We reuse the canonical-name reports from the naming study: four batches of five rounds on each of the five networks, paired with 20 data seeds and $N=1000$. Full HB is fitted within this study using 20 Top-$K$ restarts, at most three outer updates, and warm strength $.5$. Each control then receives the same learned confusion posterior and the same final graph-search budget.

The primary comparison replaces each pair's reliability by its posterior mean and compares these pair-specific values with their unweighted average over pairs having at least one report. Both arms use the same point-estimate treatment of reliability. A separate Beta-pooling control replaces each observed pair's posterior mean by that common mean while preserving its concentration $a_{ij}+b_{ij}$. The network-mean control uses one scalar per network, averaged over all pairs in the naming study, including pairs without reports. It therefore changes overall trust as well as removing differences among reported pairs. All these controls use parameters learned by HB; they evaluate reliability assignment conditional on those estimates. Table~\ref{tab:reliability_pooling_overall} summarizes their recovery and probability metrics.

\begin{table}[htbp]
\centering
\small
\caption{Reliability-assignment controls. Values are equal-network averages. Each report-based arm has 400 network--seed--batch conditions; no-KG is evaluated once per network and data seed. All edge scores use average-tie AUROC, and SHD counts reversals twice.}
\label{tab:reliability_pooling_overall}
\begin{tabular}{lrrrr}
\toprule
Control & SHD $\downarrow$ & E-SHD $\downarrow$ & AUROC $\uparrow$ & Brier $\downarrow$ \\
\midrule
No-KG & $23.960$ & $24.149$ & $.8965$ & $.0614$ \\
Full HB & $16.060$ & $16.173$ & $.9160$ & $.0364$ \\
Pooled Beta means & $16.265$ & $16.391$ & $.9220$ & $.0335$ \\
Pair-specific point means & $16.040$ & $16.129$ & $.9166$ & $.0363$ \\
Pooled observed-pair point mean & $16.265$ & $16.391$ & $.9210$ & $.0337$ \\
Network mean over all pairs & $21.075$ & $21.224$ & $.9046$ & $.0428$ \\
\bottomrule
\end{tabular}
\end{table}

For each network, we average the four source batches within each data seed before computing paired differences. Table~\ref{tab:reliability_pooling_paired} reports the primary point-mean comparison, with bootstrap intervals over 20 data seeds and two-sided Wilcoxon tests corrected by Holm across the five networks. These intervals condition on the benchmark graph and the four report batches.

\begin{table}[htbp]
\centering
\small
\caption{Primary reliability-assignment comparison. $\Delta$SHD is pooled minus pair-specific point means, so positive values favor pair-specific weighting. Confidence intervals are unadjusted 95\% bootstrap intervals for the mean difference.}
\label{tab:reliability_pooling_paired}
\begin{tabular}{lrrrrr}
\toprule
Network & Pair SHD & Pooled SHD & $\Delta$SHD & 95\% CI & $p_{\mathrm{Holm}}$ \\
\midrule
Asia & $1.925$ & $1.188$ & $-.738$ & $[-1.113,-.413]$ & $.011$ \\
Sachs & $7.338$ & $7.338$ & $.000$ & $[-.513,.538]$ & $1.000$ \\
Child & $7.338$ & $6.800$ & $-.538$ & $[-1.063,-.038]$ & $.170$ \\
Insurance & $34.475$ & $36.713$ & $2.238$ & $[.500,3.950]$ & $.096$ \\
Alarm & $29.125$ & $29.288$ & $.163$ & $[-1.075,1.275]$ & $1.000$ \\
\bottomrule
\end{tabular}
\end{table}

Pair-specific means have a small advantage in the equal-network SHD average, whereas pooling gives higher mean AUROC and lower Brier score. The SHD advantage varies across networks: Asia favors pooling after correction, and none of the five tests establishes a pair-specific advantage after correction. The benefit of estimating reliability and the incremental benefit of assigning different reliabilities to different pairs are therefore distinct empirical questions.

\subsection{Independently learned common reliability}
\label{app:common_reliability}

We fit three models independently. Common reliability uses one $\rho_{\mathrm{s}}$ for all observed pairs. Pair-intercept retains distinct $\rho_{ij}$ with common prior mean $\mu=\sigma(\beta_0)$. Full HB uses the feature-conditioned prior and its default coefficient initialization. All learn the shared error matrix $B$ with the same graph likelihood, base prior, data-only initialization, warm blend, and inference budget. Their covariates and initial reliability priors define the compared specifications.

Writing $m_{ij}=\sum_z q_{ij}(z)n_{ij}(z)$, the common-reliability update is
\[
q(\rho_{\mathrm{s}})=\mathrm{Beta}\!\left(\tau\mu+\sum_{i<j}m_{ij},\;
\tau(1-\mu)+\sum_{i<j}(R_{ij}-m_{ij})\right).
\]
The common Beta prior is counted once, with one reliability term in the empirical-Bayes objective. Unreported pairs contribute zero counts. Pair-intercept has one Beta variable and prior term per pair. Both intercept variants begin at $\beta_0=-.5$ and $\tau=8$, with the same Gaussian and Gamma hyperpriors as HB. Each arm estimates its own parameters.

Table~\ref{tab:common_learned_trust} relates learned trust to recovery. On Insurance, common reliability has a lower estimate than HB's observed-pair mean and higher SHD than the shared data-only initialization.

Relative report evidence depends jointly on reliability and the error matrix. For report $y$ and competing state $z\ne y$, the channel likelihood ratio is $\rho/[(1-\rho)B_{zy}]$ in favor of $Z=y$. A uniform error row, $B_{zy}=1/2$, gives threshold $\rho=1/3$. The graph update uses the expected log likelihoods of both quantities.

\begin{table}[htbp]
\centering
\small
\caption{Learned trust and SHD in the independent-reliability comparison. HB reliability averages observed pairs; common reliability is a single source parameter. Initial no-KG is the shared one-call initialization. The batches and data seeds are those of Table~\ref{tab:common_matched_full}.}
\label{tab:common_learned_trust}
\begin{tabular}{lrrrrr}
\toprule
Network & Common $\rho_s$ & HB mean $\rho$ & Common SHD & HB SHD & Initial no-KG \\
\midrule
Asia & $.464$ & $.638$ & $1.80$ & $1.93$ & $4.30$ \\
Sachs & $.714$ & $.767$ & $7.54$ & $7.43$ & $12.60$ \\
Child & $.564$ & $.684$ & $6.71$ & $7.61$ & $16.90$ \\
Insurance & $.219$ & $.602$ & $45.84$ & $35.59$ & $42.45$ \\
Alarm & $.427$ & $.657$ & $37.73$ & $28.73$ & $43.55$ \\
\bottomrule
\end{tabular}
\end{table}

The matched study reuses four canonical report batches and 20 data seeds on each of five networks, giving 400 conditions per variant. The accuracy study uses two networks, seven accuracy levels, and ten seeds, giving 140 conditions per variant. Report confidence is fixed at $.5$ in the accuracy study. Every variant uses warm-start strength $.5$, 20 Top-$K$ restarts and up to 60 retained DAGs per graph call, with one initial data-only fit followed by exactly three graph updates. This fixed total budget defines a separate controlled comparison from Table~\ref{tab:matched_reports_main}.

The primary contrast is full HB minus independently fitted common reliability in ordered-adjacency SHD, counting reversals twice. We average the four matched-study batches within each data seed before paired tests. Holm correction covers five network tests in the matched study and, separately, fourteen network--accuracy tests in the accuracy study. Secondary comparisons isolate pair-intercept versus common reliability and full HB versus the pair-intercept prior. Confidence intervals condition on the benchmark graphs and, for the matched study, the cached report batches. AUROC uses average ranks for tied scores, and all probability-based metrics exclude diagonal entries.

\begin{table}[htbp]
\centering
\small
\caption{Independent reliability fits on canonical LLM reports. Each entry averages four source batches and twenty data seeds. Common means one jointly learned source reliability; Pair-int uses separate pair reliabilities with an intercept-only prior.}
\label{tab:common_matched_full}
\begin{tabular}{llrrrr}
\toprule
Network & Model & SHD & E-SHD & AUROC & Brier \\
\midrule
Asia & Common & $1.800$ & $1.994$ & $.9124$ & $.0310$ \\
 & Pair-int & $3.763$ & $3.744$ & $.8509$ & $.0632$ \\
 & Full HB & $1.925$ & $1.957$ & $.9225$ & $.0329$ \\
Sachs & Common & $7.538$ & $7.746$ & $.9514$ & $.0593$ \\
 & Pair-int & $13.275$ & $13.559$ & $.8769$ & $.1139$ \\
 & Full HB & $7.425$ & $7.536$ & $.9442$ & $.0608$ \\
Child & Common & $6.713$ & $7.033$ & $.9170$ & $.0169$ \\
 & Pair-int & $12.263$ & $12.358$ & $.8968$ & $.0318$ \\
 & Full HB & $7.613$ & $7.919$ & $.9119$ & $.0195$ \\
Insurance & Common & $45.838$ & $45.845$ & $.8083$ & $.0652$ \\
 & Pair-int & $41.600$ & $41.564$ & $.8183$ & $.0587$ \\
 & Full HB & $35.588$ & $35.621$ & $.8431$ & $.0502$ \\
Alarm & Common & $37.725$ & $37.758$ & $.9402$ & $.0282$ \\
 & Pair-int & $46.188$ & $46.208$ & $.9214$ & $.0344$ \\
 & Full HB & $28.725$ & $28.774$ & $.9574$ & $.0214$ \\
\bottomrule
\end{tabular}
\end{table}

\begin{table}[htbp]
\centering
\small
\caption{Primary matched-study contrast: full HB minus common reliability in SHD. Intervals are unadjusted 95\% paired-bootstrap intervals over twenty data seeds after averaging source batches. Holm correction covers the five network tests.}
\label{tab:common_matched_tests}
\begin{tabular}{lrrr}
\toprule
Network & $\Delta$SHD & 95\% CI & $p_{\mathrm{Holm}}$ \\
\midrule
Asia & $.125$ & $[-.263,.513]$ & $1.000$ \\
Sachs & $-.113$ & $[-.750,.463]$ & $1.000$ \\
Child & $.900$ & $[.250,1.575]$ & $.066$ \\
Insurance & $-10.250$ & $[-13.338,-7.250]$ & $.000048$ \\
Alarm & $-9.000$ & $[-12.075,-5.862]$ & $.00027$ \\
\bottomrule
\end{tabular}
\end{table}

Full HB has lower mean SHD than independent common reliability on Insurance and Alarm after correction (Table~\ref{tab:common_matched_tests}). Common reliability has lower equal-network mean SHD than pair-intercept. Table~\ref{tab:common_accuracy_full} gives the synthetic comparison, where all pairs share one generating accuracy.

\begin{table}[htbp]
\centering
\small
\caption{Independent reliability fits with a common generating accuracy across all pairs. Entries average ten seeds. Data-only values are the stored initializations, used as descriptive references; each fitted reliability variant then makes three additional graph calls.}
\label{tab:common_accuracy_full}
\begin{tabular}{lrrrrrr}
\toprule
Network & $\alpha$ & Initial no-KG & Common & Pair-int & Full HB & $p_{\mathrm{Holm}}$ \\
\midrule
Asia & $0$ & $5.5$ & $.6$ & $5.1$ & $9.8$ & $.027$ \\
 & $1/3$ & $5.5$ & $3.5$ & $3.1$ & $4.1$ & $1.000$ \\
 & $.5$ & $5.5$ & $4.1$ & $3.2$ & $3.0$ & $1.000$ \\
 & $.6$ & $5.5$ & $1.9$ & $2.1$ & $1.2$ & $1.000$ \\
 & $.75$ & $5.5$ & $.8$ & $1.4$ & $1.4$ & $1.000$ \\
 & $.9$ & $5.5$ & $0$ & $1.1$ & $.8$ & $.094$ \\
 & $1$ & $5.5$ & $0$ & $1.1$ & $.8$ & $.094$ \\
\midrule
Sachs & $0$ & $9.1$ & $.9$ & $9.4$ & $19.5$ & $.027$ \\
 & $1/3$ & $9.1$ & $12.7$ & $11.0$ & $11.4$ & $1.000$ \\
 & $.5$ & $9.1$ & $9.1$ & $11.3$ & $7.5$ & $1.000$ \\
 & $.6$ & $9.1$ & $5.5$ & $6.6$ & $6.0$ & $1.000$ \\
 & $.75$ & $9.1$ & $1.3$ & $4.3$ & $1.6$ & $1.000$ \\
 & $.9$ & $9.1$ & $1.0$ & $1.9$ & $1.3$ & $1.000$ \\
 & $1$ & $9.1$ & $.3$ & $1.7$ & $.8$ & $1.000$ \\
\bottomrule
\end{tabular}
\end{table}

The last column compares full HB with common reliability, correcting fourteen network--accuracy tests by Holm. At $\alpha=0$, HB-minus-common SHD is $+9.2$ on Asia and $+18.6$ on Sachs; the other corrected tests are inconclusive. The generator independently draws labels from the two incorrect states at this endpoint, so repeated labels provide exclusion information about the truth. All arms share the data-only initialization, and observations are fixed across accuracy levels within each seed.

\subsection{Alternating inference and report-feature controls}
\label{app:feedback_controls}

This study uses Asia, Sachs, and Child with $N=1000$, ten data seeds, and generating accuracies $\alpha\in\{0,1/3,.75,1\}$. Report confidence is fixed at $.5$. Every arm receives the same observations, BDeu Top-$K$ backend, base graph prior, one data-only initialization, and three graph-update calls with 20 restarts each. The primary comparison uses $\lambda_{\mathrm{warm}}=.5$. One-time estimation freezes the observation model after its first update; report-only estimation fits the same observation channel to unconstrained pair states before graph search. Neither control uses the true graph for fitting. Table~\ref{tab:feedback_main} gives aggregate results, and Table~\ref{tab:feedback_accuracy} shows their dependence on generating accuracy.

\begin{table}[htbp]
\centering
\small
\caption{Feedback controls by report accuracy. Each row averages three networks and ten data seeds. Reliability MAE compares pairwise posterior means with generating accuracy. All held-out reports are excluded from fitting.}
\label{tab:feedback_accuracy}
\begin{tabular}{clrrr}
\toprule
$\alpha$ & Estimation & SHD & Reliability MAE & Held-out log loss \\
\midrule
$0$ & Report-only & $18.13$ & $.581$ & $1.749$ \\
 & One-time & $14.97$ & $.441$ & $1.227$ \\
 & Full & $13.40$ & $.259$ & $.958$ \\
$1/3$ & Report-only & $10.43$ & $.135$ & $1.177$ \\
 & One-time & $9.37$ & $.188$ & $1.188$ \\
 & Full & $9.40$ & $.127$ & $1.159$ \\
$.75$ & Report-only & $2.53$ & $.171$ & $.897$ \\
 & One-time & $2.20$ & $.083$ & $.777$ \\
 & Full & $2.40$ & $.091$ & $.777$ \\
$1$ & Report-only & $.23$ & $.097$ & $.108$ \\
 & One-time & $1.47$ & $.249$ & $.306$ \\
 & Full & $1.30$ & $.132$ & $.153$ \\
\bottomrule
\end{tabular}
\end{table}

Five likelihood reports are used in the primary comparison, and a further five independently generated reports evaluate predictive log loss. Full alternation has lower aggregate reliability MAE and held-out log loss than one-time estimation, but the difference varies with accuracy. At $\alpha=.75$, one-time estimation has lower reliability MAE, and both log losses round to $.777$. The twelve network-by-accuracy SHD comparisons use paired two-sided tests over ten data seeds and Holm correction; their adjusted values are all at least $.787$. Additional warm-strength comparisons at $.1$ and $1$ are retained in the complete result files.

\paragraph{Report features and information budget.}
The reused-feature arm computes features from the five likelihood reports. The intercept-only arm removes these pair summaries while continuing to estimate the intercept and prior concentration. The disjoint-feature arm uses five additional reports to compute features; its total report budget is ten, matched by the reused-$R=10$ arm. All use the same five held-out reports, which are excluded from both features and likelihood. Table~\ref{tab:feature_controls_accuracy} gives the complete accuracy profile.

\begin{table}[htbp]
\centering
\small
\caption{Report-feature controls: mean SHD over three networks and ten seeds. Disjoint $R=5$ has five likelihood reports and five auxiliary feature reports; reused $R=10$ uses the same total report budget. No-KG has mean SHD $11.10$ at every accuracy.}
\label{tab:feature_controls_accuracy}
\begin{tabular}{crrrr}
\toprule
$\alpha$ & Reused $R=5$ & Intercept $R=5$ & Disjoint $R=5$ & Reused $R=10$ \\
\midrule
$0$ & $13.40$ & $7.53$ & $12.07$ & $10.83$ \\
$1/3$ & $9.40$ & $8.57$ & $8.57$ & $7.53$ \\
$.75$ & $2.40$ & $3.83$ & $2.27$ & $2.50$ \\
$1$ & $1.30$ & $1.93$ & $1.30$ & $1.17$ \\
\bottomrule
\end{tabular}
\end{table}

The intercept-only prior has lower SHD at $\alpha=0$, and the feature prior has lower SHD at $\alpha=.75$. The disjoint-$R=5$ and reused-$R=10$ arms each use ten reports, allocated differently between feature construction and the likelihood. Table~\ref{tab:feature_controls_accuracy} shows the resulting accuracy profile.

\subsection{Using fully accurate reports}
\label{app:accurate_kg}

We provide error-free positive-edge reports for five rounds and treat unreported pairs as abstentions. Thirty paired data seeds are used on each of five networks. HB estimates reliability; the alternatives use the frozen channel at $w=.99$ or enforce the asserted edges directly. Table~\ref{tab:accurate_kg_comparison} reports SHD.

\begin{table}[htbp]
\centering
\small
\caption{Fully accurate positive-edge reports: mean SHD over thirty paired data seeds per network. The frozen channel uses $w=.99$, $\epsilon=.05$; hard use enforces the asserted edges. Unreported pairs abstain.}
\label{tab:accurate_kg_comparison}
\begin{tabular}{lcccc}
\toprule
Network & No-KG & HB & Fixed channel $w=.99$ & Hard use \\
\midrule
Alarm & $46.20$ & $2.50$ & $1.63$ & $0$ \\
Asia & $5.00$ & $.83$ & $0$ & $0$ \\
Child & $16.43$ & $3.17$ & $1.57$ & $0$ \\
Insurance & $42.30$ & $13.50$ & $9.67$ & $0$ \\
Sachs & $12.10$ & $1.07$ & $.10$ & $0$ \\
\bottomrule
\end{tabular}
\end{table}

HB lowers SHD on every network, with mean reduction $20.1933$ over 150 paired comparisons. The frozen channel at $w=.99$ lowers SHD by a further $1.62$, and hard use by $4.2133$. HB obtains approximately $93\%$ of the SHD gain achieved by the frozen channel. Table~\ref{tab:matched_reports_main}(a) gives the corresponding LLM-report comparison.

\subsection{Accuracy over the full range and the number of reports}
\label{app:accuracy_calibration}

We vary the generating accuracy over $\alpha\in\{0,1/3,0.5,0.6,0.75,0.9,1\}$ on Asia and Sachs, using three warm strengths and ten seeds. This gives 420 HB runs. Here $\alpha$ is the probability of a correct report, and $1/3$ is the accuracy of uniform guessing over the three states. In every network-by-warm setting, $\bar\rho$ increases strictly with $\alpha$ (Spearman correlation $1.0$; Pearson correlation $0.995$--$0.999$). Averaging over warm strengths gives the fitted relationships
\[
\bar\rho=0.276+0.561\alpha\quad\text{(Asia; }R^2=0.997\text{)},\qquad
\bar\rho=0.264+0.593\alpha\quad\text{(Sachs; }R^2=0.994\text{)}.
\]
Reliability is monotone in generating accuracy; the fitted slopes below one show shrinkage at this report count. Figure~\ref{fig:full_accuracy_calibration} shows the sweep at $\lambda_{\mathrm{warm}}=.5$ and the effect of adding reports.

\begin{figure}[htbp]
\centering
\includegraphics[width=\linewidth]{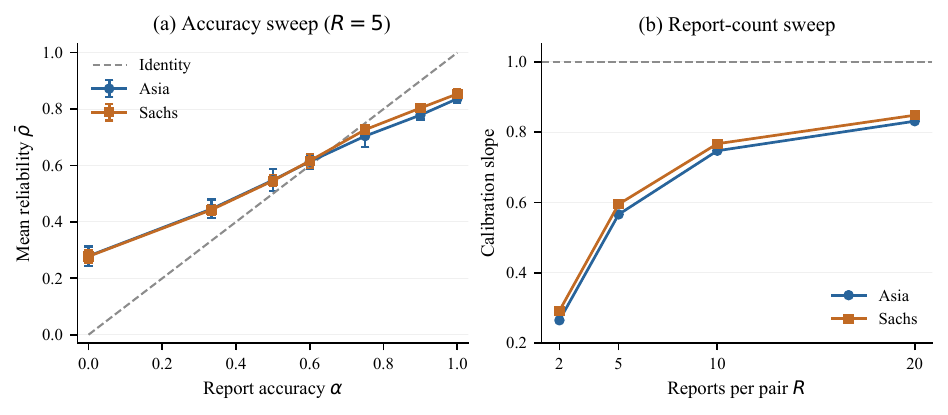}
\caption{Reliability across accuracy and report count. Left: the seven-level accuracy sweep at $R=5$ and $\lambda_{\mathrm{warm}}=0.5$; error bars are standard deviations across ten seeds. The dashed line marks exact calibration. Right: slopes of the fitted reliability--accuracy relationship in the report-count sweep.}
\label{fig:full_accuracy_calibration}
\end{figure}

\begin{table}[htbp]
\centering
\small
\caption{Report-count sweep. The first four columns give the slope of learned reliability against generating accuracy. The last two give calibration MAE at the smallest and largest report counts.}
\label{tab:report_count_calibration}
\begin{tabular}{lrrrrrr}
\toprule
 & \multicolumn{4}{c}{Slope} & \multicolumn{2}{c}{MAE} \\
\cmidrule(lr){2-5}\cmidrule(lr){6-7}
Network & $R=2$ & $R=5$ & $R=10$ & $R=20$ & $R=2$ & $R=20$ \\
\midrule
Asia & $.265$ & $.566$ & $.747$ & $.831$ & $.197$ & $.045$ \\
Sachs & $.292$ & $.595$ & $.767$ & $.848$ & $.187$ & $.043$ \\
\bottomrule
\end{tabular}
\end{table}

The calibration slopes move toward one and MAE decreases with report count (Table~\ref{tab:report_count_calibration}). At $\alpha=0$, the HB-minus-no-KG SHD gap falls from $+5.0$ to $+1.8$ on Asia and from $+10.0$ to $+1.2$ on Sachs as $R$ increases from two to twenty. Additional reports attenuate the excess error. The preceding corruption curves vary missing, spurious, flipped, and mixed reports, each changing the observation pattern differently.

\subsection{Heterogeneous and correlated errors}
\label{app:error_patterns}

At marginal report accuracy $0.75$, we cross homogeneous errors with two pair groups having opposing $90/10$ error biases, and independent reports with reports that copy the first report with probability $0.8$. Table~\ref{tab:heterogeneous_reports} gives mean paired SHD differences over ten seeds. The shared-$B$ model retains mean SHD gains in all four conditions, including combined heterogeneity and correlation.

\begin{table}[htbp]
\centering
\small
\caption{Shared-$B$ inference under heterogeneous and correlated reports. Differences are HB minus no-KG; negative SHD differences favor HB.}
\label{tab:heterogeneous_reports}
\begin{tabular}{lrr}
\toprule
Error model & Asia $\Delta$SHD & Sachs $\Delta$SHD \\
\midrule
Homogeneous, independent & $-4.5$ & $-6.8$ \\
Heterogeneous, independent & $-4.3$ & $-5.1$ \\
Homogeneous, correlated & $-3.4$ & $-6.3$ \\
Heterogeneous, correlated & $-2.6$ & $-5.9$ \\
\bottomrule
\end{tabular}
\end{table}

The shared error matrix $B$ pools the two groups' error counts. The remaining SHD gains show that this approximation still permits useful report integration under the tested biases and dependence.

\subsection{Abstention and selective refusal}
\label{app:abstention}

We vary the abstention rate over $\{0,0.25,0.5,0.75,0.9,1\}$ at report accuracy $0.75$ on Asia and Sachs, with ten paired seeds. Independent thinning is compared with two informative missingness mechanisms at the same marginal abstention rates: incorrect reports abstain more often (selective refusal), or correct reports abstain more often (adverse refusal). Both use an odds ratio of four. Table~\ref{tab:abstention_levels} gives the paired SHD differences at each abstention level.

\begin{figure}[htbp]
\centering
\includegraphics[width=\linewidth]{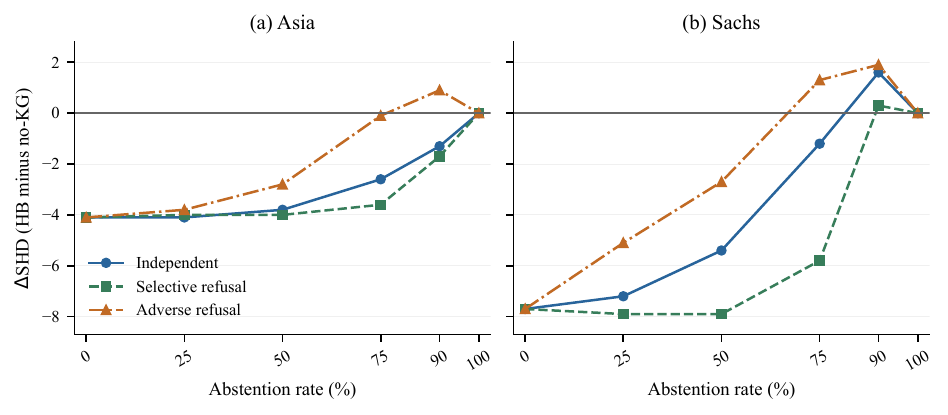}
\caption{Sensitivity to abstention under independent and informative refusal. Each point is the mean paired SHD difference over ten seeds. Negative values favor HB; zero is the no-KG reference. All three mechanisms recover no-KG at 100\% abstention.}
\label{fig:abstention_sensitivity}
\end{figure}

\begin{table}[htbp]
\centering
\small
\caption{Mean HB-minus-no-KG SHD under six abstention levels. Each entry gives Asia / Sachs. The curves are shown in Figure~\ref{fig:abstention_sensitivity}.}
\label{tab:abstention_levels}
\begin{tabular}{@{}lrrrrrr@{}}
\toprule
Mechanism & $0\%$ & $25\%$ & $50\%$ & $75\%$ & $90\%$ & $100\%$ \\
\midrule
Independent & $-4.1/-7.7$ & $-4.1/-7.2$ & $-3.8/-5.4$ & $-2.6/-1.2$ & $-1.3/+1.6$ & $0/0$ \\
Selective refusal & $-4.1/-7.7$ & $-4.0/-7.9$ & $-4.0/-7.9$ & $-3.6/-5.8$ & $-1.7/+.3$ & $0/0$ \\
Adverse refusal & $-4.1/-7.7$ & $-3.8/-5.1$ & $-2.8/-2.7$ & $-.1/+1.3$ & $+.9/+1.9$ & $0/0$ \\
\bottomrule
\end{tabular}
\end{table}

Selective refusal preserves gains over a wider range of coverage, whereas preferentially missing correct reports can reverse them. At $75\%$ abstention on Sachs, mean SHD differences are $-5.8$ under selective refusal and $+1.3$ under adverse refusal. At $100\%$, every run equals no-KG in SHD and E-SHD, as Proposition~\ref{prop:data_only_reduction} predicts. At partial coverage, surviving-report quality matters as well as report count. Abstentions add no report counts, while their rate remains a prior feature.

\subsection{Increasing the search budget}
\label{app:search_budget}

This comparison uses the five-network data and full HB fits from Appendix~\ref{app:reliability_pooling}: 20 data seeds, 1,000 observations per dataset, and four batches of five canonical-name reports. We fit no-KG once per data seed and the fixed channel once per data-seed--report-batch pair, at both 20 and 80 restarts. This gives 1,000 control fits. Their BDeu likelihood, base edge probability $.01$, maximum of four parents, search rules, and cap of 60 retained DAGs are held fixed. The fixed channel keeps the same report rewards with $w=.8$ and $\epsilon=.05$; no-KG has zero report reward. The HB references use warm strength $.5$, tolerance $.001$, and up to three updates after initialization. They use 60 total restarts in 137 fits and 80 in 263 fits.

Each control makes one search call at its specified budget. The 20-restart setting uses ten orders from the backend's randomized prior-based policy and ten unrestricted random permutations; the 80-restart setting uses forty of each. Order generation and local search share a random stream initialized by the data seed. The resulting graphs are scored and weighted under the same objective at both budgets. Table~\ref{tab:search_budget_main} gives all network means.

Paired comparisons first average four report batches within each data seed. Two-sided Wilcoxon tests use Holm correction across the ten comparisons between HB and the two 80-restart controls on five networks. Table~\ref{tab:search_budget_tests} gives paired mean differences and unadjusted 95\% bootstrap intervals from 10,000 resamples of the 20 seed-level pairs. Differences between each control's 80- and 20-restart results form a separate ten-test family. Increasing the fixed-channel budget lowers SHD on Insurance and Alarm by $4.36$ and $3.93$ ($p_{\mathrm{Holm}}=.00027$ and $.0056$). For no-KG, the corrected difference is significant on Child, with SHD reduced by $4.70$ ($p_{\mathrm{Holm}}=.027$).

\begin{table}[htbp]
\centering
\small
\caption{HB minus each 80-restart control in SHD. Brackets give unadjusted 95\% paired-bootstrap intervals; $p_{\mathrm{H}}$ adjusts the ten network--comparator Wilcoxon tests. Negative differences favor HB.}
\label{tab:search_budget_tests}
\begin{tabular}{lrrrr}
\toprule
 & \multicolumn{2}{c}{HB minus no-KG} & \multicolumn{2}{c}{HB minus fixed channel} \\
\cmidrule(lr){2-3}\cmidrule(lr){4-5}
Network & Difference [95\% CI] & $p_{\mathrm{H}}$ & Difference [95\% CI] & $p_{\mathrm{H}}$ \\
\midrule
Asia & $-1.33\;[-1.78,-.86]$ & $.0044$ & $-.75\;[-1.10,-.36]$ & $.0264$ \\
Sachs & $-3.96\;[-5.45,-2.51]$ & $.0039$ & $.56\;[.00,1.14]$ & $.0925$ \\
Child & $-4.74\;[-6.68,-2.84]$ & $.0054$ & $.64\;[.04,1.20]$ & $.0925$ \\
Insurance & $-5.49\;[-8.89,-2.16]$ & $.0280$ & $3.39\;[1.51,5.46]$ & $.0280$ \\
Alarm & $-16.14\;[-19.94,-11.89]$ & $.00019$ & $3.96\;[1.66,6.43]$ & $.0264$ \\
\bottomrule
\end{tabular}
\end{table}

\begin{table}[htbp]
\centering
\small
\caption{Secondary search-budget metrics, averaged equally over networks after averaging source batches within data seed. E-SHD uses ordered-adjacency error; AUROC gives half credit to tied scores.}
\label{tab:search_budget_secondary}
\begin{tabular}{lrrr}
\toprule
Method and restarts & E-SHD $\downarrow$ & AUROC $\uparrow$ & Brier $\downarrow$ \\
\midrule
No-KG, 20 & $24.1491$ & $.8965$ & $.06144$ \\
No-KG, 80 & $22.8228$ & $.9065$ & $.05088$ \\
Fixed channel, 20 & $16.4533$ & $.9269$ & $.03882$ \\
Fixed channel, 80 & $14.7927$ & $.9336$ & $.03552$ \\
HB, $\leq80$ total & $16.1730$ & $.9160$ & $.03642$ \\
\bottomrule
\end{tabular}
\end{table}

\section{Procedurally generated synthetic ER checks}
\label{app:continuous_synth}

\paragraph{Prior structure of the Harmonized-Prior baseline.}
The Harmonized-Prior~\citep{ban2025harmonized} (Algorithm~2, Eqs.~17--19) is a fixed-confidence path-existence soft prior: each LLM-asserted pair is read as a \emph{path-existence} claim and directed reachability in $G$ is rewarded via a shared confidence $c$. We sweep $c\in\{0.60,0.70,0.80,0.90,0.99,0.99999\}$ as in their Figure~6 and report the best setting ($c{=}0.99999$) in the main text and joint-panel rows. Even at its best sweep point, the Harmonized-Prior falls below the no-KG BGES-style baseline on the 11 LLM-asserted Sachs edges in the graph-only regime, and in the joint regimes it has nearly equal AUROC to DiBS on the linear SEM and lower AUROC on the nonlinear SEM. A wrong-direction pair claim can still be satisfied by a detouring chain of intermediate edges, and the reward then pulls those chains into $G$; the prior cannot adapt this tension on a per-pair basis.

The following ER studies examine graph-only and joint inference on procedurally generated structures. A graph-only spot check uses two $d=20$ families. The joint study uses five graph draws per SEM and dimension, five observational datasets per graph, and three report accuracies per dataset, giving 75 paired conditions. These generated structures avoid dependence on LLM benchmark familiarity. The separate accuracy and report-count studies use fixed reported confidence $.5$ and assess reliability against known generating accuracy (Section~\ref{sec:exp_mechanism}).

\paragraph{Graph-only spot check (ER2/ER4, $d{=}20$).}
We generate two $d{=}20$ ER DAG families, simulate linear-Gaussian SEM data, and construct controlled noisy KGs by retaining a subset of true edges while injecting spurious edges. We then compare data-only DiBS against HB-NoisyKG in Table~\ref{tab:continuous_synth_d20}. Across these two small-budget checks HB-NoisyKG remains competitive overall and is stronger on the denser \textsc{ER4-d20} family.
\begin{table}[h]
\caption{Small synthetic continuous graph-only check on two $d{=}20$ Erd\H{o}s--R\'enyi linear-Gaussian families with controlled noisy KGs. All values are mean $\pm$ std over 5 repeats.}
\label{tab:continuous_synth_d20}
\centering
\small
\begin{tabular}{l l c c c c}
\toprule
Graph family & Method & AUROC $\uparrow$ & E-SHD $\downarrow$ & SHD $\downarrow$ & $\bar{\rho}$ \\
\midrule
\textsc{ER2-d20} & DiBS & $0.956 \pm 0.032$ & $21.20 \pm 11.17$ & $21.2 \pm 11.2$ & -- \\
\textsc{ER2-d20} & HB-NoisyKG & $0.941 \pm 0.048$ & $20.00 \pm 15.28$ & $20.0 \pm 15.3$ & $0.400$ \\
\midrule
\textsc{ER4-d20} & DiBS & $0.834 \pm 0.071$ & $51.20 \pm 18.50$ & $51.2 \pm 18.5$ & -- \\
\textsc{ER4-d20} & HB-NoisyKG & $\mathbf{0.843 \pm 0.052}$ & $\mathbf{42.60 \pm 19.86}$ & $\mathbf{42.6 \pm 19.9}$ & $0.405$ \\
\bottomrule
\end{tabular}
\end{table}

\paragraph{Paired-test study ($d\in\{20,30\}$).}
The joint study uses ER graphs at $d\in\{20,30\}$, expected degree 2, $N=500$ observations per dataset, generating report accuracies $\alpha\in\{.60,.75,.90\}$, and $R=5$ rounds. HB and DiBS-joint share the observational datasets, likelihood, and DiBS backend~\citep{lorch2021dibs}. Each graph call uses 20 particles and 2,000 SVGD steps. HB uses three outer graph updates; no-KG uses one. HB uses base edge probability $.01$, whereas the no-KG configuration uses the backend default $0.5/(d-1)$. Thus total computation and the base prior differ as well as the knowledge term. We report the structural outcomes descriptively and use the matched studies in the main text for attribution to the report layer. Both linear-Gaussian and nonlinear additive-noise MLP SEMs use the same five-graph design; the nonlinear model follows \citet{lorch2021dibs}, Appendix~H.1 (one hidden layer of five ReLU units, $\mathcal{N}(0,1)$ weights, and noise variance $.1$). Table~\ref{tab:e2_synth} gives absolute outcomes; graph-level uncertainty is assessed below.

\begin{table}[tbp]
\centering
\small
\caption{Structural recovery in the synthetic ER comparison. Means cover five graphs, five data seeds per graph, and three report accuracies per dataset. Both metrics are lower-is-better. Repeated conditions are averaged within graphs for the analysis in Appendix~\ref{app:e2_graph_blocks}, with the base priors and total graph-update budgets specified below.}
\label{tab:e2_synth}
\begin{tabular}{lcrrrr}
\toprule
 & & \multicolumn{2}{c}{SHD $\downarrow$} & \multicolumn{2}{c}{E-SHD $\downarrow$} \\
\cmidrule(lr){3-4}\cmidrule(lr){5-6}
SEM & $d$ & DiBS & HB & DiBS & HB \\
\midrule
Linear & $20$ & $21.96$ & $20.57$ & $21.96$ & $20.57$ \\
Linear & $30$ & $48.59$ & $45.20$ & $48.59$ & $45.20$ \\
\midrule
Nonlinear & $20$ & $26.12$ & $21.41$ & $26.12$ & $21.42$ \\
Nonlinear & $30$ & $43.08$ & $41.00$ & $43.08$ & $40.99$ \\
\bottomrule
\end{tabular}
\end{table}

\subsection{Graph-level analysis and comparison methods}
\label{app:e2_graph_blocks}

The three report accuracies reuse each observational dataset. We average HB-minus-DiBS differences over accuracies within each data seed, then over seeds within each graph, assigning equal weight to the five graph draws. Table~\ref{tab:e2_graph_tests} reports exact two-sided sign-flip tests on structural effects, assuming independent symmetric graph-level effects under the null. Holm correction covers the four $d=20/30$ structural comparisons. With five graphs, the smallest attainable two-sided $p$ is $.0625$.

The comparisons below report structural outcomes; AUROC handling is specified in Appendix~\ref{app:implementation}. The Exact/Top-$K$ stability study separately uses tie-corrected edge scores. The ER base-prior and total-budget choices are given above.

\begin{table}[htbp]
\centering
\small
\caption{Graph-level E2 structural analysis. Differences are HB minus DiBS. Individual graph effects appear in Table~\ref{tab:e2_graph_effects}; $p$ is the exact sign-flip value and $p_H$ is adjusted across four structural tests.}
\label{tab:e2_graph_tests}
\begin{tabular}{lclrrr}
\toprule
SEM & $d$ & Metric & Mean difference & $p$ & $p_H$ \\
\midrule
Linear & 20 & E-SHD & $-1.39$ & $.375$ & $.500$ \\
Linear & 30 & E-SHD & $-3.39$ & $.0625$ & $.250$ \\
Nonlinear & 20 & E-SHD & $-4.70$ & $.0625$ & $.250$ \\
Nonlinear & 30 & E-SHD & $-2.09$ & $.250$ & $.500$ \\
\bottomrule
\end{tabular}
\end{table}

\begin{table}[htbp]
\centering
\small
\caption{Individual graph effects underlying Table~\ref{tab:e2_graph_tests}. Each graph effect averages five data seeds and three report accuracies.}
\label{tab:e2_graph_effects}
\begin{tabular}{lclrrrrr}
\toprule
SEM & $d$ & Metric & Graph 1 & Graph 2 & Graph 3 & Graph 4 & Graph 5 \\
\midrule
Linear & 20 & E-SHD & $2.53$ & $-4.53$ & $.73$ & $-2.33$ & $-3.33$ \\
Linear & 30 & E-SHD & $-1.20$ & $-4.53$ & $-2.67$ & $-1.60$ & $-6.93$ \\
Nonlinear & 20 & E-SHD & $-5.47$ & $-4.80$ & $-1.87$ & $-6.20$ & $-5.18$ \\
Nonlinear & 30 & E-SHD & $2.53$ & $-6.39$ & $-3.73$ & $-1.20$ & $-1.67$ \\
\bottomrule
\end{tabular}
\end{table}

The linear study also includes NOTEARS, GOLEM, and KG-SoftMAP (Table~\ref{tab:e2_linear_all_methods}). NOTEARS and GOLEM operate on continuous observations. KG-SoftMAP discretizes each variable at its median and performs greedy BDeu MAP search. NOTEARS and GOLEM achieve lower SHD on these linear generators than the tested HB and DiBS joint-inference configurations.

\begin{table}[htbp]
\centering
\small
\caption{Linear E2 structural comparisons across all evaluated methods. Each entry is mean SHD over 75 conditions based on 25 observational datasets and five graph draws. Method implementations are described in the text.}
\label{tab:e2_linear_all_methods}
\begin{tabular}{lcc}
\toprule
Method & $d=20$ & $d=30$ \\
\midrule
NOTEARS & $.60$ & $1.64$ \\
GOLEM & $3.68$ & $9.96$ \\
KG-SoftMAP (discretized BDeu) & $7.43$ & $14.80$ \\
HB-NoisyKG & $20.57$ & $45.20$ \\
DiBS-joint & $21.96$ & $48.59$ \\
\bottomrule
\end{tabular}
\end{table}

\paragraph{Reliability response in the nonlinear study.}
Table~\ref{tab:e2_synth_rho_calib} reports the mean learned reliability at each generating accuracy. It rises monotonically at both dimensions, approximately following $.71\alpha+.10$ at $d=20$ and $.74\alpha+.08$ at $d=30$. In this ER study, reported confidence equals $\alpha$, so both labels and confidence change across these conditions. The accuracy and report-count study in Appendix~\ref{app:accuracy_calibration} separately evaluates reliability with controlled reports. The concentration is initialized at $\tau=8$ and updated by empirical Bayes.

\begin{table}[h]
\caption{Learned reliability in the nonlinear E2 study. Each row averages five graphs and five data seeds. Report accuracy and reported confidence both equal $\alpha$ in this generator; Appendix~\ref{app:accuracy_calibration} supplies a separate controlled-report study.}
\label{tab:e2_synth_rho_calib}
\centering
\small
\begin{tabular}{l c c c}
\toprule
$d$ & $\alpha$ & $n$ & $\bar\rho$ (mean $\pm$ std) \\
\midrule
$20$ & $0.60$ & $25$ & $0.523 \pm 0.035$ \\
$20$ & $0.75$ & $25$ & $0.630 \pm 0.031$ \\
$20$ & $0.90$ & $25$ & $0.736 \pm 0.023$ \\
\midrule
$30$ & $0.60$ & $25$ & $0.529 \pm 0.010$ \\
$30$ & $0.75$ & $25$ & $0.636 \pm 0.015$ \\
$30$ & $0.90$ & $25$ & $0.750 \pm 0.018$ \\
\bottomrule
\end{tabular}
\end{table}

\paragraph{Per-$\alpha$ breakdown (linear, exploratory).}
At $d=20$, mean HB SHD is $19.36$, $20.96$, and $21.40$ for $\alpha=.60,.75,.90$. At $d=30$, it is $43.84$, $44.96$, and $46.80$. Reliability ordering and graph recovery capture different aspects of inference: the increase in learned reliability with accuracy is accompanied here by increasing mean SHD. Table~\ref{tab:e2_graph_tests} summarizes the repeated structural comparisons at the graph level.

\paragraph{Per-$\alpha$ breakdown (nonlinear).}
The mean nonlinear E-SHD difference favors HB at each tested accuracy, with reductions of $4.44$--$5.08$ at $d=20$ and $.80$--$2.76$ at $d=30$ (Table~\ref{tab:e2_synth_nonlinear_breakdown}). For each accuracy, the table averages seeds within each graph before computing the exact graph-level sign-flip value.

\begin{table}[h]
\caption{Nonlinear E2 results by report accuracy. Differences are HB minus DiBS. Each row includes 25 datasets from five graphs; $p$-values are raw two-sided graph-level sign-flip values after averaging the five seeds within each graph. $\bar\rho$ is mean learned reliability.}
\label{tab:e2_synth_nonlinear_breakdown}
\centering
\small
\begin{tabular}{c c c c c c}
\toprule
$d$ & $\alpha$ & $n$ & $\Delta$E-SHD $\downarrow$ & $p_{\Delta\mathrm{E-SHD}}$ & $\bar\rho$ \\
\midrule
$20$ & $0.60$ & $25$ & $-4.59$ & $.0625$ & $0.523$ \\
$20$ & $0.75$ & $25$ & $-4.44$ & $.0625$ & $0.630$ \\
$20$ & $0.90$ & $25$ & $-5.08$ & $.0625$ & $0.736$ \\
\midrule
$30$ & $0.60$ & $25$ & $-0.80$ & $.750$ & $0.529$ \\
$30$ & $0.75$ & $25$ & $-2.72$ & $.0625$ & $0.636$ \\
$30$ & $0.90$ & $25$ & $-2.76$ & $.125$ & $0.750$ \\
\bottomrule
\end{tabular}
\end{table}

\paragraph{Dimension extension to $d{=}50$.}
The linear $d=50$ extension uses five graphs, five data seeds per graph, and the same three report accuracies and joint backend. Mean E-SHD decreases by $3.57$ relative to DiBS, with graph-level sign-flip $p=.125$. Table~\ref{tab:e2_synth_d50} gives the accuracy-specific results. Mean reliability increases from approximately $.58$ to $.70$ and $.78$. Per-run HB wall time is approximately 15--17 minutes on A10 GPUs.

\begin{table}[h]
\caption{Linear $d=50$ extension. Differences are HB minus DiBS. There are 25 datasets per accuracy from five graphs. Raw $p$-values use the graph-level two-sided sign-flip test; the pooled row first averages accuracies and seeds within each graph.}
\label{tab:e2_synth_d50}
\centering
\small
\begin{tabular}{c c c c c c}
\toprule
$d$ & $\alpha$ & $n$ & $\Delta$E-SHD $\downarrow$ & $p_{\Delta\mathrm{E-SHD}}$ & $\bar\rho$ \\
\midrule
$50$ & $0.60$ & $25$ & $-3.16$ & $.1875$ & $0.576$ \\
$50$ & $0.75$ & $25$ & $-1.04$ & $.5625$ & $0.704$ \\
$50$ & $0.90$ & $25$ & $-6.51$ & $.250$ & $0.781$ \\
\midrule
\multicolumn{2}{l}{Pooled (75 conditions)} & $75$ & $-3.57$ & $.125$ & --- \\
\bottomrule
\end{tabular}
\end{table}

\paragraph{Computation.}
The nonlinear study uses two Alibaba Cloud \texttt{ecs.gn7i-4x.16xlarge} hosts, each with four A10 24\,GB GPUs, 48 vCPU, and 256\,GB RAM, running CUDA 12.8 and JAX 0.6.2. One worker uses each GPU. With twenty particles and 2,000 steps per update, HB takes $18$--$19$ minutes and DiBS $4$--$5$ minutes per condition, about $24$ minutes paired. The 300 method runs use approximately 56 GPU-hours. The linear $d=50$ extension uses one \texttt{ecs.gn7i-8x.32xlarge} host with eight A10 GPUs and eight workers. Data are indexed by graph and seed, reports also by accuracy, with seed base 7000.

\subsection{Joint inference with a larger computational budget}
\label{app:joint_budget}

The budget comparison uses one $d=20$ graph at $\alpha=.75$ and five paired seeds. The standard twenty-particle, 2,000-step results and the forty-particle, 4,000-step results share the graph, observations, reports, and likelihood. Table~\ref{tab:joint_budget_comparison} reports the two per-call budgets; the HB and no-KG update counts follow the study protocol above.

\begin{table}[htbp]
\centering
\small
\caption{Joint inference at two per-call SVGD budgets. Each result is mean E-SHD over five seeds. The last column gives HB minus DiBS differences. The two budgets use the same selected graph, observations, reports, and seeds.}
\label{tab:joint_budget_comparison}
\begin{tabular}{llccc}
\toprule
Family & Particles/steps & DiBS & HB & $\Delta$E-SHD \\
\midrule
Linear & $20/2{,}000$ & $20.0$ & $26.0$ & $+6.0$ \\
Linear & $40/4{,}000$ & $18.4$ & $15.0$ & $-3.4$ \\
Nonlinear & $20/2{,}000$ & $22.6$ & $18.4$ & $-4.2$ \\
Nonlinear & $40/4{,}000$ & $19.6$ & $14.6$ & $-5.0$ \\
\bottomrule
\end{tabular}
\end{table}

The larger budget changes the linear HB-minus-DiBS mean E-SHD difference from $+6.0$ to $-3.4$ and increases the nonlinear advantage from $4.2$ to $5.0$. Within HB, mean E-SHD falls from $26.0$ to $15.0$ on the linear example and from $18.4$ to $14.6$ on the nonlinear example. These single-graph comparisons show sensitivity to inference budget. The multi-graph ER comparison is described separately above, with its distinct priors, total budgets, and graph-level uncertainty.

\subsection{Accurate reports in joint inference}
\label{app:joint_perfect_kg}

With fully accurate reports and the stronger joint budget, we compare no-KG, HB, and the frozen channel at $w=.99$ over ten paired runs per family (Table~\ref{tab:joint_accurate_reports}). HB's E-SHD differences are $+.5$ and $-.3$ for the linear and nonlinear examples. The frozen channel reduces E-SHD by $3.9$ and $1.5$. These cases evaluate accurate-knowledge utilization in joint inference.

\begin{table}[htbp]
\centering
\small
\caption{Joint inference with accurate reports. Entries are mean E-SHD over ten paired runs. All methods use 40 particles and 4,000 steps per graph call.}
\label{tab:joint_accurate_reports}
\begin{tabular}{lccc}
\toprule
Family & No-KG & HB & Fixed channel $w=.99$ \\
\midrule
Linear & $20.9$ & $21.4$ & $17.0$ \\
Nonlinear & $22.7$ & $22.4$ & $21.2$ \\
\bottomrule
\end{tabular}
\end{table}

\section{Empirical recovery of \texorpdfstring{$\rho$ and $B$}{rho and B} under known \texorpdfstring{$G$}{G}}
\label{app:rho_b_recovery}

With the graph held at ground truth, we evaluate recovery of pair reliability and the shared error matrix as the number of reports increases. This simulation exercises the observation layer independently of graph estimation.

\paragraph{Setup.}
We fix an Erd\H{o}s--R\'enyi DAG $G^\star$ over $d{=}15$ nodes with expected degree $2$, yielding $105$ unordered pairs and a deterministic latent state $Z^\star_{ij}\in\{\textsc{none},\textsc{fwd},\textsc{rev}\}$ for each pair. True reliabilities $\rho^\star_{ij}$ are drawn once from a two-component Beta mixture ($0.5\,\mathrm{Beta}(5,2) + 0.5\,\mathrm{Beta}(2,5)$) so the population spans both trustworthy and unreliable pairs. We study two ground-truth confusion matrices $B^\star$: a symmetric off-diagonal row distribution (each row uniform on the two non-matching states) and an asymmetric one (rows $[0,0.7,0.3]$, $[0.6,0,0.4]$, $[0.7,0.3,0]$). For each pair we draw $R$ observations as $Y^{(r)}_{ij} = Z^\star_{ij}$ with probability $\rho^\star_{ij}$ and $Y^{(r)}_{ij}\sim B^\star[Z^\star_{ij},\cdot]$ otherwise, with $R\in\{3,5,10\}$ and $50$ Monte Carlo seeds per $(R, B^\star)$ setting ($300$ total runs). Inference uses the paper's closed-form updates from Section~\ref{sec:reliability_posterior} in the $G$-known mode: the graph posterior is replaced by the one-hot indicator $q_{ij}(s)=\mathbb{1}[s=Z^\star_{ij}]$, so the reliability-posterior update and the confusion-posterior update reduce to the pair-wise Beta and global Dirichlet closed forms. We use the paper's default hyperparameters ($\tau{=}8$, symmetric $\mathrm{Dirichlet}(1)$ off-diagonal prior on $B$).

\paragraph{Results.}
Table~\ref{tab:rho_b_recovery} reports recovery error over fifty seeds. Reliability MAE and $B$ Frobenius error decrease with $R$ under both error matrices. Figure~\ref{fig:rho_b_recovery} shows their 95\% confidence bands. Reliability errors are similar for the symmetric and asymmetric matrices. Holding the pair states known makes match counts available for the reliability update and off-diagonal counts available for the $B$ update. The decreasing errors are consistent with the conditional concentration in Proposition~\ref{prop:reliability_concentration}.

\begin{table}[h]
\caption{$G$-known recovery of $\rho$ and $B$ as a function of round budget $R$. All values are mean $\pm$ std over $50$ Monte Carlo seeds per $(R, B^\star)$ setting. The symmetric and asymmetric rows differ only in $B^\star$; the ground-truth DAG, node labels, and $\rho^\star$ population are identical across rows.}
\label{tab:rho_b_recovery}
\centering
\small
\begin{tabular}{l c c c}
\toprule
$B^\star$ structure & $R$ & $\mathrm{MAE}(\hat{\rho}-\rho^\star)$ $\downarrow$ & $\|\hat{B}-B^\star\|_F$ $\downarrow$ \\
\midrule
symmetric   & $3$  & $0.192 \pm 0.006$ & $0.280 \pm 0.121$ \\
symmetric   & $5$  & $0.167 \pm 0.006$ & $0.196 \pm 0.114$ \\
symmetric   & $10$ & $\mathbf{0.128 \pm 0.007}$ & $\mathbf{0.173 \pm 0.091}$ \\
\midrule
asymmetric & $3$  & $0.194 \pm 0.006$ & $0.249 \pm 0.137$ \\
asymmetric & $5$  & $0.168 \pm 0.007$ & $0.216 \pm 0.111$ \\
asymmetric & $10$ & $\mathbf{0.127 \pm 0.006}$ & $\mathbf{0.163 \pm 0.090}$ \\
\bottomrule
\end{tabular}
\end{table}

\begin{figure}[h]
\centering
\begin{minipage}{0.48\textwidth}
\centering
\includegraphics[width=\linewidth]{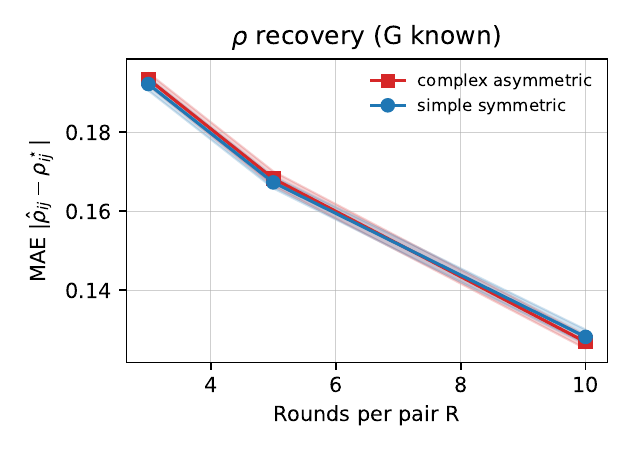}
\end{minipage}\hfill
\begin{minipage}{0.48\textwidth}
\centering
\includegraphics[width=\linewidth]{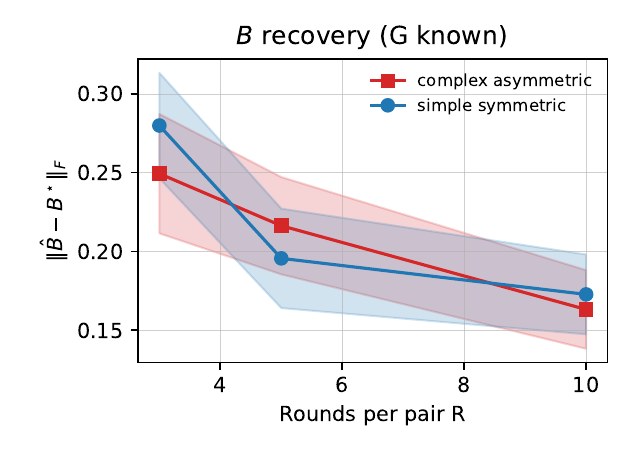}
\end{minipage}
\caption{$G$-known posterior recovery curves: $\mathrm{MAE}(\hat{\rho}-\rho^\star)$ (left) and $\|\hat{B}-B^\star\|_F$ (right) as a function of round budget $R$. Shaded bands are $95\%$ CI over $50$ Monte Carlo seeds per $(R, B^\star)$ setting.}
\label{fig:rho_b_recovery}
\end{figure}

\paragraph{Scope.}
This experiment evaluates recovery of pair reliability and shared error structure with the graph held known. It provides an empirical reference for Proposition~\ref{prop:reliability_concentration}. With uncertain graphs, the conditional finite-sample bound in Appendix~\ref{app:proofs} additionally accounts for pair-state marginal error.

\section{Hyperparameter sensitivity}
\label{app:sensitivity}

Table~\ref{tab:sensitivity} reports HB-NoisyKG sensitivity over a $3\times3\times3$ grid: $\beta_{\mathrm{var}}\in\{1,4,16\}$, $\tau_{\mathrm{shape}}\in\{1,2,4\}$, and $\lambda_{\mathrm{warm}}\in\{0.1,0.5,0.9\}$. Asia and Sachs each use ten seeds per configuration (540 total runs), with 1,000 freshly sampled observations and synthetic reports generated at retention $.7$ and spurious rate $.2$. The first two parameters specify hyperpriors for the estimated $(\beta,\tau)$; the third controls the one-time warm initialization described in Appendix~\ref{app:implementation}.

\begin{table}[h]
\caption{Hyperparameter sensitivity. Each row averages ten seeds. Most variation is associated with the warm strength $\lambda_{\mathrm{warm}}$; at a fixed strength, changes in the hyperprior parameters have smaller effects.}
\label{tab:sensitivity}
\centering
\small
\begin{tabular}{c c c c c}
\toprule
$\beta_{\mathrm{var}}$ & $\tau_{\mathrm{shape}}$ & $\lambda_{\mathrm{warm}}$ & Asia SHD $\downarrow$ & Sachs SHD $\downarrow$ \\
\midrule
1 & 1 & 0.1 & 6.3 & 11.0 \\
1 & 1 & 0.5 & 9.2 & 16.6 \\
1 & 1 & 0.9 & 10.9 & 17.3 \\
1 & 2 & 0.1 & 6.7 & 11.2 \\
1 & 2 & 0.5 & 9.2 & 16.6 \\
1 & 2 & 0.9 & 10.9 & 16.7 \\
1 & 4 & 0.1 & 7.0 & 11.4 \\
1 & 4 & 0.5 & 9.2 & 16.6 \\
1 & 4 & 0.9 & 11.3 & 16.7 \\
\midrule
4 & 1 & 0.1 & 6.3 & 11.0 \\
4 & 1 & 0.5 & 9.2 & 16.6 \\
4 & 1 & 0.9 & 10.9 & 17.3 \\
4 & 2 & 0.1 & 6.7 & 11.2 \\
4 & 2 & 0.5 & 9.2 & 16.6 \\
4 & 2 & 0.9 & 10.9 & 16.7 \\
4 & 4 & 0.1 & 7.0 & 11.9 \\
4 & 4 & 0.5 & 9.2 & 16.6 \\
4 & 4 & 0.9 & 11.3 & 16.7 \\
\midrule
16 & 1 & 0.1 & 6.3 & 11.0 \\
16 & 1 & 0.5 & 9.2 & 16.6 \\
16 & 1 & 0.9 & 10.8 & 17.3 \\
16 & 2 & 0.1 & 6.7 & 11.0 \\
16 & 2 & 0.5 & 9.2 & 16.6 \\
16 & 2 & 0.9 & 10.9 & 16.7 \\
16 & 4 & 0.1 & 7.0 & 11.6 \\
16 & 4 & 0.5 & 9.0 & 16.6 \\
16 & 4 & 0.9 & 11.0 & 16.3 \\
\bottomrule
\end{tabular}
\end{table}

\paragraph{Interpretation.}
Warm strength accounts for most variation: SHD ranges from $6.3$ to $11.3$ on Asia and from $11.0$ to $17.3$ on Sachs. Holding it fixed, variation over $\beta_{\mathrm{var}}$ and $\tau_{\mathrm{shape}}$ gives spreads of at most $.7$ and $1.0$. The warm parameter scales the initial pair-state blend described in Appendix~\ref{app:implementation}.

\paragraph{Warm strength across report accuracies.}
The additional factorial study crosses $\lambda_{\mathrm{warm}}\in\{.1,.5,.9\}$ with seven report accuracies, two networks, and ten seeds. Reliability updates respond to report quality across all tested strengths. At $\alpha=0$, mean reliability falls from $.373$--$.506$ after the first E-step to $.260$--$.326$ at the final update; at $\alpha=1$, it rises from $.728$--$.774$ to $.835$--$.872$.

At fixed accuracy, changing warm strength produces median final-$\bar\rho$ spreads of $.009$ and $.011$ on Asia and Sachs. Changing accuracy produces spreads of $.558$ and $.576$, corresponding to $59.5$ and $52.5$ times the warm-strength variation. Reliability ordering is therefore consistent across the tested initializations. Warm strength still affects graph recovery: its maximum SHD regret relative to the best tested strength is $2.2$, $1.6$, and $4.3$ for $.1$, $.5$, and $.9$. Thus $.5$ is an empirical compromise over this synthetic accuracy grid. Table~\ref{tab:sensitivity} uses a different synthetic-report mechanism with retention $.7$ and spurious rate $.2$, and favors $.1$ on both networks. The benchmark-suite configuration uses strength $1$, which is outside these grids; the grids describe sensitivity over the tested strengths.

\begin{table}[htbp]
\centering
\small
\caption{Accuracy and warm-strength study ($N=1000$, $R=5$). Each entry averages ten paired seeds. HB columns are indexed by $\lambda_{\mathrm{warm}}$; no-KG and fixed-channel SHD provide reference comparisons. The fixed channel uses $w=.8$, $\epsilon=.05$. Observations are held fixed across accuracies, so no-KG values repeat across rows.}
\label{tab:warm_accuracy_grid}
\setlength{\tabcolsep}{4pt}
\begin{tabular}{lcrrrrrrrr}
\toprule
 & & \multicolumn{3}{c}{Final $\bar\rho$ (HB)} & \multicolumn{5}{c}{SHD $\downarrow$} \\
\cmidrule(lr){3-5}\cmidrule(lr){6-10}
Network & $\alpha$ & $.1$ & $.5$ & $.9$ & No-KG & Fixed channel & HB $.1$ & HB $.5$ & HB $.9$ \\
\midrule
Asia & $0$ & $.260$ & $.279$ & $.326$ & $5.5$ & $21.3$ & $8.4$ & $10.0$ & $12.7$ \\
 & $1/3$ & $.442$ & $.446$ & $.449$ & $5.5$ & $7.3$ & $4.1$ & $4.4$ & $4.1$ \\
 & $.50$ & $.545$ & $.548$ & $.550$ & $5.5$ & $4.1$ & $3.0$ & $3.0$ & $3.9$ \\
 & $.60$ & $.614$ & $.614$ & $.621$ & $5.5$ & $.5$ & $1.7$ & $1.2$ & $.8$ \\
 & $.75$ & $.701$ & $.704$ & $.711$ & $5.5$ & $.9$ & $1.5$ & $1.4$ & $1.1$ \\
 & $.90$ & $.774$ & $.779$ & $.785$ & $5.5$ & $.0$ & $1.4$ & $.8$ & $.4$ \\
 & $1$ & $.846$ & $.837$ & $.835$ & $5.5$ & $.0$ & $1.2$ & $.8$ & $.0$ \\
\midrule
Sachs & $0$ & $.263$ & $.278$ & $.307$ & $9.1$ & $19.8$ & $19.6$ & $19.5$ & $19.3$ \\
 & $1/3$ & $.434$ & $.443$ & $.445$ & $9.1$ & $14.2$ & $11.4$ & $11.4$ & $12.2$ \\
 & $.50$ & $.541$ & $.546$ & $.550$ & $9.1$ & $6.4$ & $8.9$ & $7.1$ & $6.7$ \\
 & $.60$ & $.613$ & $.616$ & $.617$ & $9.1$ & $4.4$ & $5.8$ & $5.8$ & $5.6$ \\
 & $.75$ & $.727$ & $.727$ & $.726$ & $9.1$ & $1.4$ & $2.7$ & $1.6$ & $2.0$ \\
 & $.90$ & $.811$ & $.804$ & $.799$ & $9.1$ & $1.0$ & $1.8$ & $1.3$ & $1.3$ \\
 & $1$ & $.872$ & $.854$ & $.840$ & $9.1$ & $.5$ & $1.7$ & $.8$ & $.8$ \\
\bottomrule
\end{tabular}
\end{table}

Table~\ref{tab:warm_accuracy_grid} gives the full synthetic accuracy-by-strength comparison. Reliability increases with accuracy at every warm strength, while the SHD differences among strengths depend on the network and accuracy level. The fixed-channel control and HB share the generated reports in each network--accuracy--seed condition. For the comparison at $\lambda_{\mathrm{warm}}=.5$, two-sided paired signed-rank tests use exact sign permutations with midranks for ties, followed by Holm correction across the 14 network--accuracy comparisons. At $\alpha=0$ on Asia, $p=.00195$ and $p_{\mathrm{Holm}}=.027$.

At $\alpha=0$, HB and the fixed-channel control give SHD $10.0$ and $21.3$ on Asia, and $19.5$ and $19.8$ on Sachs. Both remain above no-KG, with SHD $5.5$ and $9.1$, respectively. At $\alpha\geq.6$, the fixed channel has $.2$--$1.4$ lower mean SHD, but none of these differences is significant after correction. These comparisons quantify the dependence on source quality without equating improvement over the fixed channel with improvement over no-KG.

\paragraph{Data-adaptive selection.}
In the synthetic report setting, we compare fixed $.5$ with a rule based on conflict between the data-only posterior and the reports, and with two-fold posterior-predictive selection over $\{.1,.5,.9\}$. Table~\ref{tab:adaptive_warm_comparison} reports differences relative to fixed $.5$. Neither rule consistently improves performance in these comparisons.

\begin{table}[htbp]
\centering
\small
\caption{Warm selection relative to fixed $\lambda_{\mathrm{warm}}=.5$. Entries are mean SHD differences; negative values favor the selector.}
\label{tab:adaptive_warm_comparison}
\begin{tabular}{lrr}
\toprule
Rule & Asia $\Delta$SHD & Sachs $\Delta$SHD \\
\midrule
Data/report conflict & $-.03$ & $+.83$ \\
Posterior predictive & $+.03$ & $+.14$ \\
\bottomrule
\end{tabular}
\end{table}

\clearpage

\section{Computational cost and inference stability}
\label{app:computational_stability}

\subsection{Runtime and larger graphs}

The observation update costs $O(d^2R)$, and graph calls dominate runtime. Figure~\ref{fig:runtime_scaling} gives no-KG/HB times of $0.14/0.34$, $0.39/1.17$, $3.49/8.80$, $15.51/48.00$, and $55.10/253.49$ seconds on Asia, Sachs, Child, Insurance, and Alarm. Log-log slopes are $3.94$ and $4.22$ ($R^2>.98$), summarizing these five benchmark measurements.

\begin{figure}[htbp]
\centering
\includegraphics[width=.9\linewidth]{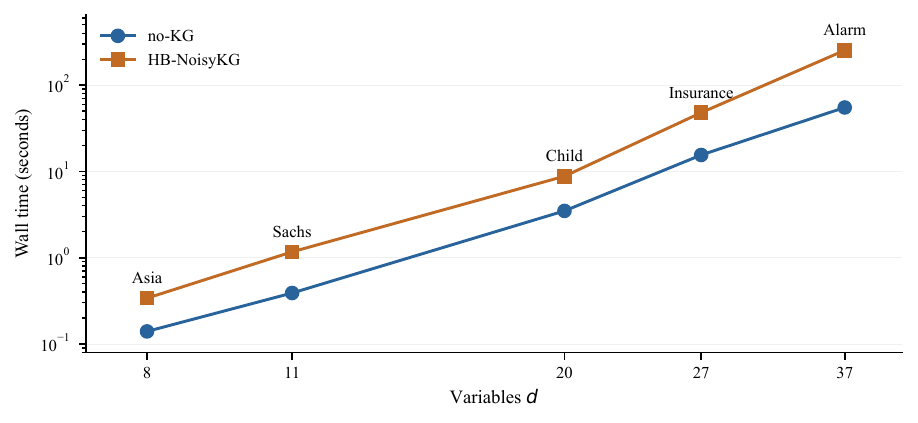}
\caption{Empirical no-KG and HB runtime on the five discrete benchmarks. Axes are logarithmic, and each point is labeled by its network. The curves use the same graph-search protocol within each comparison.}
\label{fig:runtime_scaling}
\end{figure}

The sparse binary $d=100/125$ experiments use $N=1000$, $\alpha=.75$, $R=5$, twenty restarts per graph call, and up to three HB updates over ten paired seeds. Table~\ref{tab:large_discrete_main} reports the means and computational cost. SHD reductions are $47.9$ and $55.4$ (unadjusted two-sided Wilcoxon $p=.002$), with 95\% paired-mean intervals $[-61.3,-34.5]$ and $[-60.7,-50.1]$. The intervals use $\bar d\pm t_{.975,9}s_d/\sqrt{10}$ for the HB-minus-no-KG differences.

\subsection{Stability under weak data and inaccurate reports}

The stability study extends inference to fifty updates over 5,508 settings (Table~\ref{tab:inference_stability}). Exact four-node inference uses three motifs, $N\in\{20,100\}$, and ten seeds. Top-$K$ uses Asia and Sachs at $N\in\{50,1000\}$ with ten seeds. DiBS-SVGD uses a linear Gaussian eight-node ER graph at $N\in\{50,500\}$ with three seeds. Each backend crosses accuracies $\{0,1/3,.75\}$, warm strengths $\{.1,.5,.9\}$, and independent homogeneous or heterogeneous correlated reports. Exact and Top-$K$ vary $R\in\{2,5,20\}$; DiBS uses five reports.

\paragraph{Failure criteria.}
We distinguish five outcomes. A \emph{numerical failure} is any non-finite inference quantity. \emph{Persistent period-two behavior} requires two consecutive updates whose joint state is within $10^{-3}$ of the state two iterations earlier but at least $10^{-2}$ from the immediately preceding state. \emph{Tolerance attainment} requires three consecutive joint-state changes of at most $10^{-3}$ within the 50-update budget. A \emph{boundary flag} records at least 90\% of pair reliabilities or states within $.01$ of a boundary, or a near-empty/near-maximal expected graph; a correctly concentrated posterior can trigger such a flag. Finally, \emph{statistical harm} is defined relative to paired no-KG by $\Delta\mathrm{SHD}\geq\max(5,0.5\,\mathrm{SHD}_{\mathrm{no\text{-}KG}})$ or $\Delta\mathrm{AUROC}\leq-0.10$. The primary dynamical endpoint is numerical failure or persistent period-two behavior. Non-stabilization, boundary concentration, and statistical harm are reported separately.

\begin{table}[htbp]
\centering
\small
\caption{Inference dynamics over fifty updates. Non-finite values and period-two events measure numerical failure and oscillation. Boundary flags measure concentration. Joint tolerance records attainment of the strict joint-state criterion and is sensitive to Monte Carlo variation in DiBS-SVGD.}
\label{tab:inference_stability}
\begin{tabular}{lrrrrr}
\toprule
Backend & Cells & Non-finite & Period-two & Joint tolerance & Boundary flags \\
\midrule
Exact posterior & 3,240 & 0 & 0 & $96.6\%$ & 2,143 \\
Top-$K$ & 2,160 & 0 & 2 & $81.8\%$ & 2,140 \\
DiBS-SVGD & 108 & 0 & 0 & $9/108$ & 108 \\
\bottomrule
\end{tabular}
\end{table}

\begin{table}[htbp]
\centering
\small
\caption{Performance-loss flags for Exact and Top-$K$ inference by report accuracy. ``Either'' counts the union of the unchanged SHD and AUROC criteria. The final columns give mean HB-minus-no-KG AUROC and Brier differences. AUROC uses standard handling of tied scores. Counts refer to this stress-test grid.}
\label{tab:stability_statistical_harm}
\begin{tabular}{@{}lcrrrrrr@{}}
\toprule
Backend & $\alpha$ & Cells & Either & SHD & AUROC & $\Delta$AUROC & $\Delta$Brier \\
\midrule
Exact & $0$ & 1,080 & 610 & 12 & 610 & $-.1499$ & $+.1027$ \\
 & $1/3$ & 1,080 & 561 & 0 & 561 & $-.1239$ & $+.0335$ \\
 & $.75$ & 1,080 & 485 & 0 & 485 & $-.0840$ & $-.0200$ \\
\midrule
Top-$K$ & $0$ & 720 & 474 & 289 & 457 & $-.1468$ & $+.0852$ \\
 & $1/3$ & 720 & 323 & 153 & 275 & $-.0743$ & $+.0361$ \\
 & $.75$ & 720 & 90 & 8 & 89 & $+.0084$ & $-.0210$ \\
\bottomrule
\end{tabular}
\end{table}

Table~\ref{tab:stability_statistical_harm} reports performance-loss flags for $1656/3240$ Exact and $887/2160$ Top-$K$ runs; approximately $77\%$ occur at $\alpha\leq1/3$. At $.75$, both backends improve mean Brier score, but Exact AUROC decreases. Probability error and edge ranking therefore need separate evaluation. The table uses tie-aware AUROC from the baseline and final edge marginals. DiBS is included in the numerical, cycle, and tolerance analysis; performance-loss counts are reported only for Exact and Top-$K$, whose baseline and final probability arrays are available in comparable form.

Residual damping reduces period-two events from two to one and slows tolerance attainment for Exact and Top-$K$, giving a trade-off for an optional safeguard. Practical monitoring uses an iteration cap and checks for non-finite values and persistent cycles, with data-only inference available as a fallback. SHD and AUROC flags evaluate recovery against known ground truth.

\clearpage
\section{Observation models for other graph classes}
\label{app:graph_classes}

\paragraph{Common observation protocol.}
We index the four variables by $0,\ldots,3$ and generate five independent reports for each of the six unordered pairs. A report equals the target state with probability $\alpha$ and otherwise selects uniformly among the other states; all reports are observed. The finite-candidate HB fits use $\rho_{ij}\sim\operatorname{Beta}(1,1)$ and independent $\operatorname{Dirichlet}(\mathbf{1})$ priors on the off-diagonal entries of each row of the shared error matrix $B$, whose diagonal is zero. They alternate the reliability and confusion updates with normalization over the full candidate set, using the data score plus the expected report log likelihood. The warm strength is $.5$, with at most 20 updates and a stopping tolerance of $10^{-6}$ on the largest pair-marginal change. No-KG omits report terms; the fixed comparator uses $\rho=.8$ and uniform errors over incorrect states.

\subsection{Markov equivalence and CPDAGs}

Observational data generally leave multiple Markov-equivalent DAGs plausible. Here we use reports about class-level pair states, including undirected edges, to update uncertainty over equivalence classes. To represent equivalence classes directly, we extend the pair states to include undirected edges and use a compatible CPDAG posterior.

\paragraph{CPDAG support and score.}
We enumerate all four-node DAGs and group them by skeleton and unshielded colliders, obtaining 185 Markov equivalence classes. For a class $C$, the pair states are $\{\text{none},i\to j,j\to i,i\mathbin{-}j\}$: an edge is directed only when every member DAG agrees on its orientation, and is undirected otherwise. Reports target these class states, with a $4\times4$ error matrix. Writing $\ell_{\mathrm{BDeu}}(G)$ for the binary BDeu log marginal likelihood with equivalent sample size one, the class score is
\[
 s_C=\log\!\left[\frac{1}{|C|}\sum_{G\in C}
 \exp\!\left\{\ell_{\mathrm{BDeu}}(G)
 +|E(G)|\log(.01)+(12-|E(G)|)\log(.99)\right\}\right].
\]
Class weights average member-DAG data-and-prior weights. The class prior is proportional to $.01^{e_C}.99^{12-e_C}$, with common edge count $e_C$ and uniform conditional weighting of members. Pair-state error averages $1-q(Z_{ij}=Z_{ij}^{\mathrm{true}})$ over the six pairs.

\paragraph{CPDAG data and reports.}
The generating motifs are $0\to1\to2\to3$, $0\to2\leftarrow1$ with $2\to3$, and $0\to1\leftarrow2$ with $2\to3$. Binary variables are sampled in topological order with conditional log odds $-.2+1.8\sum_r(-1)^{r+j}(X_{p_r}-.5)$ for node $j$, where $p_r$ lists its parents in increasing order starting at $r=0$. We use $N\in\{500,1000\}$, seeds $0,\ldots,19$, and $\alpha\in\{.60,.75,.90\}$. No-KG, fixed-$\rho$, and HB give $3\times2\times20\times3\times3=1{,}080$ method cells.

Table~\ref{tab:cpdag_observation_results} reports posterior mass on the true equivalence class and expected pair-state error. HB increases class mass from $.889$ to $.933$, $.925$, and $.967$ at the three accuracies, and reduces pair-state error from $.045$ to $.028$, $.031$, and $.015$.

\begin{table}[htbp]
\centering
\small
\caption{Controlled CPDAG study. Higher posterior class mass and lower expected pair-state error indicate better recovery. Values average the tested motifs, sample sizes, and seeds.}
\label{tab:cpdag_observation_results}
\begin{tabular}{lrr}
\toprule
Method & True-class mass $\uparrow$ & Pair-state error $\downarrow$ \\
\midrule
No-KG & $.889$ & $.045$ \\
HB, $\alpha=.60$ & $.933$ & $.028$ \\
HB, $\alpha=.75$ & $.925$ & $.031$ \\
HB, $\alpha=.90$ & $.967$ & $.015$ \\
\bottomrule
\end{tabular}
\end{table}

A fixed-$\rho=.8$ comparator attains true-class mass close to one in this small setting with accurate reports. The separate DAG accuracy sweep in Table~\ref{tab:warm_accuracy_grid} examines how fixed and learned trust perform as report accuracy varies.

\subsection{Hidden confounding and cyclic graphs}

\paragraph{Gaussian candidate scores.}
Both Gaussian studies use $N=1000$, seeds $0,\ldots,19$, and centered, standardized observations. With $W_{ij}$ denoting the coefficient of $i\to j$, their SEM likelihood has covariance $\Sigma=(I-W)^{-\top}\Omega(I-W)^{-1}$, where $\Omega$ is the error covariance matrix and $W_{ij}=0$ outside the candidate's directed edges. The data score is $s_G=\ell_G(\widehat W,\widehat\Omega)-\frac{k_G}{2}\log N$, where $\ell_G$ is the Gaussian log likelihood and $k_G=|E_{\to}(G)|+4+|E_{\leftrightarrow}(G)|$. Candidate priors are uniform, so no-KG weights are proportional to $\exp(s_G)$, a BIC approximation. Coefficients and error variances are fitted separately for each candidate using two L-BFGS-B starts with at most 250 iterations each, $W_{ij}\in[-1.8,1.8]$, and $\log\Omega_{ii}\in[-6,5]$. For an allowed bidirected pair, $\Omega_{ij}=.9\tanh(\eta_{ij})\sqrt{\Omega_{ii}\Omega_{jj}}$ with $|\eta_{ij}|\leq2.5$; all other off-diagonal entries are zero. Cyclic fits additionally require the spectral radius of $W$ to be below $.995$.

\paragraph{Hidden confounding.}
Before standardization, data follow $X_0=sU+\epsilon_0$, $X_1=sU+\epsilon_1$, $X_2=.8X_0+.6X_1+\epsilon_2$, and $X_3=.7X_2+\epsilon_3$, with independent standard Gaussian $U,\epsilon_0,\ldots,\epsilon_3$ and $s\in\{.4,.8,1.2\}$. Thus $s$ is the latent loading, giving error covariance $s^2$ between nodes 0 and 1 before standardization. The true ADMG has $0\leftrightarrow1$ and directed edges $0\to2,1\to2,2\to3$. Its 13 candidates comprise this graph and separate modifications: removing the bidirected edge; replacing it by either directed orientation; deleting or reversing one of the three directed edges while retaining $0\leftrightarrow1$; or relocating the bidirected edge to $\{0,3\}$, $\{1,3\}$, or $\{0,2\}$, removing any directed edge on that pair. Each pair has at most one relation. The states are $\{\text{none},i\to j,j\to i,i\leftrightarrow j\}$, with a $4\times4$ error matrix and $\alpha\in\{.75,.90,1\}$.

\paragraph{Feedback.}
The cyclic study retains the three states $\{\text{none},i\to j,j\to i\}$ and a $3\times3$ error matrix. Its 81 candidates independently assign these states to $\{0,1\}$, $\{1,2\}$, $\{0,2\}$, and $\{2,3\}$; the other pairs remain absent. This support allows directed cycles but not reciprocal edges. Before standardization, the true SEM has $W_{01}=W_{12}=W_{20}=c$, $W_{23}=.6$, all other coefficients zero, and $\Omega=I$, with $c\in\{.2,.5,.8\}$ and report accuracies $\alpha\in\{.75,1\}$. The separately reported DAG control uses the same data and reports with BGe scoring, the $.01$ base edge prior, and Top-$K$ search with 20 restarts and up to three HB updates. Cyclic SHD thresholds edge marginals at $>.5$ and counts ordered adjacency disagreements, so a reversal costs two.

Table~\ref{tab:other_graph_classes} reports results pooled over the tested strengths at $\alpha=.75$.

\begin{table}[htbp]
\centering
\small
\caption{Controlled tests with compatible graph support at $\alpha=.75$. Higher posterior masses and lower SHD are better. These tests concern finite four-node candidate sets.}
\label{tab:other_graph_classes}
\begin{tabular}{llrr}
\toprule
Graph class & Quantity & No-KG & HB \\
\midrule
ADMG & True-candidate mass $\uparrow$ & $.197$ & $.812$ \\
ADMG & Mass on $0\leftrightarrow1$ $\uparrow$ & $.591$ & $.926$ \\
Cyclic & SHD $\downarrow$ & $3.03$ & $.57$ \\
\bottomrule
\end{tabular}
\end{table}

With the compatible cyclic backend, HB attains SHD zero at $\alpha=1$. The DAG-only BGe comparison gives E-SHD $2.02$; any DAG must omit at least one edge of the cycle. The observation model can therefore use reports about broader graph classes when the pair states and posterior support represent those relations. The main DAG implementation retains its causally sufficient, acyclic scope.

\section{Optical-system evaluation}
\label{app:optical}

\subsection{Data and report elicitation}

We use the \texttt{uniform\_reference} observations in the public Causal Chambers dataset \texttt{lt\_interventions\_standard\_v1}~\citep{gamella2025chambers}, whose data are distributed under CC BY 4.0. Following its independent-observation causal-discovery example, we select 20 variables: the red, green, and blue source settings; electrical current; three infrared and three visible-light readings; two polarizer motor settings and two angle readings; and six auxiliary-light settings. The induced reference graph has 39 directed edges. The standard configuration fixes the remaining sensor parameters, and the reference graph is used only to evaluate the inferred structures.

We shuffle the 10,000 observations with seed 8467 and divide them into ten disjoint blocks of 1,000 rows. The first 250 rows of each shuffled block form its smaller-sample condition. Each training subset is centered and scaled using its own mean and population standard deviation. All methods receive the same subset and preprocessing. The two sample sizes are therefore paired, and all blocks describe one physical apparatus.

The LLM receives a generic optical-apparatus description and the variables' physical descriptions, requesting immediate direct parents. Natural labels retain dataset names; a permutation at seed 8468 assigns random IDs with the same descriptions. The dataset name, paper title, and reference graph are excluded from the prompt. GPT-4.1-mini produces four batches of five reports per condition at temperatures $.65,.70,.75,.80,.85$ with a common JSON schema. All forty calls return valid reports. Omitted pairs are encoded as abstentions.

For source evaluation, aggregation retains directed edges reported in at least two of five rounds. Table~\ref{tab:optical_sources} gives precision, recall, directed-edge F1, and coverage over 190 unordered pairs. Single-round means average twenty rounds, and aggregate means average four batches. Graph inference uses the individual rounds, with the same batches at both sample sizes.

\begin{table}[tbp]
\centering
\small
\caption{LLM report quality on the optical task. Coverage concerns unordered pairs; precision, recall, and F1 concern directed edges.}
\label{tab:optical_sources}
\begin{tabular}{llrrrrr}
\toprule
Labels & Reports & Precision & Recall & F1 & Coverage & Edges \\
\midrule
Natural & Single round & $.628$ & $.267$ & $.374$ & $.086$ & $16.40$ \\
Natural & At least $2/5$ rounds & $.629$ & $.301$ & $.407$ & $.099$ & $18.75$ \\
Random IDs & Single round & $.669$ & $.282$ & $.397$ & $.087$ & $16.45$ \\
Random IDs & At least $2/5$ rounds & $.657$ & $.282$ & $.395$ & $.088$ & $16.75$ \\
\bottomrule
\end{tabular}
\end{table}

\subsection{Matched inference and evaluation}

All four methods use BGe with the same base edge prior $\theta_0=.01$, allow up to 19 parents per node, and impose no known-root constraints. Each fit makes four graph calls: one data-only initialization and three updates, with 20 search restarts and at most 60 retained DAGs per call. The report-based methods use warm blend $\lambda_{\mathrm{warm}}=.5$ and report-reward scale one. No-KG uses a zero report reward in all four calls. Search seeds are $8467+b$ for block $b=0,\ldots,9$.

Full HB uses the feature-conditioned pair reliability model. The common-reliability control fits a single Beta reliability independently, counting its prior once as in Appendix~\ref{app:common_reliability}. Fixed-reliability HB (Fixed-$\rho$ HB in tables) uses the same observation channel as HB, with
\[
 P(Y=z\mid Z=z)=.8,\qquad
 P(Y=y\mid Z=z)=.2B_{zy}\quad(y\ne z).
\]
Its error matrix has zero diagonal and initial off-diagonal Dirichlet parameters one, and is updated from that control's own graph posterior at each step. Thus this comparison holds the likelihood form fixed while replacing inferred reliability by a fixed value. The three report-based methods each have $2$ sample sizes $\times\,2$ label conditions $\times\,10$ blocks $\times\,4$ batches, or 160 fits. No-KG has 20 fits shared across label conditions and batches, giving 500 fits in total.

SHD thresholds directed edge probabilities at $>.5$ and counts a reversal as two errors. E-SHD sums expected directed adjacency disagreements under the weighted graph samples. AUROC assigns half credit to tied scores, and Brier score averages squared error over the 380 off-diagonal directed entries. Table~\ref{tab:optical_metrics} gives all four metrics; Table~\ref{tab:optical_paired} reports paired SHD and AUROC differences.

Exploratory 95\% intervals use 10,000 percentile bootstrap replicates at seed 8469. We resample the ten data blocks and four whole report batches separately, following \citet{owen2007pigeonhole}, and preserve method pairing. No-KG uncertainty uses blocks alone. Intervals describe variation among the observed blocks and batches from this apparatus and have no multiplicity adjustment.

\begin{table}[tbp]
\centering
\small
\caption{All optical-task metrics. Fixed-$\rho$ HB uses $\rho=.8$ with learned $B$. Report-based means average ten blocks and four report batches; the repeated no-KG entries denote the same ten fits at each sample size. SHD, E-SHD, and Brier are lower-is-better; AUROC is higher-is-better.}
\label{tab:optical_metrics}
\begin{tabular}{rllrrrr}
\toprule
$N$ & Labels & Model & SHD & E-SHD & AUROC & Brier \\
\midrule
1000 & Natural & No-KG & $39.000$ & $39.9274$ & $.7383$ & $.0945$ \\
 & & HB & $29.550$ & $29.5899$ & $.7579$ & $.0765$ \\
 & & Fixed-$\rho$ HB & $31.200$ & $31.3108$ & $.7872$ & $.0814$ \\
 & & Common reliability & $44.800$ & $44.8140$ & $.6961$ & $.1141$ \\
 & Random IDs & No-KG & $39.000$ & $39.9274$ & $.7383$ & $.0945$ \\
 & & HB & $27.300$ & $27.2371$ & $.7643$ & $.0702$ \\
 & & Fixed-$\rho$ HB & $28.300$ & $28.3484$ & $.8182$ & $.0740$ \\
 & & Common reliability & $42.500$ & $43.8825$ & $.6992$ & $.1106$ \\
\midrule
250 & Natural & No-KG & $47.900$ & $48.3628$ & $.6938$ & $.1189$ \\
 & & HB & $36.250$ & $36.5519$ & $.7310$ & $.0927$ \\
 & & Fixed-$\rho$ HB & $31.775$ & $31.7803$ & $.7856$ & $.0807$ \\
 & & Common reliability & $55.025$ & $55.2041$ & $.6347$ & $.1393$ \\
 & Random IDs & No-KG & $47.900$ & $48.3628$ & $.6938$ & $.1189$ \\
 & & HB & $37.600$ & $37.5203$ & $.7306$ & $.0966$ \\
 & & Fixed-$\rho$ HB & $33.350$ & $33.1677$ & $.7994$ & $.0851$ \\
 & & Common reliability & $55.700$ & $55.2433$ & $.6328$ & $.1391$ \\
\bottomrule
\end{tabular}
\end{table}

\begin{table}[tbp]
\centering
\small
\caption{Optical-task paired contrasts: HB minus the comparator, with exploratory 95\% crossed-bootstrap intervals. Fixed-$\rho$ HB uses $\rho=.8$ with learned $B$. Negative SHD and positive AUROC differences favor HB.}
\label{tab:optical_paired}
\begin{tabular}{rllrr}
\toprule
$N$ & Labels & Comparator & $\Delta$SHD [95\% interval] & $\Delta$AUROC [95\% interval] \\
\midrule
1000 & Natural & No-KG & $-9.45\;[-16.30,-1.38]$ & $.0195\;[.0044,.0336]$ \\
 & & Fixed-$\rho$ HB & $-1.65\;[-5.48,1.73]$ & $-.0294\;[-.0573,-.0026]$ \\
 & & Common reliability & $-15.25\;[-28.60,-3.40]$ & $.0617\;[.0260,.0968]$ \\
 & Random IDs & No-KG & $-11.70\;[-16.68,-6.43]$ & $.0260\;[.0088,.0419]$ \\
 & & Fixed-$\rho$ HB & $-1.00\;[-6.73,4.40]$ & $-.0538\;[-.0896,-.0159]$ \\
 & & Common reliability & $-15.20\;[-28.13,-4.13]$ & $.0651\;[.0213,.1095]$ \\
\midrule
250 & Natural & No-KG & $-11.65\;[-16.08,-6.45]$ & $.0372\;[.0210,.0502]$ \\
 & & Fixed-$\rho$ HB & $4.48\;[.20,9.08]$ & $-.0546\;[-.0885,-.0169]$ \\
 & & Common reliability & $-18.78\;[-27.03,-10.45]$ & $.0964\;[.0767,.1130]$ \\
 & Random IDs & No-KG & $-10.30\;[-14.78,-5.53]$ & $.0368\;[.0285,.0460]$ \\
 & & Fixed-$\rho$ HB & $4.25\;[1.08,7.30]$ & $-.0688\;[-.0909,-.0392]$ \\
 & & Common reliability & $-18.10\;[-26.68,-9.65]$ & $.0978\;[.0769,.1155]$ \\
\bottomrule
\end{tabular}
\end{table}

HB improves all four means over no-KG and independent common reliability in each condition. These gains persist under random IDs with preserved descriptions. At $N=1000$, HB has lower SHD, E-SHD, and Brier than fixed-reliability HB. Fixed-reliability HB performs better at $N=250$ and has higher AUROC at both sizes. The experiment demonstrates useful integration of sparse, imperfect reports in a physical domain using semantic descriptions of the measured variables.

\end{document}